\documentclass{article}
\pdfoutput=1

\PassOptionsToPackage{numbers,sort&compress}{natbib}
\PassOptionsToPackage{hypertexnames=false}{hyperref}
\usepackage[final,main]{neurips_2026}

\usepackage[utf8]{inputenc} 
\usepackage[T1]{fontenc}    
\usepackage{url}            
\usepackage{booktabs}       
\usepackage{amsfonts}       
\usepackage{amsmath}
\usepackage{amsthm}         
\usepackage{thmtools}
\usepackage{thm-restate}
\usepackage{nicefrac}       
\usepackage{microtype}      
\usepackage{xcolor}         
\usepackage{graphicx}
\usepackage{float}
\usepackage{multirow}
\usepackage{array}
\usepackage{enumitem}
\usepackage{pifont}
\usepackage{hyperref}       
\hypersetup{pdftitle={A Unified Uncertainty Representation for Graph Neural Networks via Doubly-Spectral Stochastic Expansion},
  pdfauthor={Fred Xu, Thomas Markovich, Florence Regol, Yizhou Sun}}


\newcommand{\cmark}{\ding{51}}
\newcommand{\xmark}{\ding{55}}

\newtheorem{theorem}{Theorem}
\newtheorem{lemma}{Lemma}
\newtheorem{proposition}{Proposition}
\newtheorem{corollary}{Corollary}

\newtheorem{remark}{Remark}

\newif\ifarxiv
\arxivtrue

\newcommand{\coderepo}{\url{https://github.com/heraclixus/DSSGNN}}

\title{A Unified Uncertainty Representation for Graph Neural Networks via Doubly-Spectral Stochastic Expansion}

\author{%
  Fred Xu\thanks{Work done during an internship at Block, Inc.} \\
  Department of Computer Science \\
  University of California, Los Angeles \\
  \texttt{fredxu@cs.ucla.edu} \\
  \And
  Thomas Markovich \\
  Block, Inc. \\
  \texttt{tmarkovich@squareup.com} \\
  \And
  Florence Regol \\
  Mila, Quebec AI Institute \\
  \texttt{florence.regol@mila.quebec} \\
  \And
  Yizhou Sun \\
  Department of Computer Science \\
  University of California, Los Angeles \\
  \texttt{yzsun@cs.ucla.edu} \\
}

\begin{document}

\addtocontents{toc}{\protect\setcounter{tocdepth}{-10}}
\maketitle

\begin{abstract}
Reliable deployment of graph neural networks requires calibration,
out-of-distribution (OOD) detection, and robustness to distribution shift,
yet existing methods address these needs with separate models and
objectives. We model uncertain node embeddings as random graph
signals: graph Fourier filters capture structural variation, and a scalar
orthogonal-polynomial chaos coordinate captures latent stochastic variation.
The resulting \emph{doubly-spectral stochastic} (DSS) expansion supplies
task-matched readouts from one representation: the mean coefficient encodes
class evidence for the energy-based OOD score, the higher-order coefficients
encode structured logit variation, and quadrature averaging over the chaos coordinate defines the single
predictive distribution used for prediction and calibration. A capacity
theorem shows that, under a full-rank feature assumption, a restricted
subfamily matches the chaos coefficients of any Gaussian-latent random graph
signal, with exponentially decaying truncation error under a growth
condition; the task-level claims
are established empirically. \emph{DSS-GNN} has two
deployment modes: standalone, or as a residual branch beside a deterministic
encoder (\emph{DSS-Hybrid}). Standalone DSS-GNN achieves the lowest Brier
score among the compared uncertainty-aware baselines on all 14 node
classification benchmarks without post-hoc correction; DSS-Hybrid achieves
the best AUROC on most node-OOD settings, competitive cross-graph OOD
detection, and the strongest shifted accuracy on all 7 GOOD concept-shift
benchmarks under standard empirical risk minimization (ERM). Cross-evaluating both modes on all
three tasks shows that each remains effective on the other's tasks, with documented
exceptions, and yields explicit deployment guidance.
\end{abstract}

\section{Introduction}

Graph neural networks (GNNs) perform well on graph-structured data, yet
reliable deployment in drug discovery~\citep{stokes2020deep}, fraud
detection~\citep{dou2020caregnn}, or traffic forecasting~\citep{li2018dcrnn}
requires calibrated probabilities, out-of-distribution (OOD) detection, and
robustness to distribution shift. On graphs the three needs are coupled,
because uncertainty depends on neighborhood structure as well as node
features, and they are instances of one question: \emph{when, and how
confidently, can the model predict?} Calibration measures the agreement between
confidence and accuracy in distribution (ID), OOD detection requires low
confidence outside the training support, and shift robustness requires that
accuracy decrease slowly as the input distribution changes. A representation
that explicitly encodes predictive uncertainty addresses all three, with
calibrated probabilities, OOD evidence, and a regularized residual under
shift obtained from one signal.
Most existing methods target one requirement: sampling-based uncertainty
methods add computational overhead~\citep{gal2016dropout,lakshminarayanan2017simple,lin2024gnsd}; evidential models
such as Graph Posterior Network (GPN)~\citep{stadler2021gpn} provide OOD evidence
but not shift-robust classification, and intrinsic uncertainty models such
as G-$\Delta$UQ~\citep{trivedi2024gdeltauq} target calibration under
shift; energy-based detectors~\citep{wu2023gnnsafe,fuchsgruber2024gebm}
provide separation scores rather than calibrated probabilities; and
robust-training methods~\citep{sagawa2020groupdro,arjovsky2019irm,tar} do
not model uncertainty explicitly.

Spectral GNNs learn filters along the graph-frequency
axis~\citep{defferrard2016chebnet}; we add a \emph{second} spectral axis for
stochastic variation. \emph{DSS-GNN} represents each node embedding as a
polynomial chaos expansion in a scalar Gaussian latent~\citep{xiu_wiener} and
propagates the joint graph-frequency/chaos-order coefficients through a
learnable doubly-spectral layer (Figure~\ref{fig:architecture},
Appendix~\ref{app:overview}). One deterministic forward pass gives
task-matched readouts: the quadrature-averaged predictive distribution,
whose argmax is the point predictor, is used for prediction and calibration,
with the higher-order coefficients encoding structured logit variation; the
zeroth-order mean logit gives an
energy score for OOD detection that retains the common-logit shifts that
softmax cancels; and, paired with a deterministic encoder
(\emph{DSS-Hybrid}, e.g., GCN~\citep{kipf2017gcn}), the DSS branch acts as a
regularized low-/high-pass residual under distribution shift. DSS-GNN is
thus one uncertainty representation with two deployment modes;
Section~\ref{sec:exp_crosseval} evaluates both on all three tasks and states
which to prefer where. Our contributions are:
\begin{enumerate}[leftmargin=*, topsep=2pt, itemsep=2pt, parsep=0pt]
\item A \emph{doubly-spectral stochastic representation} for GNNs with a
capacity theorem: under a full-rank feature assumption, a restricted DSS
subfamily matches the chaos coefficients of any Gaussian-latent random graph
signal, with exponentially decaying truncation error for targets with an
entire extension of Gaussian growth (Theorem~\ref{thm:universal_approx}).
The theorem (Section~\ref{sec:theory}) establishes representational
capacity, consistent with the empirical plateau at $P{=}1$--$2$, but does not
by itself guarantee calibration, OOD detection, or robustness, which are
supported empirically.
\item A learnable doubly-spectral layer whose quadrature-averaged loss
differs from the mean-logit loss by a curvature-weighted correction over the
chaos covariance (Proposition~\ref{prop:curvature}), relating the chaos
coefficients to calibration, bounding the gap between the
quadrature predictive and the mean-logit softmax
(Corollary~\ref{cor:readout_gap}), and identifying the mean-logit energy as
the readout that retains common-logit shifts.
\item Empirically, the lowest Brier score among compared uncertainty-aware
baselines on all 14 calibration benchmarks (standalone), the best AUROC on
most node-OOD settings with competitive cross-graph OOD, and the strongest
shifted accuracy on all 7 concept-shift benchmarks under ERM (hybrid), plus
a cross-evaluation of both modes on all three tasks, component-wise
ablations (Table~\ref{tab:component_summary}), and explicit deployment
guidance.
\end{enumerate}

\section{Background}
\label{sec:background}

\paragraph{Spectral GNNs and uncertainty on graphs.}
A graph signal on $\mathcal{G}=(V,E)$ with $N$ nodes is a vector
$x\in\mathbb{R}^N$ in the Hilbert space $\ell^2(V)$; the Laplacian $L_G$ is
self-adjoint and its eigenvectors form the graph Fourier
basis~\citep{shuman2013gsp,ortega2018graph}. Spectral GNNs define graph
convolutions as polynomial filters of $L_G$; ChebNet~\citep{defferrard2016chebnet}
uses $g(L_G)=\sum_{k=0}^K c_k T_k(\widetilde L_G)$ with learned coefficients
$c_k$, Chebyshev polynomials $T_k$, and the rescaled Laplacian
$\widetilde L_G=2L_G/\lambda_{\max}-I$ ($\lambda_{\max}$ the largest
eigenvalue), which keeps filtering localized and allows low-pass or
high-pass responses under homophily or heterophily~\citep{bo2021fagcn}.
DSS-GNN adds a second spectral axis, chaos order, orthogonal to graph
frequency, and embeds uncertainty directly into spectral filtering.

\paragraph{Polynomial chaos and the doubly-spectral expansion.}
To organize \emph{stochastic} variation spectrally, as the graph Fourier
basis organizes structural variation, we use polynomial chaos expansion
(PCE)~\citep{wiener1938homogeneous,cameron1947orthogonal,xiu_wiener},
classical in uncertainty quantification~\citep{oladyshkin2023deeppce} but,
to our knowledge, not previously combined with graph spectral filtering
inside a GNN (Appendix~\ref{sec:chaos_background} gives a self-contained
introduction). PCE represents a square-integrable random variable
$Y(\omega)$ as a series in the orthonormal Hermite polynomials
$\{\Psi_n(\omega)\}_{n\ge0}$ of a Gaussian latent
$\omega\sim\mathcal{N}(0,1)$; truncating at order~$P$ gives
$Y(\omega)\approx\sum_{n=0}^P Y_n\Psi_n(\omega)$, where $Y_0$ is the mean and
$\sum_{n=1}^P Y_n^2$ (the \emph{chaos energy}) the variance. Combined with the Laplacian eigenbasis $\{u_j\}$
on $\ell^2(V)$, the Hermite basis on $L^2(\mathbb{R},\gamma)$ ($\gamma$ the
standard Gaussian measure) gives a tensor-product representation: a random
graph signal determined by $\omega$ is in
$\ell^2(V)\otimes L^2(\mathbb{R},\gamma)$ and admits the
\emph{doubly-spectral stochastic expansion}
\begin{equation}
Y(\omega)
=
\sum_{n=0}^{\infty}\sum_{j=1}^{N}
u_j\,\hat Y_{n,j}^{\top}\,\Psi_n(\omega),
\label{eq:doubly_spectral_expansion}
\end{equation}
where $\hat Y_{n,j}\in\mathbb{R}^d$ is the joint coefficient at graph
frequency~$\lambda_j$ (the $j$-th Laplacian eigenvalue) and chaos order~$n$,
and $d$ is the signal dimension; $\{u_j\otimes\Psi_n\}$ is a complete
orthonormal basis of this space (Remark~\ref{rem:completeness}), $\omega$ is
scalar, and each chaos coefficient is a $d$-dimensional graph signal. DSS-GNN
learns a truncation and a propagation rule on this tensor-product space.

\section{Method}
\label{sec:method}

DSS-GNN instantiates the doubly-spectral
expansion~\eqref{eq:doubly_spectral_expansion} as a trainable network by
learning a coupled propagation rule for the joint coefficients
$\hat Y_{n,j}$ through parameterized graph filters and non-intrusive chaos
projection (Figure~\ref{fig:architecture}, Appendix~\ref{app:overview}).

\paragraph{Intuition.}
A random graph signal assigns to each node a random variable with finite
second moment. We model all randomness as driven by a single shared standard
Gaussian factor $\omega$: every node embedding is a deterministic function of
$\omega$, expanded in the Hermite polynomials $\Psi_n(\omega)$, which play
along the stochastic axis exactly the role that the graph Fourier modes play
along the structural axis. A DSS layer therefore stores, for every node, $P+1$
coefficient vectors (one per chaos order) instead of a single hidden vector; a
deterministic embedding is the special case in which only the order-0
coefficient is nonzero. Each layer decomposes the embeddings onto this basis,
applies learned filters jointly per graph frequency and per chaos order (dual
low/high-pass Chebyshev filters along the graph axis, per-order gates along
the chaos axis), and reassembles. The core Gaussian is never materialized:
because pointwise nonlinearities mix chaos orders, the layer evaluates the
embedding field at $S$ deterministic Gauss--Hermite nodes, applies the
activation there, and projects back onto the Hermite basis. With $S \ge P+1$
the polynomial filter and gate terms are integrated exactly at the gate orders
used in practice ($P_g \le 1$, Appendix~\ref{sec:proofs}), and the
nonlinearity is handled by a controlled non-intrusive projection
(Remark~\ref{rem:nonintrusive}), so the network is deterministic end to end:
no Monte-Carlo sampling and no sampling variance. The expansion itself
requires only the shared Gaussian factor and finite second moments; the
exponential truncation rate of Theorem~\ref{thm:universal_approx}
additionally requires targets with an entire (analytic) Gaussian extension.
Readouts then select task-appropriate functionals of the same expansion: the
quadrature-averaged predictive~\eqref{eq:predictive_average} for prediction
and calibration, and the mean-logit energy~\eqref{eq:energy_score} for OOD
scores.

\subsection{DSS layer}
\label{sec:representation}

Let $\mathcal{G}=(V,E)$ have $N=|V|$ nodes, features
$X\in\mathbb{R}^{N\times d_{\rm in}}$, and Laplacian~$L_G$. DSS-GNN stacks
$L$ layers on a truncated expansion at chaos order~$P$; the input is lifted
by $H_n^{(0)}=XW_{\mathrm{in}}^{(n)}$ with
$W_{\mathrm{in}}^{(n)}\in\mathbb{R}^{d_{\rm in}\times d_0}$, and at
layer~$\ell\in\{0,\dots,L{-}1\}$ the embeddings are
\begin{equation}
H^{(\ell)}(\omega)
=
\sum_{n=0}^{P} H_n^{(\ell)}\Psi_n(\omega),
\qquad
H_n^{(\ell)}\in\mathbb{R}^{N\times d_\ell},
\label{eq:pce_state}
\end{equation}
with $d_\ell$ the hidden width, so each node maintains $P{+}1$
coefficient-valued hidden states. The layer updates them along \emph{both
spectral axes}. Along the graph-frequency axis, dual Chebyshev filters
propagate each coefficient through localized low-pass and high-pass branches,
\begin{equation}
G_{\mathrm{lp}}(L_G)=\sum_{k=0}^{K_{\mathrm{lp}}} c_k^{(\mathrm{lp})}T_k(\widetilde L_G),
\qquad
G_{\mathrm{hp}}(L_G)=\sum_{k=0}^{K_{\mathrm{hp}}} c_k^{(\mathrm{hp})}T_k(\widetilde L_G),
\label{eq:cheby_filters}
\end{equation}
with filter degrees $K_{\mathrm{lp}},K_{\mathrm{hp}}$ and learned
coefficients $c_k^{(\mathrm{lp})},c_k^{(\mathrm{hp})}\in\mathbb{R}$. Along the
chaos-order axis, per-order \emph{chaos gates} $\alpha_n(\omega),\beta_n(\omega)$
couple the two branches; they are deterministic scalars or low-order chaos
expansions of gate order~$P_g$ with learnable $a_{n,r},b_{n,r}\in\mathbb{R}$:
\begin{equation}
\alpha_n(\omega)=\textstyle\sum_{r=0}^{P_g} a_{n,r}\Psi_r(\omega),
\qquad
\beta_n(\omega)=\textstyle\sum_{r=0}^{P_g} b_{n,r}\Psi_r(\omega).
\label{eq:random_gates}
\end{equation}
Because pointwise nonlinearities mix chaos orders, the update has no closed
form. A non-intrusive projection evaluates the field at $S$ Gauss--Hermite
nodes $\{(\omega_s,\mu_s)\}_{s=1}^S$ (nodes $\omega_s\in\mathbb{R}$, weights
$\mu_s>0$), applies the activation $\sigma$, and projects back:
\begin{equation}
\begin{aligned}
\widetilde H^{(\ell+1)}(\omega_s)
&=
\Bigg(\sum_{n=0}^{P}\Psi_n(\omega_s)\Big(
\alpha_n(\omega_s)G_{\mathrm{lp}}(L_G)H_n^{(\ell)}
+\beta_n(\omega_s)G_{\mathrm{hp}}(L_G)H_n^{(\ell)}
\Big)\Bigg)W_\ell .
\end{aligned}
\label{eq:sample_preact}
\end{equation}
\begin{equation}
H_m^{(\ell+1)}
=
\sum_{s=1}^{S} \mu_s\,\sigma\!\big(\widetilde H^{(\ell+1)}(\omega_s)\big)\Psi_m(\omega_s),
\qquad m=0,\dots,P,
\label{eq:dss_projection_update}
\end{equation}
with $W_\ell\in\mathbb{R}^{d_\ell\times d_{\ell+1}}$ trainable (quadrature
details in Appendix~\ref{sec:chaos_background}; notation in
Table~\ref{tab:notation}, Appendix~\ref{app:overview}).

\subsection{Uncertainty readout}
\label{sec:readout}

After $L$ layers a linear readout $W_{\mathrm{out}}\in\mathbb{R}^{d_L\times C}$
($C$ classes, labels $y_i\in\{1,\dots,C\}$) produces per-order logit
coefficients $Z_{i,n}:=(H_n^{(L)}W_{\mathrm{out}})_i\in\mathbb{R}^C$ for
node~$i$. The mean logit $Z_{i,0}$ encodes the zeroth-order class evidence
used by the energy score; the logit covariance and chaos energy are
\begin{equation}
\Sigma_{z_i}
=
\sum_{n=1}^{P} Z_{i,n}Z_{i,n}^\top \in \mathbb{R}^{C\times C},
\qquad
\mathcal{E}_i
=
\operatorname{tr}(\Sigma_{z_i})
=
\sum_{n=1}^{P}\|Z_{i,n}\|^2.
\label{eq:logit_covariance}
\end{equation}
The quadrature logits $z_i^{(s)}=\sum_{n=0}^{P}Z_{i,n}\Psi_n(\omega_s)$ are
deterministic evaluations of the learned expansion (no Gaussian noise is
injected in training or inference) and define the readouts. The
\emph{quadrature-averaged predictive}
\begin{equation}
\bar p_i = \sum_{s=1}^{S} \mu_s\,\mathrm{softmax}(z_i^{(s)})
\in \mathbb{R}^C
\label{eq:predictive_average}
\end{equation}
is the model's predictive distribution and the point prediction is its
argmax, $\hat y_i=\arg\max_c\bar p_{i,c}$, so prediction and calibration are
two functionals of one distribution. Its deviation from the mean-logit
softmax is bounded by the chaos covariance,
$\|\bar p_i-\mathrm{softmax}(Z_{i,0})\|_\infty\le\tfrac{B}{2}
\operatorname{tr}(\Sigma_{z_i})$ for a universal constant $B$
(Corollary~\ref{cor:readout_gap}), which training keeps small
(Section~\ref{sec:training}); Section~\ref{sec:exp_calibration} measures
the agreement. For OOD detection, the \emph{energy score}~\citep{liu2020energy}
\begin{equation}
s_{\mathrm{energy}}(i) = - \log \sum_{c=1}^{C} \exp(Z_{i,0,c}),
\label{eq:energy_score}
\end{equation}
with $Z_{i,0,c}$ the $c$-th mean logit, captures confidence in the mean
prediction; predictive entropy and mutual information of the quadrature
softmaxes are auxiliary summaries of the same readout. The OOD experiments
use the energy score with graph propagation, following
GNNSafe~\citep{wu2023gnnsafe}.

\subsection{Theoretical properties}
\label{sec:theory}

Four results connect the layer to the
expansion~\eqref{eq:doubly_spectral_expansion}; proofs are in
Appendix~\ref{sec:proofs}.

\begin{restatable}[Generalization]{proposition}{propgeneralization}
\label{prop:generalization}
With $P=0$ and $P_g=0$, \eqref{eq:pce_state} has a single deterministic
coefficient $H_0^{(\ell)}$ with scalar gates, and the DSS layer reduces to a
standard polynomial spectral GNN layer, recovering ChebNet, GCN, and related
models as special cases.
\end{restatable}

\begin{restatable}[Doubly-spectral operator]{theorem}{thmjointspectral}
\label{thm:joint_spectral_operator}
With the pointwise nonlinearity omitted, gates of chaos degree at most $P_g$,
states of degree at most $P$, and a Gauss--Hermite rule exact for all triple
products $\Psi_r\Psi_n\Psi_m$ ($r\le P_g$; $n,m\le P$; $S\ge\lceil(P_g+2P+1)/2\rceil$
suffices), the linearized DSS layer is a linear operator on the joint
graph-frequency/chaos-order domain of~\eqref{eq:doubly_spectral_expansion},
acting frequency by frequency: the Chebyshev filters set the graph-frequency
response and the chaos gates the chaos-order mixing
(Appendix~\ref{sec:proof_joint_spectral}; the omitted nonlinearity's
projection error is decomposed in Proposition~\ref{prop:projection_error}).
\end{restatable}

\begin{restatable}[Representational capacity and truncation under full-rank features]{theorem}{thmuniversalapprox}
\label{thm:universal_approx}
Let $\mathcal{G}$ have $N$ nodes and features
$X\in\mathbb{R}^{N\times d_{\rm in}}$ with $\operatorname{rank}(X)=N$ (which
requires $d_{\rm in}\ge N$). With ReLU activation, a restricted subfamily of
DSS-GNNs (one layer, lifted width $2d(P{+}1)$, quadrature size $S=P{+}1$,
constant gates, identity low-pass filter) has the following properties.
For any target $Y\in L^2(\mathbb{R},\gamma;\mathbb{R}^{N\times d})$ and any
chaos order $P$, some output $\widetilde Y_P$ matches the chaos coefficients
of~$Y$ through order~$P$; letting $P$ and width vary, the resulting outputs
are dense in $L^2(\mathbb{R},\gamma;\mathbb{R}^{N\times d})$. If $Y$ in
addition admits an entire Gaussian extension with controlled growth
(Appendix~\ref{sec:proof_decay}), the truncation error decays exponentially:
\begin{equation}
\|Y-\widetilde Y_P\|_{L^2(\mathbb{R},\gamma;\,\mathbb{R}^{N\times d})}
\le C_{\rm tr}\,\beta^{P}
\label{eq:approx_truncation_claims}
\end{equation}
for some $C_{\rm tr}>0$ and $\beta\in(0,1)$ depending on~$Y$.
\end{restatable}

Theorems~\ref{thm:joint_spectral_operator} and~\ref{thm:universal_approx}
are capacity statements: they establish the operator structure and the
representational capacity of the layer and identify no inherent expressivity
obstruction, but they do not by themselves guarantee calibration, OOD
detection, or robustness, which are empirical claims supported task by task
in Section~\ref{sec:experiments}. The full-rank case requires $d_{\rm in}\ge N$; for
$\operatorname{rank}(X)<N$, the common case (Cora has 1{,}433 features for
2{,}708 nodes), Remark~\ref{rem:rank_deficient} shows that the same
construction still matches
every truncated target whose chaos-coefficient columns lie in
$\operatorname{range}(X)$, under one positivity condition satisfiable by
appending a constant feature column.

\begin{restatable}[Stochastic and mean-logit readouts]{proposition}{propcurvature}
\label{prop:curvature}
Assume the number of Gauss--Hermite quadrature nodes is at least one more
than the chaos order, and that the per-node loss $\ell(\cdot, y_i)$ satisfies
the regularity conditions in Lemma~\ref{lem:readout_conditions}. Then:
\begin{enumerate}
\item (Calibration channel.) The per-node quadrature-averaged loss
$\bar\ell_i := \sum_{s=1}^{S} \mu_s\,\ell(z_i^{(s)},y_i)$ admits the Taylor
expansion
\(\bar\ell_i = \ell(Z_{i,0},y_i) + \tfrac{1}{2}\operatorname{tr}(\mathcal{H}_i\,\Sigma_{z_i}) + O\!\big(\max_s\|z_i^{(s)}-Z_{i,0}\|^3\big),\)
where $\mathcal{H}_i = \nabla_z^2\ell(Z_{i,0},y_i)$ is the loss Hessian and
$\Sigma_{z_i}$ is the chaos covariance~\eqref{eq:logit_covariance}.
\item (OOD channel.) Shifting every class logit by the same constant leaves
the softmax output, and hence any softmax-derived summary such as predictive
entropy or mutual information, unchanged; by contrast, the energy
score~\eqref{eq:energy_score} decreases by exactly that constant.
\end{enumerate}
\end{restatable}

\emph{Calibration:} the curvature-weighted term
$\tfrac{1}{2}\operatorname{tr}(\mathcal{H}_i\Sigma_{z_i})$ is the
leading-order difference between the quadrature-averaged and the mean-logit
objectives; it weights the chaos covariance by the loss curvature and, for
cross-entropy, is at most $\tfrac14\mathcal{E}_i$, the quantity
that~\eqref{eq:loss_base} regularizes; no post-hoc temperature scaling is
used. The same expansion applied to the softmax readout bounds the gap
between $\bar p_i$ and $\mathrm{softmax}(Z_{i,0})$ by the chaos covariance
(Corollary~\ref{cor:readout_gap}), which the objective keeps small; the
calibration gains of $P>0$ therefore arise through training rather than a
readout-time correction, and the two trained readouts agree closely
(Section~\ref{sec:exp_calibration}).
\emph{OOD detection:} the energy score is linear in a common logit
shift, while softmax-derived summaries cancel it, so an OOD input that
uniformly lowers the class-evidence logits changes the energy but not the
entropy; Section~\ref{sec:exp_ood} confirms that
detection is insensitive to the chaos order, as this readout predicts.

\subsection{Hybrid deployment}
\label{sec:hybrid}

Distribution shift weakens the common-evidence signal and changes which
graph frequencies are predictive; like OOD detection it requires a strong
deterministic predictor on seen structure paired with a regularized
residual that accounts for cross-distribution variability. \emph{DSS-Hybrid} adds
DSS-GNN as a residual
branch beside a deterministic base encoder with one shared class readout,
$z_{i,0}^{\mathrm{hyb}}=z_i^{\mathrm{base}}+\gamma_{\mathrm{res}}Z_{i,0}$,
with $z_i^{\mathrm{base}}\in\mathbb{R}^C$ the base logit and
$\gamma_{\mathrm{res}}\in\mathbb{R}$ a learned scale (lowercase for
post-combination logits, uppercase $Z_{i,n}$ for the DSS coefficients); the
quadrature logits combine the same way,
$z_i^{\mathrm{hyb},(s)}=z_i^{\mathrm{base}}+\gamma_{\mathrm{res}}\sum_{n=0}^P Z_{i,n}\Psi_n(\omega_s)$,
so the readouts of Section~\ref{sec:readout} transfer directly. The
chaos-energy regularizer keeps the residual from fitting in-distribution
noise, so its dual filters retain capacity for shifts in graph-frequency
response and its mean logit remains a common-evidence channel for the
energy score. The base encoder, a 2-layer GCN for fair comparison with
baselines, is warmed up before the DSS branch is activated. This design is
an inductive bias rather than a distribution-free guarantee.

\subsection{Training and downstream tasks}
\label{sec:training}

Both modes are trained with the chaos-energy-regularized mean-logit loss
\begin{equation}
\mathcal{L}_{\mathrm{base}}
=
\frac{1}{|\mathcal{V}_{\mathrm{tr}}|}
\sum_{i\in\mathcal{V}_{\mathrm{tr}}}
\ell(Z_{i,0},\,y_i)
\;+\;
\lambda_{\mathrm{reg}}
\frac{1}{|\mathcal{V}_{\mathrm{tr}}|}
\sum_{i\in\mathcal{V}_{\mathrm{tr}}} \mathcal{E}_i,
\label{eq:loss_base}
\end{equation}
with $\mathcal{V}_{\mathrm{tr}}$ the training nodes, $\ell$ the per-node
cross-entropy, and $\lambda_{\mathrm{reg}}\ge0$; for DSS-Hybrid, $Z_{i,0}$ is
replaced by $z_{i,0}^{\mathrm{hyb}}$ and the regularizer applies to the DSS
coefficients only. Only the mean logit is supervised, but the projection
update~\eqref{eq:dss_projection_update} couples $Z_{i,0}$ to all chaos
orders, so the higher-order coefficients are trained implicitly while the
regularizer constrains their variation. By Proposition~\ref{prop:curvature},
$\mathcal{L}_{\mathrm{base}}$ agrees to first order with the
quadrature-averaged loss $\mathcal{L}_{\mathrm{quad}}$~\eqref{eq:loss_quad}
of Appendix~\ref{app:implementation}; we use $\mathcal{L}_{\mathrm{base}}$
throughout. The trace that Corollary~\ref{cor:readout_gap} bounds is the
regularized chaos energy $\mathcal{E}_i$; where the selected
$\lambda_{\mathrm{reg}}$ is small or zero, the measured readout gap remains
small (Table~\ref{tab:readout_coincidence}).

\section{Experiments}
\label{sec:experiments}

\subsection{Setup}
\label{sec:exp_setup}

For \emph{calibration}, standalone DSS-GNN is compared with
TFE-GNN~\citep{duan2024tfe} and G-$\Delta$UQ~\citep{trivedi2024gdeltauq} on
14 node classification
benchmarks~\citep{kipf2017gcn,shchur2018pitfalls,pei2020geomgcn,platonov2023critical},
reporting accuracy and the Brier score~\citep{brier1950} over 10 splits
without post-hoc calibration. For \emph{OOD detection}, DSS-Hybrid is
evaluated on the GNNSafe benchmark~\citep{wu2023gnnsafe} (nine node-OOD
settings on Cora, Amazon-Photo, and Coauthor-CS, plus cross-graph Twitch and
Arxiv; 3 seeds) against GNNSafe, GNNSafe++, and GPN~\citep{stadler2021gpn} at the values
reported by \citet{wu2023gnnsafe}, and against
Graph-EBM~\citep{fuchsgruber2024gebm} (scored with its own diffusion),
MC-dropout ($M{=}20$)~\citep{gal2016dropout} and Deep Ensembles
($M{=}5$)~\citep{lakshminarayanan2017simple} on the same GCN backbone with
the same score propagation. For \emph{distribution shift}, DSS-Hybrid
is trained with empirical risk minimization (ERM) on 7 concept-shift settings of
GOOD~\citep{gui2022good} and compared with the 12 GOOD/TAR
baselines~\citep{arjovsky2019irm,sagawa2020groupdro,tar} and G-$\Delta$UQ.
Point predictions are $\arg\max\bar p_i$ (Section~\ref{sec:readout});
Table~\ref{tab:acc_brier} reports the mean-logit readout, which Section~\ref{sec:exp_calibration} shows to agree with it. All experiments
use
$\mathcal{L}_{\mathrm{base}}$~\eqref{eq:loss_base}; in the OOD experiments DSS-Hybrid,
like GNNSafe++, adds the GNNSafe++ energy-margin regularizer at per-dataset
margins. Baselines, datasets, and configurations are in Appendices~\ref{app:baselines_datasets}--\ref{app:implementation};
code is available at \coderepo.

\subsection{Node classification and calibration}
\label{sec:exp_calibration}

Table~\ref{tab:acc_brier} shows that DSS-GNN achieves the lowest Brier score
among the compared uncertainty-aware baselines on all 14 datasets,
homophilous and heterophilous, without post-hoc calibration, and the best
accuracy on 13. The Brier gains over the deterministic $P=0$ arm are largest
on Citeseer, Texas, and Minesweeper (Table~\ref{tab:p_brier}), consistent
with the curvature-weighted correction of Proposition~\ref{prop:curvature}.
This is not an accuracy effect: standard GNNs match or exceed DSS-GNN in
accuracy on Cora and Citeseer while remaining less calibrated (Appendix~\ref{app:simple_baselines}), and at
$P=0$ DSS-GNN is a competitive dual-filter spectral GNN
(Proposition~\ref{prop:generalization}), so the gains are attributable to
the chaos expansion (Table~\ref{tab:decompose}). Figure~\ref{fig:proposition2}
visualizes the
correction $c_i=\tfrac12\operatorname{tr}(\mathcal{H}_i\Sigma_{z_i})$ on
Minesweeper: it is largest where stochastic variation aligns with
high-curvature directions and concentrates on the hardest nodes.

\paragraph{Prediction and calibration use one predictive distribution.}
Under the full protocol with both readouts evaluated at the same checkpoints,
$\arg\max\bar p_i$ and $\arg\max Z_{i,0}$ disagree on at most $0.21\%$ of
test nodes (exactly zero on 10 of 14 datasets), with accuracy within $0.04$
points and Brier within $0.003$ (Table~\ref{tab:readout_coincidence},
Appendix~\ref{app:crosseval}); in hybrid mode the disagreement is $0.0000$ on
all 14 datasets. The agreement is structural: for $S\ge P{+}1$
Corollary~\ref{cor:readout_gap}
bounds the readout gap by $\operatorname{tr}(\Sigma_{z_i})$, which the
objective keeps small (Section~\ref{sec:training}), so the reported
calibration is that of the distribution the predictions are drawn from.

\paragraph{Chaos order.}
\label{sec:ablation}
Sweeping $P\in\{0,\dots,6,8\}$ (Appendix~\ref{app:ablation},
Figure~\ref{fig:p_ablation}) improves Brier on 12 of 14 datasets at $P>0$,
and at least one of $P=1$ or $P=2$ is within $0.02$ Brier of the best swept
order on every dataset, consistent with the exponential truncation bound of
Theorem~\ref{thm:universal_approx}. $P$ is selected per dataset on
validation Brier (Table~\ref{tab:selection}); the larger selected orders are
near-ties rather than requirements, and we recommend $P=1\text{--}2$ as the
default. Within the shared-factor family $P=1$ is exactly a Gaussian
embedding; the higher orders add non-Gaussian structure (Citeseer Brier
$0.346$ to $0.308$ at $P=2$; Tables~\ref{tab:p_acc}--\ref{tab:p_brier}).

\begin{figure*}[t]
\centering
\includegraphics[width=\textwidth]{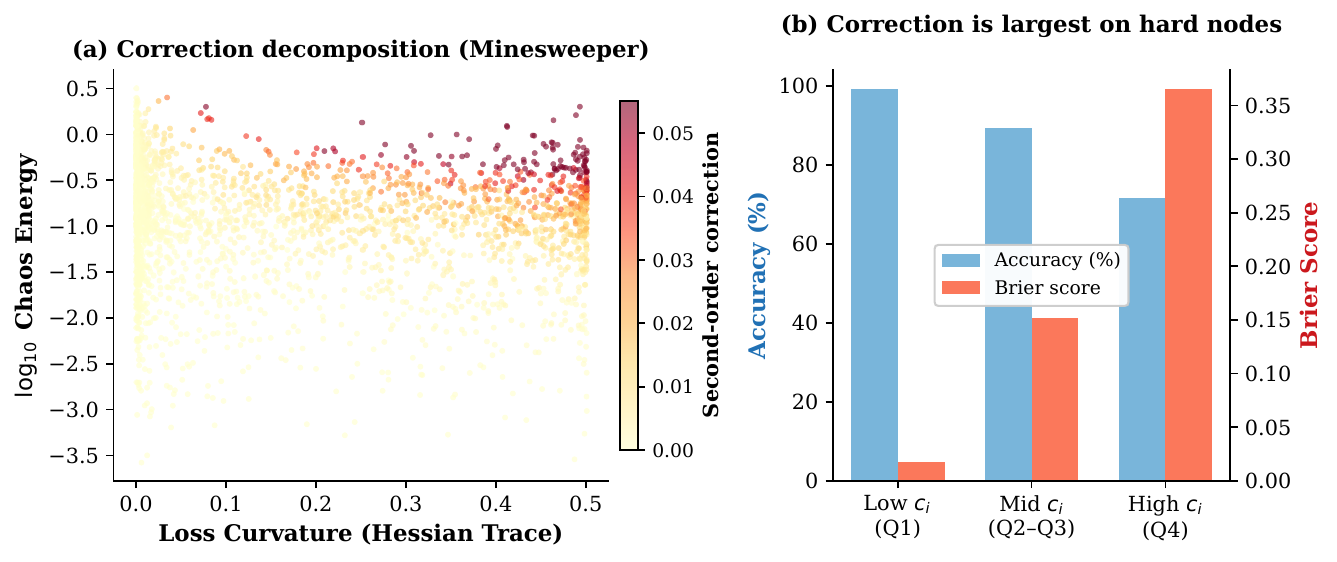}
\caption{\textbf{(a)}~Second-order correction $c_i$
(Proposition~\ref{prop:curvature}) on Minesweeper. Each node is positioned by
loss curvature ($\operatorname{tr}(\mathcal{H}_i)$) and chaos energy
($\mathcal{E}_i$); color shows correction magnitude.
\textbf{(b)}~Nodes binned by correction quartile: the correction concentrates
on hard nodes (low accuracy, high Brier).}
\label{fig:proposition2}
\end{figure*}

%

\begin{table*}[t]
\centering
\small
\caption{Node classification: test accuracy (\%, $\uparrow$) and Brier score ($\downarrow$). Mean $\pm$ standard deviation over 10 runs. Best result per dataset in \textbf{bold}. DSS-GNN uses the per-dataset configuration of Table~\ref{tab:selection} and the mean-logit readout.}
\label{tab:acc_brier}
\resizebox{\textwidth}{!}{%
\begin{tabular}{l ccc ccc}
\toprule
 & \multicolumn{3}{c}{Accuracy (\%)} & \multicolumn{3}{c}{Brier Score} \\
\cmidrule(lr){2-4} \cmidrule(lr){5-7}
Dataset & DSS-GNN & TFE-GNN & G-$\Delta$UQ & DSS-GNN & TFE-GNN & G-$\Delta$UQ \\
\midrule
Cora       & \textbf{86.63 $\pm$ 1.26} & 71.25 $\pm$ 1.90 & 84.35 $\pm$ 1.97 & \textbf{0.207 $\pm$ 0.019} & 0.449 $\pm$ 0.024 & 0.226 $\pm$ 0.023 \\
Citeseer   & \textbf{80.20 $\pm$ 1.28} & 68.87 $\pm$ 1.87 & 70.85 $\pm$ 2.22 & \textbf{0.308 $\pm$ 0.010} & 0.463 $\pm$ 0.020 & 0.433 $\pm$ 0.026 \\
PubMed     & \textbf{89.72 $\pm$ 0.31} & 82.84 $\pm$ 0.76 & 88.18 $\pm$ 0.66 & \textbf{0.156 $\pm$ 0.005} & 0.265 $\pm$ 0.011 & 0.179 $\pm$ 0.008 \\
Texas      & \textbf{91.31 $\pm$ 3.44} & 84.92 $\pm$ 4.26 & 9.02 $\pm$ 1.83  & \textbf{0.240 $\pm$ 0.111} & 0.255 $\pm$ 0.083 & 0.779 $\pm$ 0.029 \\
Cornell    & \textbf{85.11 $\pm$ 5.12} & 81.28 $\pm$ 5.45 & 21.06 $\pm$ 3.49 & \textbf{0.231 $\pm$ 0.079} & 0.290 $\pm$ 0.069 & 0.793 $\pm$ 0.023 \\
Wisconsin  & \textbf{93.25 $\pm$ 3.39} & 76.50 $\pm$ 23.95 & 52.25 $\pm$ 17.80 & \textbf{0.104 $\pm$ 0.047} & 0.292 $\pm$ 0.186 & 0.571 $\pm$ 0.109 \\
Chameleon  & \textbf{75.36 $\pm$ 1.19} & 73.41 $\pm$ 1.53 & 37.31 $\pm$ 2.36 & \textbf{0.339 $\pm$ 0.019} & 0.365 $\pm$ 0.016 & 0.765 $\pm$ 0.011 \\
Squirrel   & 68.06 $\pm$ 2.21 & \textbf{71.56 $\pm$ 1.42} & 25.88 $\pm$ 1.91 & \textbf{0.394 $\pm$ 0.027} & 0.395 $\pm$ 0.017 & 0.795 $\pm$ 0.006 \\
CS         & \textbf{96.17 $\pm$ 0.23} & 95.05 $\pm$ 0.24 & 94.37 $\pm$ 0.23 & \textbf{0.062 $\pm$ 0.004} & 0.082 $\pm$ 0.005 & 0.085 $\pm$ 0.004 \\
\midrule
Roman-Emp. & \textbf{78.84 $\pm$ 0.61} & 70.35 $\pm$ 0.60 & 51.24 $\pm$ 1.07 & \textbf{0.300 $\pm$ 0.007} & 0.414 $\pm$ 0.006 & 0.641 $\pm$ 0.008 \\
Amz-Rat.   & \textbf{49.72 $\pm$ 0.64} & 48.79 $\pm$ 0.43 & 44.43 $\pm$ 0.42 & \textbf{0.634 $\pm$ 0.005} & 0.637 $\pm$ 0.004 & 0.674 $\pm$ 0.003 \\
Minesweeper & \textbf{87.66 $\pm$ 1.10} & 82.42 $\pm$ 0.33 & 80.22 $\pm$ 0.22 & \textbf{0.169 $\pm$ 0.014} & 0.248 $\pm$ 0.003 & 0.288 $\pm$ 0.003 \\
Tolokers   & \textbf{79.88 $\pm$ 0.49} & 78.18 $\pm$ 0.10 & 79.75 $\pm$ 0.55 & \textbf{0.268 $\pm$ 0.005} & 0.382 $\pm$ 0.004 & 0.289 $\pm$ 0.007 \\
Questions  & \textbf{97.17 $\pm$ 0.04} & 97.07 $\pm$ 0.09 & 97.06 $\pm$ 0.05 & \textbf{0.053 $\pm$ 0.001} & 0.061 $\pm$ 0.001 & 0.054 $\pm$ 0.001 \\
\bottomrule
\end{tabular}%
}
\end{table*}

\subsection{OOD detection}
\label{sec:exp_ood}

Tables~\ref{tab:gnnsafe_local_table2style}--\ref{tab:gnnsafe_local_table1style}
compare DSS-Hybrid with GNNSafe++~\citep{wu2023gnnsafe}, the
strongest OOD baseline: it has the best AUROC on 6 of 9 node-OOD cells (Cora
$94.32/97.60/94.11$ versus $90.62/95.56/92.75$ for structure/feature/label;
Amazon-Photo feature/label $99.66/97.52$; Coauthor-CS label $98.07$) and
is within $0.13$ points of the best method on the three remaining cells,
where both it and GNNSafe++ exceed $99.6$. On Twitch it improves over
GNNSafe++ on every metric (AUROC $95.75$ vs $95.36$, FPR95 $22.40$ vs
$33.57$); on Arxiv it gives the best FPR95 and ID accuracy while remaining
within $1.41$ AUROC points of GNNSafe++ ($73.36$ vs $74.77$). DSS-Hybrid follows the
GNNSafe++ protocol (same GCN backbone configuration, energy-margin objective
with its per-dataset margins, and score propagation;
Appendix~\ref{app:implementation}), so the two differ in the DSS
residual branch. Against MC-dropout and
Deep Ensembles, which receive the same propagation, DSS-Hybrid has the
higher AUROC on all 11 settings in one forward pass versus their 20 passes
and 5 models; the published GPN values are below GNNSafe++ throughout.

\begin{table}[t]
\centering
\small
\caption{OOD detection (AUROC, \%) on the nine node-OOD settings and the two
cross-graph settings (AUPR, FPR95, and ID accuracy for the latter are in
Table~\ref{tab:gnnsafe_local_table1style}).
GNNSafe/GNNSafe++ and GPN~\citep{stadler2021gpn} values as reported by
\citet{wu2023gnnsafe} on this benchmark (their Tables~1--2; their appendix
lists $81.77/93.24$ for GPN on Coauthor-CS feature/label; GPN out-of-memory
on Arxiv); Graph-EBM~\citep{fuchsgruber2024gebm}, MC-dropout ($M{=}20$),
and Deep Ensembles ($M{=}5$) run in our pipeline on the same GCN backbone
(3 seeds; the latter two with the same score propagation). Bold: best.}
\label{tab:gnnsafe_local_table2style}
\resizebox{\textwidth}{!}{%
\begin{tabular}{l ccc ccc ccc cc}
\hline
Model & \multicolumn{3}{c}{\textsc{Cora}} & \multicolumn{3}{c}{\textsc{Amazon-Photo}} & \multicolumn{3}{c}{\textsc{Coauthor-CS}} & \multicolumn{2}{c}{\textsc{Cross-graph}} \\
 & Structure & Feature & Label & Structure & Feature & Label & Structure & Feature & Label & Twitch & Arxiv \\
\hline
GNNSafe     & 87.52 & 93.44 & 92.80 & 99.58 & 98.55 & 97.35 & 99.60 & 99.64 & 97.23 & 66.82 & 71.06 \\
GNNSafe++   & 90.62 & 95.56 & 92.75 & \textbf{99.82} & 99.64 & 97.51 & \textbf{99.99} & \textbf{99.97} & 97.89 & 95.36 & \textbf{74.77} \\
Graph-EBM   & 61.14 & 72.42 & 92.69 & 75.22 & 86.49 & 97.24 & 72.75 & 89.28 & 97.91 & 44.43 & 52.80 \\
MC-dropout  & 87.50 & 93.12 & 93.03 & 98.67 & 98.50 & 96.80 & 99.51 & 99.53 & 97.25 & 68.58 & 66.34 \\
Deep Ensemble & 87.87 & 93.63 & 93.75 & 98.58 & 98.43 & 97.35 & 98.18 & 98.45 & 94.97 & 72.25 & 67.72 \\
GPN         & 77.47 & 85.88 & 90.34 & 97.17 & 87.91 & 92.72 & 34.67 & 72.56 & 83.65 & 51.73 & OOM \\
DSS-Hybrid  & \textbf{94.32} & \textbf{97.60} & \textbf{94.11} & 99.69 & \textbf{99.66} & \textbf{97.52} & 99.96 & 99.95 & \textbf{98.07} & \textbf{95.75} & 73.36 \\
\hline
\end{tabular}%
}
\vspace{-0.4em}
\end{table}

\begin{table}[t]
\centering
\small
\caption{Cross-graph OOD (AUROC, AUPR: $\uparrow$; FPR95, false positive rate
at 95\% true positive rate: $\downarrow$; ID: in-distribution; bold: best).
GNNSafe/GNNSafe++ and GPN values as reported by \citet{wu2023gnnsafe} (OOM:
out of memory on a 24\,GB GPU); Graph-EBM~\citep{fuchsgruber2024gebm}, MC-dropout, and Deep Ensembles run in
our pipeline (3 seeds).}
\label{tab:gnnsafe_local_table1style}
\begin{tabular}{l cccc cccc}
\hline
 & \multicolumn{4}{c}{\textsc{Twitch}} & \multicolumn{4}{c}{\textsc{Arxiv}} \\
Model & AUROC & AUPR & FPR95 & ID Acc. & AUROC & AUPR & FPR95 & ID Acc. \\
\hline
GNNSafe     & 66.82 & 70.97 & 76.24 & 70.40 & 71.06 & 80.44 & 87.01 & 53.39 \\
GNNSafe++   & 95.36 & 97.12 & 33.57 & 70.18 & \textbf{74.77} & \textbf{83.21} & 77.43 & 53.50 \\
Graph-EBM   & 44.43 & 57.91 & 94.84 & 63.04 & 52.80 & 65.88 & 97.40 & 53.45 \\
MC-dropout  & 68.58 & 81.25 & 94.84 & 64.10 & 66.34 & 74.88 & 89.69 & 53.56 \\
Deep Ensemble & 72.25 & 83.19 & 94.29 & 66.94 & 67.72 & 75.98 & 87.32 & 47.58 \\
GPN         & 51.73 & 66.36 & 95.51 & 68.09 & OOM & OOM & OOM & OOM \\
DSS-Hybrid  & \textbf{95.75} & \textbf{97.51} & \textbf{22.40} & \textbf{70.52} & 73.36 & 79.11 & \textbf{76.67} & \textbf{53.99} \\
\hline
\end{tabular}
\end{table}

Detection comes from the energy readout, with propagation, acting on the DSS
mean logit~\eqref{eq:energy_score}, which Proposition~\ref{prop:curvature}
formalizes as the readout that retains common-evidence shifts. Under this
identical readout the DSS representation separates ID from OOD more widely
than the GCN encoder alone (the Cora columns of Table~\ref{tab:gnnsafe_local_table2style}), and a controlled ablation
that varies only the chaos order
$P\in\{0,\dots,3\}$ at one fixed configuration leaves energy AUROC flat
(Cora-structure $91.0/91.2/90.9/90.8$; the level differs from
Table~\ref{tab:gnnsafe_local_table2style} as the ablation is a separate set
of runs; Appendix~\ref{app:ablation}), as the design predicts since
the energy score reads the mean logit alone. Score propagation contributes
up to $+18.2$ AUROC (Twitch) and $+14.5$ (Cora-structure) at identical
checkpoints (Appendix~\ref{app:ablation}). Figure~\ref{fig:ood_density}
shows the resulting wider energy-density gap between ID and OOD modes
relative to the energy-based baselines.

\begin{figure}[!htbp]
\centering
\includegraphics[width=0.9\textwidth]{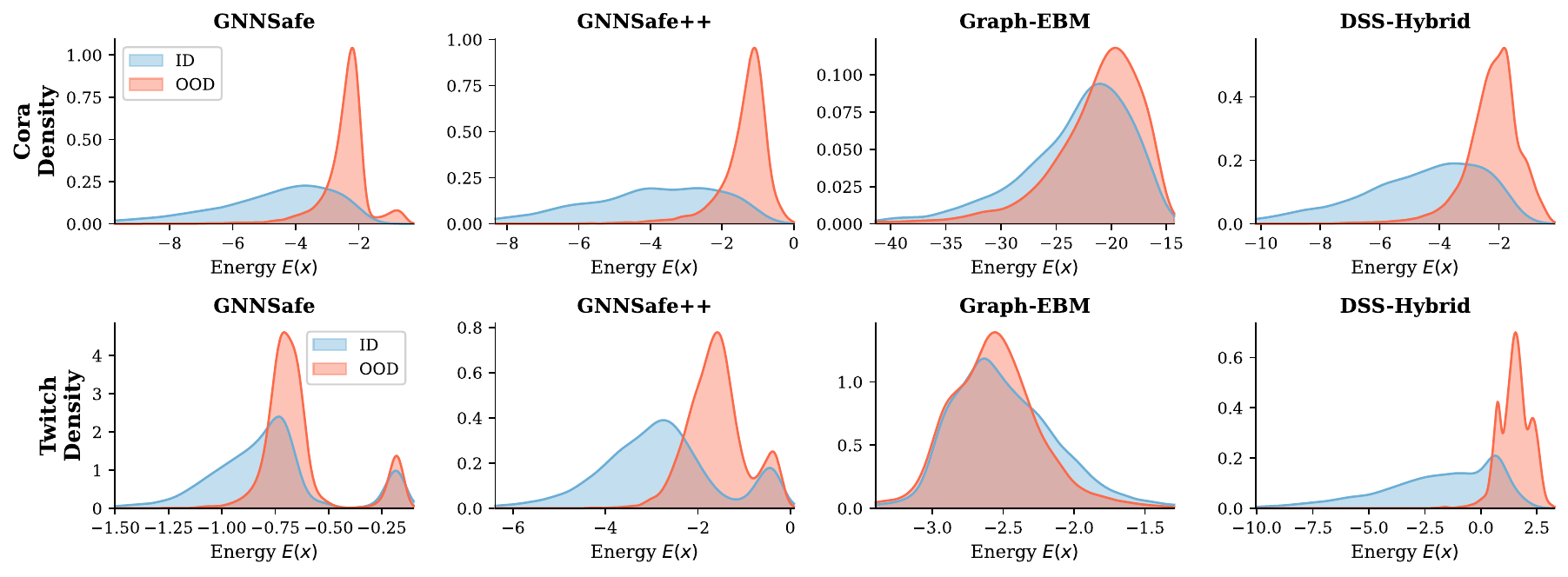}
\vspace{-0.6em}
\caption{Energy density for ID (blue) vs.\ OOD (red). Top: Cora/structure.
Bottom: Twitch (cross-graph). Lower energy indicates higher confidence.
DSS-Hybrid separates the ID and OOD modes on both datasets.}
\label{fig:ood_density}
\vspace{-0.6em}
\end{figure}

\subsection{Distribution-shifted classification}
\label{sec:exp_good}

The GOOD benchmark~\citep{gui2022good} tests generalization \emph{under} distribution shift. Table~\ref{tab:good_table2_style} compares DSS-Hybrid,
trained with standard ERM, against the GOOD/TAR baselines and G-$\Delta$UQ
on seven concept-shift settings (covariate-shift splits in
Appendix~\ref{app:covariate}): it has the best mean shifted accuracy in
every case, by margins from
$0.08$ (GOOD-Cora/word) to $9.94$ points (GOOD-WebKB/university). This is consistent with the inductive bias of Section~\ref{sec:hybrid}: the
regularized low-/high-pass residual can model the shift in which graph
frequencies are predictive while the base predictor is preserved. Section~\ref{sec:exp_crosseval} isolates the residual pathway as a unit (hybrid versus its identically trained GCN base), and
Appendix~\ref{app:ablation} (Table~\ref{tab:good_pabl}) isolates the chaos
order under shift: at the validation-selected standalone configurations, $P{=}1$
improves shifted accuracy over $P{=}0$ on GOOD-Cora/word ($+1.7$),
GOOD-Cora/degree ($+1.2$), and GOOD-Arxiv/time ($+0.6$), so on those
settings the \emph{uncertainty} representation, not only the dual-filter
structure, contributes.

\begin{table}[t]
\centering
\caption{Test accuracy (\%) on the concept-shift GOOD node benchmark for
DSS-Hybrid (standard ERM training) and 13 baselines; ERM through TAR as
reported by \citet{tar}, G-$\Delta$UQ run in our pipeline. Best in bold,
second-best underlined; OOM: out of memory.}
\label{tab:good_table2_style}
\resizebox{\textwidth}{!}{%
\begin{tabular}{lccccccc}
\toprule
Model & GOOD-CBAS & GOOD-WebKB & GOOD-Twitch & GOOD-Cora & GOOD-Cora & GOOD-Arxiv & GOOD-Arxiv \\
 & color & university & language & word & degree & time & degree \\
\midrule
ERM & 82.43 & 27.16 & 51.59 & 64.03 & 60.30 & 65.64 & 54.81 \\
IRM & 82.00 & 26.06 & 49.78 & 63.93 & 60.26 & 65.54 & 56.72 \\
VREx & 82.86 & 26.61 & 55.75 & 64.03 & 60.53 & 65.92 & 56.68 \\
Coral & 81.57 & 28.07 & 51.80 & 64.04 & 60.30 & 65.79 & 55.14 \\
DANN & 83.57 & 29.36 & 51.67 & 63.96 & 60.23 & 65.67 & 55.34 \\
SRGNN & 82.14 & 26.42 & 51.58 & 63.96 & 60.27 & 65.64 & 55.08 \\
EERM & 65.71 & 29.91 & OOM & 63.42 & 60.21 & OOM & OOM \\
FLOOD & 84.29 & 28.62 & 54.22 & 64.01 & 60.31 & 65.66 & 58.59 \\
CIT & 83.71 & 28.99 & OOM & 63.77 & 60.05 & OOM & OOM \\
KL-DRO & 81.14 & 29.54 & 51.87 & 64.03 & 60.52 & 65.51 & 54.70 \\
GroupDRO & 82.71 & 29.17 & 52.24 & 64.10 & 60.43 & 65.93 & 56.24 \\
TAR & \underline{87.29} & \underline{30.83} & \underline{57.20} & \underline{64.73} & \underline{61.73} & \underline{66.08} & 59.26 \\
G-$\Delta$UQ & 67.14 & 30.77 & 54.97 & 58.90 & 60.06 & 65.46 & \underline{64.34} \\
\midrule
DSS-Hybrid & \textbf{88.57} & \textbf{40.77} & \textbf{61.21} & \textbf{64.81} & \textbf{62.73} & \textbf{66.76} & \textbf{66.43} \\
\bottomrule
\end{tabular}%
}
\end{table}

\subsection{Standalone versus hybrid: cross-evaluation and deployment guidance}
\label{sec:exp_crosseval}

Each mode is also evaluated on the other's tasks under the same protocols.
Table~\ref{tab:head_to_head} summarizes the outcome, and
Tables~\ref{tab:crosseval_calibration}--\ref{tab:crosseval_good} in
Appendix~\ref{app:crosseval} give the per-dataset numbers with standard
deviations and ID accuracies (protocols in
Appendix~\ref{app:crosseval_protocol}).

\begin{table}[t]
\centering
\small
\caption{Both deployment modes on all three tasks. Each cell summarizes the
full per-dataset comparison in Appendix~\ref{app:crosseval}
(Tables~\ref{tab:crosseval_calibration}, \ref{tab:crosseval_ood},
and~\ref{tab:crosseval_good}); the standalone OOD and GOOD cells and the
hybrid calibration cell are the cross-evaluations; the other cells restate
Tables~\ref{tab:acc_brier}, \ref{tab:gnnsafe_local_table2style}--\ref{tab:gnnsafe_local_table1style},
and~\ref{tab:good_table2_style}.}
\label{tab:head_to_head}
\begin{tabular}{>{\raggedright\arraybackslash}p{0.17\textwidth} >{\raggedright\arraybackslash}p{0.39\textwidth} >{\raggedright\arraybackslash}p{0.34\textwidth}}
\toprule
Task & Standalone DSS-GNN & DSS-Hybrid \\
\midrule
Calibration (14 datasets) & Best Brier on all 14 and best accuracy on 13
(Table~\ref{tab:acc_brier}); stronger than the hybrid on 13 of 14 & Never
loses beyond noise to its own GCN base; Roman-Empire $+27.7$ accuracy,
$-0.33$ Brier \\
OOD detection (11 settings) & Effective on 10 of 11: Cora $79.8/88.2/94.0$,
Photo $96.5$--$98.1$, CS $94.3$--$97.8$, Arxiv $73.7$ (margin objective);
the Twitch energy score does not separate & Best published AUROC on 7 of 11
(Table~\ref{tab:gnnsafe_local_table2style}); stronger than the standalone on
10 of 11 \\
GOOD concept shift (7 settings) & Above every published baseline on 5 of 7,
at ERM level on the other 2 & Strongest on all 7
(Table~\ref{tab:good_table2_style}) \\
\bottomrule
\end{tabular}
\end{table}

\paragraph{Classification and calibration.}
The standalone is the stronger calibrated classifier on 13 of 14 datasets;
Tolokers is the one exception (hybrid $81.2$ accuracy / $0.255$ Brier versus
$79.9$ / $0.268$). Where the
GCN encoder is strong the two are close (Cora $86.6$ vs $84.3$); on
heterophilous graphs the DSS residual compensates for a weak encoder
(Roman-Empire $78.8$ vs $74.4$); and where a plain GCN is unsuitable the
standalone is higher (Texas $91.3$ vs $45.6$, Squirrel $68.1$ vs $26.4$).
Against its own identically trained GCN base, the hybrid never loses beyond
noise on either metric (worst accuracy delta $-0.8$ on Texas, inside one
standard deviation of $7.7$) and gains where the encoder is weak
(Roman-Empire $+27.7$ accuracy / $-0.33$ Brier, Minesweeper $+6.6$ / $-0.11$):
the residual never measurably reduces its encoder's calibration or clean
accuracy, and where the base is unsuitable (small heterophilous graphs) the
hybrid matches it rather than improving on it.

\paragraph{OOD detection.}
The hybrid is the stronger detector on 10 of 11 settings. Its
GCN-plus-residual encoder is stronger in the sparse-label perturbation regime
(Cora-structure $94.3$ vs $79.8$) and for cross-graph transfer (Twitch),
while on Arxiv the standalone run under the GNNSafe++ margin objective is
$0.4$ higher with the best Arxiv ID accuracy ($61.2$ vs $54.0$). The one
standalone exception is Twitch: the classifier trains (ID accuracy
$68.4$ vs $70.5$) but its energy score does not separate the cross-graph OOD
inputs, so the hybrid ($95.8$) is preferred there. On the perturbation settings with
tiny public splits the GNNSafe pipeline gives every backbone BatchNorm, and
plain GCN in the same pipeline does not train without it (ID accuracy $75.7$
with BatchNorm, $45.8$ without, on Cora-structure); the standalone uses one
BatchNorm1d per chaos channel (Appendix~\ref{app:implementation}) in all OOD and GOOD cells but
OOD Twitch and GOOD-Arxiv/degree, and does not need it
under the full-label calibration protocol
(Appendix~\ref{app:bn_calibration}).

\paragraph{Distribution shift.}
On GOOD the standalone is above every baseline of
Table~\ref{tab:good_table2_style} on 5 of 7 settings (WebKB $36.4$ and Twitch
$60.3$ vs TAR $30.8$ and $57.2$) and at ERM's level on the other two (CBAS
$82.4$, Arxiv/time $65.4$), with ID accuracy between $63.5$ and $96.2$
(Appendix~\ref{app:crosseval}); the hybrid is above every baseline on all 7.

\paragraph{Deployment guidance.}
The two modes share one uncertainty representation and the same two
readouts, \eqref{eq:predictive_average} and~\eqref{eq:energy_score}, and
differ in the encoder. Outside the exceptions above, each is the better mode
on its deployment task and remains effective on the other's. We recommend the standalone for calibration-critical single-graph
use, in particular under heterophily, and the hybrid when a strong
conventional encoder exists, for cross-graph transfer, and for OOD detection
in sparse-label regimes; the standalone requires BatchNorm stabilization in
sparse-label regimes, and the hybrid shares its encoder's failures
(Section~\ref{sec:conclusion}).

\section{Conclusion and Limitations}
\label{sec:conclusion}

DSS-GNN is one uncertainty representation, a doubly-spectral stochastic
expansion over graph-frequency and chaos-order axes, with task-matched
readouts and two deployment modes. Empirically, the standalone has the lowest Brier score among
the compared uncertainty-aware baselines on all 14 calibration benchmarks, the hybrid the
strongest AUROC on most node-OOD settings and the strongest shifted accuracy
on all 7 concept-shift benchmarks under ERM, and both modes are effective on
all three tasks with the exceptions noted in Section~\ref{sec:exp_crosseval}.

\paragraph{Limitations.}
(i)~\emph{Two deployment modes.} The modes share one representation and the
same readouts but are not interchangeable in every regime. The standalone
model requires BatchNorm stabilization in sparse-label regimes, where every
baseline backbone in the OOD pipeline also uses it, and uses it on 6 of 7
GOOD settings; under the full-label calibration protocol it is not needed and is not used
(Appendix~\ref{app:bn_calibration}). The hybrid is
preferable when a strong conventional encoder exists and shares that
encoder's limits: on small heterophilous graphs the hybrid matches its base
without improving on it. The standalone energy score does not separate cross-graph OOD inputs on
Twitch, where only the hybrid is effective.
(ii)~\emph{Scalar latent.} All stochastic variation is due to one shared
Gaussian factor, perfectly dependent across nodes, so the logit covariance
across nodes has rank at most $P$. This suffices for the marginal readouts
used here (per-node calibration and energy scores) and is what makes
single-pass quadrature possible, but it cannot represent independent
per-node noise, so joint cross-node uncertainty is not modeled; the
sampling baselines (MC-dropout, Deep Ensembles) have higher Brier than the
standalone and lower AUROC than the hybrid
(Appendix~\ref{app:simple_baselines},
Table~\ref{tab:gnnsafe_local_table2style}). A vector latent of dimension $q$ would cost $S^q$
quadrature nodes on a full tensor grid, which sparse grids mitigate; we leave
this to future work.
(iii)~\emph{Compute.} The chaos expansion adds a $(P{+}1)\times S$ operation
count per layer; the measured per-epoch overhead at $P{=}2$ is $2.9\times$
on Cora and $1.2\times$ on Amazon-Ratings, and about $4\%$ per additional
order on ogbn-arxiv, with memory linear in $P{+}1$ (Table~\ref{tab:timing}).

\label{main:end}

\begin{ack}
This work was supported in part by the National Science Foundation under
grant 2531008. Part of this work was done while Fred Xu was an intern at
Block, Inc., and we thank Block, Inc.\ for its support.
\end{ack}

\bibliographystyle{plainnat}
\bibliography{paper}

@inproceedings{duan2024tfe,
  title     = {Unifying Homophily and Heterophily for Spectral Graph Neural Networks via Triple Filter Ensembles},
  author    = {Duan, Rui and Guang, Mingjian and Wang, Junli and Yan, Chungang and Qi, Hongda and Su, Wenkang and Tian, Can and Yang, Haoran},
  booktitle = {Advances in Neural Information Processing Systems},
  year      = {2024},
  doi       = {10.52202/079017-2966},
  url       = {https://proceedings.neurips.cc/paper_files/paper/2024/hash/a9db2b121c3517fd559ecbe5038701ee-Abstract-Conference.html}
}

@inproceedings{defferrard2016chebnet,
  title     = {Convolutional Neural Networks on Graphs with Fast Localized Spectral Filtering},
  author    = {Defferrard, Micha{\"e}l and Bresson, Xavier and Vandergheynst, Pierre},
  booktitle = {Advances in Neural Information Processing Systems},
  year      = {2016},
  url       = {https://proceedings.neurips.cc/paper/2016/hash/04df4d434d481c5bb723be1b6df1ee65-Abstract.html}
}

@inproceedings{stadler2021gpn,
  title     = {{Graph Posterior Network}: {Bayesian} Predictive Uncertainty for Node Classification},
  author    = {Stadler, Maximilian and Charpentier, Bertrand and Geisler, Simon and Z{\"u}gner, Daniel and G{\"u}nnemann, Stephan},
  booktitle = {Advances in Neural Information Processing Systems},
  year      = {2021},
  url       = {https://proceedings.neurips.cc/paper/2021/hash/95b431e51fc53692913da5263c214162-Abstract.html}
}

@inproceedings{trivedi2024gdeltauq,
  title     = {Accurate and Scalable Estimation of Epistemic Uncertainty for Graph Neural Networks},
  author    = {Trivedi, Puja and Heimann, Mark and Anirudh, Rushil and Koutra, Danai and Thiagarajan, Jayaraman J.},
  booktitle = {International Conference on Learning Representations},
  year      = {2024},
  url       = {https://openreview.net/forum?id=ZL6yd6N1S2}
}

@inproceedings{lin2024gnsd,
  title     = {Graph Neural Stochastic Diffusion for Estimating Uncertainty in Node Classification},
  author    = {Lin, Xixun and Zhang, Wenxiao and Shi, Fengzhao and Zhou, Chuan and Zou, Lixin and Zhao, Xiangyu and Yin, Dawei and Pan, Shirui and Cao, Yanan},
  booktitle = {Proceedings of the 41st International Conference on Machine Learning},
  pages     = {30457--30478},
  volume    = {235},
  series    = {Proceedings of Machine Learning Research},
  publisher = {PMLR},
  year      = {2024},
  url       = {https://proceedings.mlr.press/v235/lin24x.html}
}

@inproceedings{fuchsgruber2024gebm,
  title     = {Energy-based Epistemic Uncertainty for Graph Neural Networks},
  author    = {Fuchsgruber, Dominik and Wollschl{\"a}ger, Tom and G{\"u}nnemann, Stephan},
  booktitle = {Advances in Neural Information Processing Systems},
  year      = {2024},
  doi       = {10.52202/079017-1084},
  url       = {https://proceedings.neurips.cc/paper_files/paper/2024/hash/3cd50f2922b7adaaa9e5113e35bae095-Abstract-Conference.html}
}

@inproceedings{gal2016dropout,
  title     = {Dropout as a {Bayesian} Approximation: Representing Model Uncertainty in Deep Learning},
  author    = {Gal, Yarin and Ghahramani, Zoubin},
  booktitle = {Proceedings of the 33rd International Conference on Machine Learning},
  pages     = {1050--1059},
  volume    = {48},
  series    = {Proceedings of Machine Learning Research},
  publisher = {PMLR},
  year      = {2016},
  url       = {https://proceedings.mlr.press/v48/gal16.html}
}

@inproceedings{lakshminarayanan2017simple,
  title     = {Simple and Scalable Predictive Uncertainty Estimation using Deep Ensembles},
  author    = {Lakshminarayanan, Balaji and Pritzel, Alexander and Blundell, Charles},
  booktitle = {Advances in Neural Information Processing Systems},
  year      = {2017},
  url       = {https://proceedings.neurips.cc/paper_files/paper/2017/hash/9ef2ed4b7fd2c810847ffa5fa85bce38-Abstract.html}
}

@article{xiu_wiener,
  author  = {Xiu, Dongbin and Karniadakis, George Em},
  title   = {The {Wiener--Askey} Polynomial Chaos for Stochastic Differential Equations},
  journal = {SIAM Journal on Scientific Computing},
  volume  = {24},
  number  = {2},
  pages   = {619--644},
  year    = {2002},
  doi     = {10.1137/S1064827501387826},
  url     = {https://doi.org/10.1137/S1064827501387826},
}

@inproceedings{shchur2018pitfalls,
  title     = {Pitfalls of Graph Neural Network Evaluation},
  author    = {Shchur, Oleksandr and Mumme, Maximilian and Bojchevski, Aleksandar and G{\"u}nnemann, Stephan},
  booktitle = {Relational Representation Learning Workshop, NeurIPS 2018},
  year      = {2018},
  doi       = {10.48550/arXiv.1811.05868},
  url       = {https://arxiv.org/abs/1811.05868}
}

@inproceedings{pei2020geomgcn,
  title     = {Geom-{GCN}: Geometric Graph Convolutional Networks},
  author    = {Pei, Hongbin and Wei, Bingzhe and Chang, Kevin Chen-Chuan and Lei, Yu and Yang, Bo},
  booktitle = {International Conference on Learning Representations},
  year      = {2020},
  url       = {https://openreview.net/forum?id=S1e2agrFvS}
}

@inproceedings{platonov2023critical,
  title     = {A critical look at the evaluation of {GNNs} under heterophily: Are we really making progress?},
  author    = {Platonov, Oleg and Kuznedelev, Denis and Diskin, Michael and Babenko, Artem and Prokhorenkova, Liudmila},
  booktitle = {International Conference on Learning Representations},
  year      = {2023},
  url       = {https://openreview.net/forum?id=tJbbQfw-5wv}
}

@inproceedings{liu2020energy,
  title     = {Energy-based Out-of-distribution Detection},
  author    = {Liu, Weitang and Wang, Xiaoyun and Owens, John and Li, Yixuan},
  booktitle = {Advances in Neural Information Processing Systems},
  year      = {2020},
  url       = {https://proceedings.neurips.cc/paper/2020/hash/f5496252609c43eb8a3d147ab9b9c006-Abstract.html}
}

@inproceedings{kipf2017gcn,
  title     = {Semi-Supervised Classification with Graph Convolutional Networks},
  author    = {Kipf, Thomas N. and Welling, Max},
  booktitle = {International Conference on Learning Representations},
  year      = {2017},
  url       = {https://openreview.net/forum?id=SJU4ayYgl}
}

@inproceedings{chien2021gprgnn,
  title     = {Adaptive Universal Generalized {PageRank} Graph Neural Network},
  author    = {Chien, Eli and Peng, Jianhao and Li, Pan and Milenkovic, Olgica},
  booktitle = {International Conference on Learning Representations},
  year      = {2021},
  url       = {https://openreview.net/forum?id=n6jl7fLxrP}
}

@inproceedings{he2021bernnet,
  title     = {{BernNet}: Learning Arbitrary Graph Spectral Filters via {Bernstein} Approximation},
  author    = {He, Mingguo and Wei, Zhewei and Huang, Zengfeng and Xu, Hongteng},
  booktitle = {Advances in Neural Information Processing Systems},
  year      = {2021},
  url       = {https://proceedings.neurips.cc/paper/2021/hash/76f1cfd7754a6e4fc3281bcccb3d0902-Abstract.html}
}

@article{shuman2013gsp,
  title     = {The emerging field of signal processing on graphs: Extending high-dimensional data analysis to networks and other irregular domains},
  author    = {Shuman, David I. and Narang, Sunil K. and Frossard, Pascal and Ortega, Antonio and Vandergheynst, Pierre},
  journal   = {IEEE Signal Processing Magazine},
  volume    = {30},
  number    = {3},
  pages     = {83--98},
  year      = {2013},
  doi       = {10.1109/MSP.2012.2235192},
  url       = {https://doi.org/10.1109/MSP.2012.2235192}
}

@book{nualart2006malliavin,
  title     = {The Malliavin Calculus and Related Topics},
  author    = {Nualart, David},
  publisher = {Springer Berlin Heidelberg},
  series    = {Probability and Its Applications},
  edition   = {2nd},
  year      = {2006},
  doi       = {10.1007/3-540-28329-3},
  url       = {https://doi.org/10.1007/3-540-28329-3}
}

@book{janson1997gaussian,
  title     = {Gaussian Hilbert Spaces},
  author    = {Janson, Svante},
  publisher = {Cambridge University Press},
  year      = {1997},
  doi       = {10.1017/CBO9780511526169},
  url       = {https://doi.org/10.1017/CBO9780511526169}
}

@article{bo2021fagcn,
  title     = {Beyond Low-frequency Information in Graph Convolutional Networks},
  author    = {Bo, Deyu and Wang, Xiao and Shi, Chuan and Shen, Huawei},
  journal   = {Proceedings of the AAAI Conference on Artificial Intelligence},
  volume    = {35},
  number    = {5},
  pages     = {3950--3957},
  year      = {2021},
  doi       = {10.1609/aaai.v35i5.16514},
  url       = {https://ojs.aaai.org/index.php/AAAI/article/view/16514}
}

@inproceedings{wu2023gnnsafe,
  title     = {Energy-based Out-of-Distribution Detection for Graph Neural Networks},
  author    = {Wu, Qitian and Chen, Yiting and Yang, Chenxiao and Yan, Junchi},
  booktitle = {International Conference on Learning Representations},
  year      = {2023},
  url       = {https://openreview.net/forum?id=zoz7Ze4STUL}
}

@inproceedings{sagawa2020groupdro,
  title     = {Distributionally Robust Neural Networks},
  author    = {Sagawa, Shiori and Koh, Pang Wei and Hashimoto, Tatsunori B. and Liang, Percy},
  booktitle = {International Conference on Learning Representations},
  year      = {2020},
  url       = {https://openreview.net/forum?id=ryxGuJrFvS}
}

@article{arjovsky2019irm,
  title     = {Invariant Risk Minimization},
  author    = {Arjovsky, Martin and Bottou, L{\'e}on and Gulrajani, Ishaan and Lopez-Paz, David},
  journal   = {arXiv preprint arXiv:1907.02893},
  year      = {2019},
  doi       = {10.48550/arXiv.1907.02893},
  url       = {https://arxiv.org/abs/1907.02893}
}

@inproceedings{krueger2021vrex,
  title     = {Out-of-Distribution Generalization via Risk Extrapolation ({REx})},
  author    = {Krueger, David and Caballero, Ethan and Jacobsen, Joern-Henrik and Zhang, Amy and Binas, Jonathan and Zhang, Dinghuai and Le Priol, Remi and Courville, Aaron},
  booktitle = {Proceedings of the 38th International Conference on Machine Learning},
  pages     = {5815--5826},
  volume    = {139},
  series    = {Proceedings of Machine Learning Research},
  publisher = {PMLR},
  year      = {2021},
  url       = {https://proceedings.mlr.press/v139/krueger21a.html}
}

@inproceedings{gui2022good,
  title     = {{GOOD}: A Graph Out-of-Distribution Benchmark},
  author    = {Gui, Shurui and Li, Xiner and Wang, Limei and Ji, Shuiwang},
  booktitle = {Advances in Neural Information Processing Systems},
  year      = {2022},
  url       = {https://proceedings.neurips.cc/paper_files/paper/2022/hash/0dc91de822b71c66a7f54fa121d8cbb9-Abstract-Datasets_and_Benchmarks.html}
}

@inproceedings{tar,
  title     = {Topology-Aware Dynamic Reweighting for Distribution Shifts on Graph},
  author    = {Zheng, Weihuang and Liu, Jiashuo and Li, Jiaxing and Wu, Jiayun and Cui, Peng and Kong, Youyong},
  booktitle = {Proceedings of the 42nd International Conference on Machine Learning},
  pages     = {78260--78275},
  volume    = {267},
  series    = {Proceedings of Machine Learning Research},
  publisher = {PMLR},
  year      = {2025},
  url       = {https://proceedings.mlr.press/v267/zheng25k.html}
}

@article{cameron1947orthogonal,
  title     = {The Orthogonal Development of Non-Linear Functionals in Series of {F}ourier-{H}ermite Functionals},
  author    = {Cameron, Robert H. and Martin, William T.},
  journal   = {Annals of Mathematics},
  volume    = {48},
  number    = {2},
  pages     = {385--392},
  year      = {1947},
  doi       = {10.2307/1969178},
  url       = {https://doi.org/10.2307/1969178}
}

@inproceedings{velickovic2018gat,
  title     = {Graph Attention Networks},
  author    = {Veli{\v{c}}kovi{\'c}, Petar and Cucurull, Guillem and Casanova, Arantxa and Romero, Adriana and Li{\`o}, Pietro and Bengio, Yoshua},
  booktitle = {International Conference on Learning Representations},
  year      = {2018},
  url       = {https://openreview.net/forum?id=rJXMpikCZ}
}

@article{wiener1938homogeneous,
  title     = {The Homogeneous Chaos},
  author    = {Wiener, Norbert},
  journal   = {American Journal of Mathematics},
  volume    = {60},
  number    = {4},
  pages     = {897--936},
  year      = {1938},
  doi       = {10.2307/2371268},
  url       = {https://doi.org/10.2307/2371268}
}

@article{nelson1973free,
  title     = {The free {Markoff} field},
  author    = {Nelson, Edward},
  journal   = {Journal of Functional Analysis},
  volume    = {12},
  number    = {2},
  pages     = {211--227},
  year      = {1973},
  doi       = {10.1016/0022-1236(73)90025-6},
  url       = {https://doi.org/10.1016/0022-1236(73)90025-6}
}

@article{ortega2018graph,
  title     = {Graph Signal Processing: Overview, Challenges, and Applications},
  author    = {Ortega, Antonio and Frossard, Pascal and Kova{\v{c}}evi{\'c}, Jelena and Moura, Jos{\'e} M. F. and Vandergheynst, Pierre},
  journal   = {Proceedings of the IEEE},
  volume    = {106},
  number    = {5},
  pages     = {808--828},
  year      = {2018},
  doi       = {10.1109/JPROC.2018.2820126},
  url       = {https://doi.org/10.1109/JPROC.2018.2820126}
}

@book{szego1975orthogonal,
  title     = {Orthogonal Polynomials},
  author    = {Szeg{\H{o}}, G{\'a}bor},
  publisher = {American Mathematical Society},
  edition   = {4th},
    year      = {1975}
}

@article{golub1969calculation,
  title     = {Calculation of {G}auss quadrature rules},
  author    = {Golub, Gene H. and Welsch, John H.},
  journal   = {Mathematics of Computation},
  volume    = {23},
  number    = {106},
  pages     = {221--230},
  year      = {1969},
  doi       = {10.1090/S0025-5718-69-99647-1},
  url       = {https://doi.org/10.1090/S0025-5718-69-99647-1}
}

@article{stokes2020deep,
  title     = {A Deep Learning Approach to Antibiotic Discovery},
  author    = {Stokes, Jonathan M. and Yang, Kevin and Swanson, Kyle and Jin, Wengong and Cubillos-Ruiz, Andres and Donghia, Nina M. and MacNair, Craig R. and French, Shawn and Carfrae, Lindsey A. and Bloom-Ackermann, Zohar and others},
  journal   = {Cell},
  volume    = {180},
  number    = {4},
  pages     = {688--702.e13},
  year      = {2020},
  doi       = {10.1016/j.cell.2020.01.021},
  url       = {https://doi.org/10.1016/j.cell.2020.01.021}
}

@inproceedings{dou2020caregnn,
  title     = {Enhancing Graph Neural Network-based Fraud Detectors against Camouflaged Fraudsters},
  author    = {Dou, Yingtong and Liu, Zhiwei and Sun, Li and Deng, Yutong and Peng, Hao and Yu, Philip S.},
  booktitle = {Proceedings of the 29th ACM International Conference on Information \& Knowledge Management},
  pages     = {315--324},
  publisher = {ACM},
  year      = {2020},
  doi       = {10.1145/3340531.3411903},
  url       = {https://doi.org/10.1145/3340531.3411903}
}

@inproceedings{li2018dcrnn,
  title     = {Diffusion Convolutional Recurrent Neural Network: Data-Driven Traffic Forecasting},
  author    = {Li, Yaguang and Yu, Rose and Shahabi, Cyrus and Liu, Yan},
  booktitle = {International Conference on Learning Representations},
  year      = {2018},
  url       = {https://openreview.net/forum?id=SJiHXGWAZ}
}

@article{oladyshkin2023deeppce,
  title     = {The deep arbitrary polynomial chaos neural network or how {Deep Artificial Neural Networks} could benefit from data-driven homogeneous chaos theory},
  author    = {Oladyshkin, Sergey and Praditia, Timothy and Kr{\"o}ker, Ilja and Mohammadi, Farid and Nowak, Wolfgang and Otte, Sebastian},
  journal   = {Neural Networks},
  volume    = {166},
  pages     = {85--104},
  year      = {2023},
  doi       = {10.1016/j.neunet.2023.06.036},
  url       = {https://doi.org/10.1016/j.neunet.2023.06.036}
}

@article{brier1950,
  title     = {Verification of Forecasts Expressed in Terms of Probability},
  author    = {Brier, Glenn W.},
  journal   = {Monthly Weather Review},
  volume    = {78},
  number    = {1},
  pages     = {1--3},
  year      = {1950},
  doi       = {10.1175/1520-0493(1950)078<0001:VOFEIT>2.0.CO;2},
  url       = {https://doi.org/10.1175/1520-0493(1950)078\%3C0001:VOFEIT\%3E2.0.CO;2}
}

\clearpage
\appendix

\addtocontents{toc}{\protect\setcounter{tocdepth}{2}}
\renewcommand{\contentsname}{Appendix: Table of Contents}
\tableofcontents

\clearpage
\section{Proofs and Derivations}
\label{sec:proofs}

We begin with a self-contained introduction to Wiener chaos and Gauss--Hermite
quadrature (Appendix~\ref{sec:chaos_background}), then present supporting
technical results (Appendix~\ref{sec:supporting_results}), followed by
individual proofs of each main-text theorem
(Appendices~\ref{sec:proof_generalization}--\ref{sec:proof_readouts}).

\subsection{Background: Wiener chaos and Gauss--Hermite quadrature}
\label{sec:chaos_background}

This section collects the probability and approximation theory underlying the
doubly-spectral expansion. The material is classical; we follow the
presentation in \citet{janson1997gaussian} and
\citet{xiu_wiener}. This appendix is self-contained.

\paragraph{Gaussian measure and Hermite polynomials.}
Let $\gamma$ denote the standard Gaussian measure on~$\mathbb{R}$, i.e.,
$d\gamma(\omega) = (2\pi)^{-1/2}e^{-\omega^2/2}\,d\omega$. The Hilbert space
$L^2(\mathbb{R},\gamma)$ consists of all square-integrable functions with
respect to~$\gamma$, equipped with the inner product
$\langle f,g\rangle = \int_{\mathbb{R}} f(\omega)\,g(\omega)\,d\gamma(\omega)
= \mathbb{E}[f(\omega)\,g(\omega)]$ where $\omega\sim\mathcal{N}(0,1)$.

The \emph{probabilists' Hermite polynomials}~\citep{szego1975orthogonal} are
defined by the Rodrigues formula
\begin{equation}
\mathrm{He}_n(\omega)
= (-1)^n\,e^{\omega^2/2}\,\frac{d^n}{d\omega^n}\,e^{-\omega^2/2},
\qquad n=0,1,2,\ldots
\label{eq:rodrigues}
\end{equation}
The first few are $\mathrm{He}_0(\omega)=1$,
$\mathrm{He}_1(\omega)=\omega$,
$\mathrm{He}_2(\omega)=\omega^2-1$,
$\mathrm{He}_3(\omega)=\omega^3-3\omega$. They satisfy the three-term
recurrence
\begin{equation}
\mathrm{He}_{n+1}(\omega)
= \omega\,\mathrm{He}_n(\omega) - n\,\mathrm{He}_{n-1}(\omega),
\label{eq:hermite_recurrence}
\end{equation}
with $\mathrm{He}_{-1}=0$, and the orthogonality relation
$\mathbb{E}[\mathrm{He}_m(\omega)\,\mathrm{He}_n(\omega)] = n!\,\delta_{mn}$.
Normalizing gives the \emph{orthonormal Hermite basis}
\begin{equation}
\Psi_n(\omega) := \frac{\mathrm{He}_n(\omega)}{\sqrt{n!}},
\qquad
\mathbb{E}[\Psi_m\Psi_n] = \delta_{mn},
\label{eq:orthonormal_hermite}
\end{equation}
which forms a complete orthonormal system for
$L^2(\mathbb{R},\gamma)$~\citep{szego1975orthogonal,janson1997gaussian}. Every
$f\in L^2(\mathbb{R},\gamma)$ can therefore be expanded as
$f(\omega)=\sum_{n=0}^{\infty}f_n\,\Psi_n(\omega)$ with
$f_n=\mathbb{E}[f\Psi_n]$ and Parseval's identity
$\|f\|_{L^2(\gamma)}^2=\sum_n f_n^2$.

\paragraph{Wiener--It\^{o} chaos decomposition.}
The Hermite expansion generalizes to a fundamental structural result in
Gaussian analysis. Let $(\Omega,\mathcal{F},\mathbb{P})$ be a probability space
carrying a standard Gaussian random variable~$\omega$, and define the
associated Gaussian subspace
\[
L_\omega^2 := \{f(\omega): f\in L^2(\mathbb{R},\gamma)\}
\subseteq L^2(\Omega,\mathbb{P}).
\]
Define the $n$-th \emph{Wiener chaos} $\mathcal{C}_n$ as the closed linear span
of $\Psi_n(\omega)$ in $L_\omega^2$. Since we use a single Gaussian latent
($q=1$), each $\mathcal{C}_n$ is one-dimensional and spanned by
$\Psi_n(\omega)$; with $q>1$, the corresponding chaos subspace would be indexed
by Hermite multi-indices. The \emph{Wiener--It\^{o} chaos decomposition}
states~\citep{wiener1938homogeneous,cameron1947orthogonal,janson1997gaussian}:
\begin{equation}
L_\omega^2
= \bigoplus_{n=0}^{\infty}\mathcal{C}_n,
\label{eq:wiener_ito}
\end{equation}
that is, every square-integrable random variable measurable with respect to
$\omega$ admits a unique, orthogonal expansion in the Hermite basis. Denoting
the orthogonal projection onto $\mathcal{C}_n$ by~$P_n$, any $f\in L_\omega^2$
satisfies
$f=\sum_{n=0}^{\infty}P_n f$ with $\|f\|^2=\sum_n\|P_n f\|^2$. The zeroth
chaos $\mathcal{C}_0$ consists of constants ($P_0 f = \mathbb{E}[f]$); the
first chaos is the space of centered Gaussian random variables; higher chaoses
capture progressively more complex nonlinear functionals of~$\omega$.

\begin{remark}[Analogy with graph spectral decomposition]
\label{rem:analogy}
The decomposition~\eqref{eq:wiener_ito} is the stochastic counterpart of the
graph spectral decomposition
$\ell^2(V) = \bigoplus_{j=1}^{N}\mathrm{span}(u_j)$. Both decompose a
Hilbert space into eigenspaces of a self-adjoint operator: the graph Laplacian
$L_G$ for $\ell^2(V)$, and the Ornstein--Uhlenbeck operator~$\mathcal{L}$
(defined below) for the Gaussian Hilbert space~$L_\omega^2$.
\end{remark}

\begin{remark}[Completeness of the doubly-spectral basis]
\label{rem:completeness}
Take the signal domain to be the random graph signals generated by the shared
latent: a random variable per node with finite second moment, measurable
with respect to $\omega$, i.e. an element of the tensor-product Hilbert space
$\ell^2(V)\otimes L^2(\mathbb{R},\gamma)$. The graph Fourier eigenvectors
$\{u_j\}_{j=1}^N$ form an orthonormal basis of $\ell^2(V)$ and, by the
Wiener--It\^{o} decomposition~\eqref{eq:wiener_ito}, the normalized Hermite
polynomials $\{\Psi_n\}_{n\ge0}$ form an orthonormal basis of
$L^2(\mathbb{R},\gamma)$. The tensor product of orthonormal bases is an
orthonormal basis of the tensor-product space, so
$\{u_j\otimes\Psi_n\}_{j\le N,\,n\ge0}$ is a complete orthonormal basis of
this domain, and the double expansion~\eqref{eq:doubly_spectral_expansion}
is a bi-spectral decomposition in the strict sense: both factors are
eigenbases of self-adjoint operators (Remark~\ref{rem:analogy}). The scope
caveat is that the space is built over the scalar $\omega$: signals with
independent per-node latents are outside it
(Section~\ref{sec:conclusion}).
\end{remark}

\paragraph{The Ornstein--Uhlenbeck operator and hypercontractivity.}
The \emph{Ornstein--Uhlenbeck (O-U) operator}
$\mathcal{L}$ is defined on its standard dense domain in
$L^2(\mathbb{R},\gamma)$ by
\begin{equation}
\mathcal{L}f(\omega) = f''(\omega) - \omega\,f'(\omega),
\label{eq:ou_operator}
\end{equation}
with eigenvalues and eigenfunctions
$\mathcal{L}\Psi_n = -n\,\Psi_n$~\citep{nualart2006malliavin}. This is
directly analogous to the graph Laplacian having eigenpairs
$(u_j,\lambda_j)$: $-\mathcal{L}$ is the corresponding operator on
$L^2(\mathbb{R},\gamma)$, with the chaos orders $0,1,2,\ldots$ as eigenvalues. The
associated \emph{O-U semigroup} $T_t = e^{t\mathcal{L}}$ acts as
$T_t\Psi_n = e^{-nt}\Psi_n$, damping higher-order chaos components
exponentially.

A fundamental property of $T_t$ is \emph{hypercontractivity}: for $1<p\le q$
and $e^{2t}\ge (q-1)/(p-1)$,
\begin{equation}
\|T_t f\|_{L^q(\gamma)} \le \|f\|_{L^p(\gamma)}
\label{eq:hypercontractivity}
\end{equation}
for all $f\in L^p(\gamma)$~\citep{nelson1973free,janson1997gaussian}.
A direct consequence for chaos components is the \emph{moment comparison
inequality}: for every $g\in\mathcal{C}_n$ and $q\ge 2$,
\begin{equation}
\|g\|_{L^q(\gamma)} \le (q-1)^{n/2}\,\|g\|_{L^2(\gamma)}.
\label{eq:moment_comparison}
\end{equation}
This follows by applying~\eqref{eq:hypercontractivity} with $p=2$ and
$\rho=e^{-t}=(q-1)^{-1/2}$ to $g\in\mathcal{C}_n$, using
$T_t g = e^{-nt}g$~\citep[Corollary~5.11]{janson1997gaussian}.
Equation~\eqref{eq:moment_comparison} bounds the higher moments of a fixed
chaos component; the exponential \emph{decay} of chaos coefficients with
order~$n$ requires additional regularity and is established separately in
Theorem~\ref{thm:universal_approx} via a Cauchy-estimate argument.

\paragraph{Gauss--Hermite quadrature.}
To evaluate integrals against~$\gamma$ numerically, we use Gauss--Hermite
quadrature~\citep{golub1969calculation,xiu_wiener}. An $S$-point rule
$\{(\omega_s,\mu_s)\}_{s=1}^S$ approximates
\begin{equation}
\int_{\mathbb{R}}f(\omega)\,d\gamma(\omega)
\;\approx\;
\sum_{s=1}^{S}\mu_s\,f(\omega_s),
\label{eq:gauss_hermite}
\end{equation}
where the nodes $\omega_1<\cdots<\omega_S$ are the roots of
$\mathrm{He}_S(\omega)$ and the weights $\mu_s>0$ sum to one. The main
property is \emph{polynomial exactness}: the approximation is exact whenever
$f$ is a polynomial of degree at most $2S-1$~\citep{szego1975orthogonal}. In
particular:
\begin{itemize}
\item Orthonormality:
$\sum_{s=1}^{S}\mu_s\,\Psi_m(\omega_s)\,\Psi_n(\omega_s) = \delta_{mn}$
for all $m+n \le 2S-1$.
\item Triple products:
$\sum_{s=1}^{S}\mu_s\,\Psi_r(\omega_s)\Psi_n(\omega_s)\Psi_m(\omega_s) =
\mathbb{E}[\Psi_r\Psi_n\Psi_m]$ whenever $r+n+m\le 2S-1$.
\end{itemize}
For DSS-GNN with chaos order~$P$ and gate order~$P_g$, the most demanding
integrand is the triple product $\Psi_r\Psi_n\Psi_m$ with
$r\le P_g$, $n,m\le P$, which has degree $P_g+2P$. Exactness therefore
requires $S\ge\lceil(P_g+2P+1)/2\rceil$ quadrature nodes. For $P_g\le 1$ the
threshold is $S\ge P+1$; the experiments use $S=4$, which is exact for
$P\le 3$.

\begin{remark}[Non-intrusive vs.\ intrusive projection]
\label{rem:nonintrusive}
In the polynomial chaos literature~\citep{xiu_wiener}, two strategies compute
the chaos coefficients of a transformed variable $\sigma(H(\omega))$:
(i)~\emph{intrusive} methods analytically propagate the expansion through
$\sigma$ using triple-product coefficients (Lemma~\ref{lem:product_triple});
(ii)~\emph{non-intrusive} methods evaluate $\sigma$ at quadrature nodes and
project back~\eqref{eq:gauss_hermite}. DSS-GNN uses the non-intrusive
approach~\eqref{eq:dss_projection_update} because it applies unchanged to
\emph{any} pointwise activation~$\sigma$ without requiring closed-form
triple-product expansions. Theorem~\ref{thm:nonintrusive_equals_intrusive}
proves exact agreement for product terms $\alpha(\omega)H(\omega)$ when the
quadrature rule integrates the relevant triple products exactly; for a general
nonlinear activation~$\sigma$, the non-intrusive projection remains an
approximation whose error is decomposed in
Proposition~\ref{prop:projection_error}.
\end{remark}

\subsection{Supporting technical results}
\label{sec:supporting_results}

\begin{lemma}[Products in chaos: triple-product contraction]
\label{lem:product_triple}
Let $\alpha(\omega)=\sum_{r=0}^{P_g} a_r\Psi_r(\omega)$ and
$H(\omega)=\sum_{n=0}^{P} H_n\Psi_n(\omega)$.
Then the order-$m$ coefficient of the product $\alpha(\omega)H(\omega)$ is
\[
(\alpha H)_m = \sum_{r=0}^{P_g}\sum_{n=0}^{P} a_r H_n\,c_{rnm},
\qquad
c_{rnm}:=\mathbb{E}[\Psi_r\Psi_n\Psi_m].
\]
\end{lemma}

\begin{proof}
Expand the product in the Hermite basis:
\[
\alpha(\omega)H(\omega)
= \Bigl(\sum_{r=0}^{P_g} a_r\Psi_r(\omega)\Bigr)
  \Bigl(\sum_{n=0}^{P} H_n\Psi_n(\omega)\Bigr)
= \sum_{r,n} a_r H_n\,\Psi_r(\omega)\Psi_n(\omega).
\]
The order-$m$ coefficient is obtained by projecting onto $\Psi_m$ using the
orthonormality $\mathbb{E}[\Psi_i\Psi_j]=\delta_{ij}$:
\[
(\alpha H)_m
= \mathbb{E}\bigl[\alpha(\omega)H(\omega)\Psi_m(\omega)\bigr]
= \sum_{r,n} a_r H_n\,\mathbb{E}[\Psi_r\Psi_n\Psi_m]
= \sum_{r,n} a_r H_n\,c_{rnm}.
\]
The triple-product constants $c_{rnm} = \mathbb{E}[\Psi_r\Psi_n\Psi_m]$ are
computable in closed form from the three-term recurrence of the probabilists'
Hermite polynomials~\citep{xiu_wiener}. In particular,
$c_{rnm} = 0$ unless $|r-n| \le m \le r+n$ and $r+n+m$ is even, which makes
the coupling matrices sparse.
\end{proof}

\begin{theorem}[Non-intrusive projection recovers product coefficients when quadrature is exact]
\label{thm:nonintrusive_equals_intrusive}
Let $\widehat{(\alpha H)}_m := \sum_{s=1}^{S} \mu_s\,\alpha(\omega_s)H(\omega_s)\Psi_m(\omega_s)$
be the discrete projection of the product using a Gauss--Hermite quadrature
rule $\{(\omega_s,\mu_s)\}_{s=1}^S$.
If the quadrature integrates all triple products
$\Psi_r(\omega)\Psi_n(\omega)\Psi_m(\omega)$ exactly for
$r\le P_g$, $n\le P$, $m\le P$, then
$\widehat{(\alpha H)}_m = (\alpha H)_m$ for all $m\le P$.
\end{theorem}

\begin{proof}
Substituting the expansions of $\alpha$ and $H$:
\begin{align*}
\widehat{(\alpha H)}_m
&= \sum_{s=1}^{S} \mu_s\,\alpha(\omega_s)H(\omega_s)\Psi_m(\omega_s) \\
&= \sum_{s=1}^{S} \mu_s \sum_{r=0}^{P_g}\sum_{n=0}^{P} a_r H_n\,
   \Psi_r(\omega_s)\Psi_n(\omega_s)\Psi_m(\omega_s) \\
&= \sum_{r,n} a_r H_n \underbrace{\sum_{s=1}^{S} \mu_s\,
   \Psi_r(\omega_s)\Psi_n(\omega_s)\Psi_m(\omega_s)}_{\displaystyle
   Q_S[\Psi_r\Psi_n\Psi_m]}.
\end{align*}
The integrand $\Psi_r\Psi_n\Psi_m$ is a polynomial of degree $r+n+m \le P_g + 2P$
under the Gaussian weight $\gamma$.  An $S$-point Gauss--Hermite rule integrates
polynomials of degree up to $2S-1$
exactly~\citep{xiu_wiener}. The condition $2S - 1 \ge P_g + 2P$, i.e.\
$S \ge \lceil(P_g + 2P + 1)/2\rceil$, therefore ensures
$Q_S[\Psi_r\Psi_n\Psi_m] = \mathbb{E}[\Psi_r\Psi_n\Psi_m] = c_{rnm}$.
By Lemma~\ref{lem:product_triple},
$\widehat{(\alpha H)}_m = (\alpha H)_m$.
\end{proof}

\begin{proposition}[Discrete projection error decomposition]
\label{prop:projection_error}
Let $\Pi_P$ be the exact $L^2(\mathbb{R},\gamma)$ projection onto the first $P+1$
Wiener chaos orders, and let $\Pi_{P,S}Q$ be the finite chaos expansion whose
coefficients are induced by a quadrature rule:
$(\Pi_{P,S}Q)_m := \sum_{s=1}^S \mu_s\, Q(\omega_s)\Psi_m(\omega_s)$.
Then for any $Q\in L^2(\mathbb{R},\gamma)$,
\[
\|Q-\Pi_{P,S}Q\|_{L^2}
\;\le\;
\underbrace{\|Q-\Pi_P Q\|_{L^2}}_{\text{truncation error}}
+
\underbrace{\|\Pi_P Q-\Pi_{P,S}Q\|_{L^2}}_{\text{quadrature error}}.
\]
Moreover, Gauss--Hermite rules yield exact coefficients whenever
$Q(\omega)\Psi_m(\omega)$ is a polynomial of sufficiently low degree under the
Gaussian weight.
\end{proposition}

\begin{proof}
The triangle inequality in $L^2(\mathbb{R},\gamma)$ gives
$\|Q - \Pi_{P,S}Q\| \le \|Q - \Pi_P Q\| + \|\Pi_P Q - \Pi_{P,S}Q\|$.
For the truncation term, $\Pi_P$ is the orthogonal projection onto
$\bigoplus_{n=0}^{P}\mathcal{C}_n$, the direct sum of the first $P+1$ Wiener
chaos subspaces~\citep{janson1997gaussian}, so
$\|Q - \Pi_P Q\|^2 = \sum_{n>P}\|Q_n\|^2$ where $Q_n = \mathbb{E}[Q\Psi_n]$
are the exact chaos coefficients. For the quadrature term, the $m$-th
coefficient of $\Pi_{P,S}Q$ is
$\sum_s \mu_s Q(\omega_s)\Psi_m(\omega_s)$, which approximates
$\int Q\Psi_m\,d\gamma$. When $Q\Psi_m$ is a polynomial of degree at most
$2S-1$, Gauss--Hermite quadrature integrates it
exactly~\citep{xiu_wiener}, making the quadrature error zero. In the DSS layer
context, the pre-activation $\widetilde H^{(\ell+1)}(\omega)$ before
applying~$\sigma$ is a polynomial of degree at most $P_g + P$
in~$\omega$; after applying $\sigma$ it is no longer polynomial, and the
quadrature error reflects the approximation quality of the $S$-point rule on
this non-polynomial integrand.
\end{proof}

\subsection{Proof of Proposition~\ref{prop:generalization} (Generalization)}
\label{sec:proof_generalization}

\propgeneralization*

\begin{proof}
When $P=0$, the chaos expansion~\eqref{eq:pce_state} contains a single term:
$H^{(\ell)}(\omega) = H_0^{(\ell)}\Psi_0(\omega)$. Since the zeroth
normalized Hermite polynomial satisfies $\Psi_0 \equiv 1$~\citep{xiu_wiener},
the representation is deterministic:
$H^{(\ell)}(\omega) = H_0^{(\ell)}$ for all $\omega \in \Omega$.

\textbf{Step 1 (Gate reduction).}
With $P=0$, the mixing weights~\eqref{eq:random_gates} are indexed by $n=0$
only. Setting the gate order $P_g = 0$ (no stochastic gating when the
representation itself is deterministic) yields scalar constants
$\alpha_0 = a_{0,0}\in\mathbb{R}$ and $\beta_0 = b_{0,0}\in\mathbb{R}$.

\textbf{Step 2 (Pre-activation collapse).}
The quadrature-node pre-activation~\eqref{eq:sample_preact} reduces to
\begin{align*}
\widetilde H^{(\ell+1)}(\omega_s)
&= \Psi_0(\omega_s)\bigl(\alpha_0\,G_{\mathrm{lp}}(L_G)H_0^{(\ell)}
  + \beta_0\,G_{\mathrm{hp}}(L_G)H_0^{(\ell)}\bigr)W_\ell \\
&= \bigl(\alpha_0\,G_{\mathrm{lp}}(L_G)
  + \beta_0\,G_{\mathrm{hp}}(L_G)\bigr)H_0^{(\ell)}W_\ell,
\end{align*}
which is independent of the quadrature node $\omega_s$.

\textbf{Step 3 (Projection collapse).}
The projection~\eqref{eq:dss_projection_update} for $m=0$ becomes
\[
H_0^{(\ell+1)}
= \sum_{s=1}^{S}\mu_s\,\sigma\!\bigl(\widetilde H^{(\ell+1)}(\omega_s)\bigr)
  \Psi_0(\omega_s)
= \sigma\!\bigl(\widetilde H^{(\ell+1)}\bigr)\sum_{s=1}^{S}\mu_s
= \sigma\!\bigl(\widetilde H^{(\ell+1)}\bigr),
\]
where the second equality uses the fact that $\widetilde H^{(\ell+1)}$ does not
depend on~$\omega_s$ (Step~2), and the third uses the normalization
$\sum_s\mu_s = \int_{\mathbb{R}} d\gamma = 1$ of the Gauss--Hermite
weights~\citep{xiu_wiener}.

\textbf{Step 4 (Recovery of spectral GNN).}
Combining Steps~2 and~3, the full DSS layer update at $P=0$ is
\[
H_0^{(\ell+1)}
= \sigma\!\bigl(g(L_G)\,H_0^{(\ell)}\,W_\ell\bigr),
\qquad
g(L_G) := \alpha_0\,G_{\mathrm{lp}}(L_G) + \beta_0\,G_{\mathrm{hp}}(L_G).
\]
Since $G_{\mathrm{lp}}$ and $G_{\mathrm{hp}}$ are Chebyshev polynomials
in~$L_G$~\eqref{eq:cheby_filters}, their linear combination $g(L_G)$ is a
polynomial in~$L_G$ of degree $\max(K_{\mathrm{lp}},K_{\mathrm{hp}})$.
This is exactly the form of a polynomial spectral graph convolution
layer~\citep{defferrard2016chebnet,shuman2013gsp}: the filter $g$ acts in the
graph-frequency domain as
$g(\lambda_j) = \alpha_0\,g_{\mathrm{lp}}(\lambda_j) +
\beta_0\,g_{\mathrm{hp}}(\lambda_j)$, where
$g_{\mathrm{lp}}(\lambda) = \sum_k c_k^{(\mathrm{lp})}T_k(\tilde\lambda)$
is the scalar frequency response of the low-pass branch and
$\tilde\lambda=2\lambda/\lambda_{\max}-1$.

\textbf{Special cases.}
Setting $\beta_0=0$ and $K_{\mathrm{lp}}=1$ with the
renormalization trick of \citet{kipf2017gcn} yields GCN as a
single-layer DSS-GNN with $P=0$.  Allowing arbitrary Chebyshev degree with both
filter branches active recovers ChebNet~\citep{defferrard2016chebnet} in its
scalar-coefficient form (filter weights $\Theta_k=c_kW_\ell$). With
generalized polynomial coefficient learning,
GPR-GNN~\citep{chien2021gprgnn} and BernNet~\citep{he2021bernnet} are similarly
obtained as $P=0$ specializations.
\end{proof}

\subsection{Proof of Theorem~\ref{thm:joint_spectral_operator} (Doubly-spectral operator)}
\label{sec:proof_joint_spectral}

\begin{lemma}[Exact-projection conditions for the joint transfer]
\label{lem:joint_transfer_conditions}
In Theorem~\ref{thm:joint_spectral_operator}, assume the pointwise
nonlinearity is omitted, the gates have chaos degree at most~$P_g$, the state
has chaos degree at most~$P$, and the Gauss--Hermite rule integrates all
triple products $\Psi_r\Psi_n\Psi_m$ with $r\le P_g$ and $n,m\le P$. A
sufficient condition is
$S\ge \lceil(P_g+2P+1)/2\rceil$. Under this condition, the stochastic coupling
matrices in the joint-transfer formula are
\[
M_{mn}^{(\ell,\mathrm{lp})}
:= \sum_{r=0}^{P_g} a_{n,r}\,c_{rnm},
\qquad
M_{mn}^{(\ell,\mathrm{hp})}
:= \sum_{r=0}^{P_g} b_{n,r}\,c_{rnm},
\qquad
c_{rnm}:=\mathbb{E}[\Psi_r\Psi_n\Psi_m].
\]
\end{lemma}

\begin{proof}
The product $\Psi_r\Psi_n\Psi_m$ has degree at most $P_g+2P$ under the stated
index ranges. An $S$-point Gauss--Hermite rule is exact for polynomials of
degree at most $2S-1$, so $2S-1\ge P_g+2P$ suffices. The coupling matrix
definitions are obtained by collecting the exactly projected triple-product
coefficients in the linearized DSS update.
\end{proof}

\thmjointspectral*

\begin{proof}
Let $\{u_j\}_{j=1}^N$ be the orthonormal eigenvectors of the graph
Laplacian~$L_G$ with eigenvalues $\{\lambda_j\}_{j=1}^N$, so that
$L_G = \sum_{j=1}^N \lambda_j\,u_j u_j^\top$~\citep{shuman2013gsp}. Because
the Chebyshev filters~\eqref{eq:cheby_filters} are polynomials in~$L_G$, each
eigenvector is simultaneously an eigenvector of both filters:
\begin{equation}
G_{\mathrm{lp}}(L_G)\,u_j = g_{\mathrm{lp}}(\lambda_j)\,u_j,
\qquad
G_{\mathrm{hp}}(L_G)\,u_j = g_{\mathrm{hp}}(\lambda_j)\,u_j,
\label{eq:filter_eigenvectors}
\end{equation}
where $g_{\mathrm{lp}}(\lambda) = \sum_{k} c_k^{(\mathrm{lp})}
T_k(\tilde\lambda)$ and
$g_{\mathrm{hp}}(\lambda) = \sum_{k} c_k^{(\mathrm{hp})}T_k(\tilde\lambda)$
are the scalar frequency responses.

\textbf{Step 1 (Graph Fourier transform of chaos coefficients).}
Define the joint spectral coefficient at graph frequency~$\lambda_j$ and chaos
order~$n$ as
$\hat H_n^{(\ell)}(\lambda_j) := u_j^\top H_n^{(\ell)} \in
\mathbb{R}^{1\times d_\ell}$.
For proof convenience, write
$U_n^{\mathrm{lp}}:=G_{\mathrm{lp}}(L_G)H_n^{(\ell)}$ and
$U_n^{\mathrm{hp}}:=G_{\mathrm{hp}}(L_G)H_n^{(\ell)}$. These spectrally
filtered coefficients from~\eqref{eq:sample_preact} satisfy
\begin{equation}
u_j^\top U_n^{\mathrm{lp}}
= u_j^\top G_{\mathrm{lp}}(L_G)\,H_n^{(\ell)}
= g_{\mathrm{lp}}(\lambda_j)\;\hat H_n^{(\ell)}(\lambda_j),
\label{eq:fourier_filtered}
\end{equation}
and analogously
$u_j^\top U_n^{\mathrm{hp}} =
g_{\mathrm{hp}}(\lambda_j)\;\hat H_n^{(\ell)}(\lambda_j)$.

\textbf{Step 2 (Linearized update in the chaos basis).}
Omitting the pointwise nonlinearity~$\sigma$ (linearization), the DSS
update~\eqref{eq:sample_preact}--\eqref{eq:dss_projection_update} gives
\begin{align}
H_m^{(\ell+1)}
&= \sum_{s=1}^{S} \mu_s\,\widetilde H^{(\ell+1)}(\omega_s)\,\Psi_m(\omega_s)
\notag\\
&= \sum_{s=1}^{S} \mu_s \sum_{n=0}^{P} \Psi_n(\omega_s)
   \bigl(\alpha_n(\omega_s)\,U_n^{\mathrm{lp}}
   + \beta_n(\omega_s)\,U_n^{\mathrm{hp}}\bigr)
   W_\ell\;\Psi_m(\omega_s).
\label{eq:app_linearized}
\end{align}
Expanding the gates $\alpha_n(\omega_s) = \sum_{r=0}^{P_g}
a_{n,r}\Psi_r(\omega_s)$ from~\eqref{eq:random_gates} and exchanging the order
of summation:
\begin{align}
H_m^{(\ell+1)}
&= \sum_{n=0}^{P}\Bigl[
   \Bigl(\sum_{r=0}^{P_g} a_{n,r}
     \underbrace{\sum_{s=1}^{S}\mu_s\,
     \Psi_r(\omega_s)\Psi_n(\omega_s)\Psi_m(\omega_s)}_{=\;c_{rnm}}
   \Bigr) U_n^{\mathrm{lp}}
\notag\\
&\quad\;+
   \Bigl(\sum_{r=0}^{P_g} b_{n,r}\,c_{rnm}\Bigr) U_n^{\mathrm{hp}}
\Bigr]W_\ell,
\label{eq:app_coupling_step}
\end{align}
where the identification
$\sum_s \mu_s\,\Psi_r(\omega_s)\Psi_n(\omega_s)\Psi_m(\omega_s) = c_{rnm} :=
\mathbb{E}[\Psi_r\Psi_n\Psi_m]$ follows from
Theorem~\ref{thm:nonintrusive_equals_intrusive}: the Gauss--Hermite rule
integrates the triple product exactly when $S$ is sufficiently large.

Define the \emph{chaos coupling matrices}
$M^{(\ell,\mathrm{lp})},\,M^{(\ell,\mathrm{hp})} \in
\mathbb{R}^{(P+1)\times(P+1)}$ with entries
\begin{equation}
M_{mn}^{(\ell,\mathrm{lp})}
:= \sum_{r=0}^{P_g} a_{n,r}\,c_{rnm},
\qquad
M_{mn}^{(\ell,\mathrm{hp})}
:= \sum_{r=0}^{P_g} b_{n,r}\,c_{rnm}.
\label{eq:coupling_matrices}
\end{equation}
Then~\eqref{eq:app_coupling_step} simplifies to
\begin{equation}
H_m^{(\ell+1)}
= \sum_{n=0}^{P}
  \bigl(M_{mn}^{(\ell,\mathrm{lp})}\,U_n^{\mathrm{lp}}
  + M_{mn}^{(\ell,\mathrm{hp})}\,U_n^{\mathrm{hp}}\bigr)\,W_\ell.
\label{eq:app_chaos_coupling}
\end{equation}

\textbf{Step 3 (Joint spectral representation).}
Left-multiplying~\eqref{eq:app_chaos_coupling} by $u_j^\top$ and
applying~\eqref{eq:fourier_filtered}:
\[
\hat H_m^{(\ell+1)}(\lambda_j)
= \sum_{n=0}^{P}
  \bigl(M_{mn}^{(\ell,\mathrm{lp})}\,g_{\mathrm{lp}}(\lambda_j)
  + M_{mn}^{(\ell,\mathrm{hp})}\,g_{\mathrm{hp}}(\lambda_j)\bigr)
  \;\hat H_n^{(\ell)}(\lambda_j)\;W_\ell.
\]
Stack the chaos orders into a matrix
$\widehat{\mathbf H}^{(\ell)}(\lambda_j) :=
\bigl[\hat H_0^{(\ell)}(\lambda_j);\;\ldots\;;\;
\hat H_P^{(\ell)}(\lambda_j)\bigr]
\in\mathbb{R}^{(P+1)\times d_\ell}$. The equation above is the
$(m,\cdot)$~row of the matrix relation
\begin{equation}
\widehat{\mathbf H}^{(\ell+1)}(\lambda_j)
= \mathcal{H}^{(\ell)}(\lambda_j)\;
  \widehat{\mathbf H}^{(\ell)}(\lambda_j)\;W_\ell,
\label{eq:app_joint_transfer}
\end{equation}
where the $(P{+}1)\times(P{+}1)$ transfer matrix decomposes as
\begin{equation}
\mathcal{H}^{(\ell)}(\lambda)
= g_{\mathrm{lp}}(\lambda)\;M^{(\ell,\mathrm{lp})}
+ g_{\mathrm{hp}}(\lambda)\;M^{(\ell,\mathrm{hp})}.
\label{eq:app_transfer_decomp}
\end{equation}

\textbf{Interpretation.}
Equation~\eqref{eq:app_joint_transfer} shows that, in the linearized regime,
the DSS layer implements a \emph{learned linear operator on the joint spectral
domain} of the doubly-spectral
expansion~\eqref{eq:doubly_spectral_expansion}. The decomposition
\eqref{eq:app_transfer_decomp} separates two independent mechanisms:
\begin{enumerate}
\item \emph{Graph-frequency response}: the scalar functions
$g_{\mathrm{lp}}(\lambda)$ and $g_{\mathrm{hp}}(\lambda)$ shape how the layer
responds to each graph eigenvalue $\lambda_j$, exactly as in standard spectral
GNNs~\citep{defferrard2016chebnet,shuman2013gsp}.
\item \emph{Chaos-order coupling}: the matrices $M^{(\ell,\mathrm{lp})}$
and $M^{(\ell,\mathrm{hp})}$ control the coupling between chaos
orders~$0,\ldots,P$, enabling the layer to redistribute energy across the
chaos spectrum.
\end{enumerate}
Together, these two axes parameterize a bilinear family of operators on the
tensor-product Hilbert
space~$\ell^2(V)\otimes L^2(\mathbb{R},\gamma)$,
thereby justifying the \emph{doubly-spectral} designation.
\end{proof}

\subsection{Proof of Theorem~\ref{thm:universal_approx} (Representational capacity and truncation)}
\label{sec:proof_universal}

\thmuniversalapprox*

\begin{proof}[Proof of the approximation claim]
The proof proceeds in three stages: (i)~truncation in the chaos-order
axis, (ii)~a rank argument for the quadrature projection, and
(iii)~a one-layer constructive realization using the final linear readout.

\textbf{Step 1 (Chaos-order truncation).}
By the Wiener--It\^{o} chaos decomposition~\citep{janson1997gaussian}, any
$Y\in L^2(\mathbb{R},\gamma;\mathbb{R}^{N\times d})$ admits the expansion
$Y(\omega)=\sum_{n=0}^{\infty}Y_n\Psi_n(\omega)$ with deterministic
coefficients $Y_n=\mathbb{E}[Y\Psi_n]\in\mathbb{R}^{N\times d}$, and by
Parseval's identity
$\|Y\|_{L^2}^2 = \sum_{n=0}^{\infty}\|Y_n\|_F^2 < \infty$.
Fix a truncation order $P$ for the construction. For the density claim, fix
$\varepsilon>0$ and take this $P$ large enough that
$\sum_{n>P}\|Y_n\|_F^2 < \varepsilon^2$; for the explicit rate, keep $P$
arbitrary and apply the tail estimate proved below.

\textbf{Step 2 (Quadrature projection rank).}
Take $S=P+1$ and let $\{(\omega_s,\mu_s)\}_{s=1}^S$ be the associated
$S$-point Gauss--Hermite rule. Define the projection matrix
$\Phi\in\mathbb{R}^{(P+1)\times S}$ with entries
$\Phi_{ns}=\mu_s\,\Psi_n(\omega_s)$.
The Hermite polynomials $\Psi_0,\ldots,\Psi_P$ are linearly independent on
the $S=P+1$ distinct quadrature nodes, and $\mu_s>0$ for all~$s$, so $\Phi$
is invertible. Consequently, for any target coefficients
$\{Y_n\}_{n=0}^P$, there exist matrices $\{V_s\}_{s=1}^S$ in
$\mathbb{R}^{N\times d}$ satisfying
\begin{equation}
\sum_{s=1}^{S}\mu_s\,V_s\,\Psi_n(\omega_s) = Y_n,
\qquad n=0,\ldots,P.
\label{eq:quadrature_inversion}
\end{equation}

\textbf{Step 3 (Constructive one-layer realization with ReLU).}
We construct a one-layer DSS-GNN ($L=1$) whose post-readout output matches the
truncated coefficients exactly. Let $q:=2d$ and define the nonnegative hidden
targets
\[
R_s := [V_s^+,\,V_s^-]\in\mathbb{R}_+^{N\times q},
\qquad
V_s = V_s^+ - V_s^-,
\]
where $V_s^+,V_s^-$ are the entrywise positive and negative parts of $V_s$.
Choose the final readout
\[
W_{\mathrm{out}}
=
\begin{bmatrix}
I_d \\[2pt] -I_d
\end{bmatrix}
\in\mathbb{R}^{q\times d},
\]
so that $R_sW_{\mathrm{out}}=V_s$ for every $s$.

Next define the generalized Vandermonde matrix
\[
\Theta\in\mathbb{R}^{S\times S},
\qquad
\Theta_{st} := \Psi_{t-1}(\omega_s),
\qquad s,t=1,\ldots,S.
\]
Because $\Psi_0,\ldots,\Psi_{S-1}$ are linearly independent and the quadrature
nodes $\omega_1,\ldots,\omega_S$ are distinct, $\Theta$ is invertible. Hence
for the prescribed matrices $R_s$ there exist unique matrices
$A_1,\ldots,A_S\in\mathbb{R}^{N\times q}$ such that
\begin{equation}
\sum_{t=1}^{S}\Psi_{t-1}(\omega_s)\,A_t = R_s,
\qquad s=1,\ldots,S.
\label{eq:block_linear_system}
\end{equation}

Next, we realize these $A_t$ with the DSS layer. Set the lifted input width to
$d_0 = S q$ and partition the lifted features into $S$ blocks
$b_1,\ldots,b_S$, each of width $q$. Choose the graph filters and gates as
\[
G_{\mathrm{lp}}(L_G)=I
\quad (K_{\mathrm{lp}}=0,\; c_0^{(\mathrm{lp})}=1),\qquad
G_{\mathrm{hp}}(L_G)=0,
\]
\[
\alpha_n(\omega)\equiv 1 \;\; (n=0,\ldots,P),\qquad
\beta_n(\omega)\equiv 0,
\]
so $P_g=0$ suffices. Since $\operatorname{rank}(X)=N$, for each
$t=1,\ldots,S$ there exists a matrix $B_t\in\mathbb{R}^{d_{\rm in}\times q}$ with
$XB_t = A_t$. Set $W_{\mathrm{in}}^{(t-1)}$ to place $B_t$ in block $b_t$ and
zeros elsewhere. Then $H_{t-1}^{(0)} = XW_{\mathrm{in}}^{(t-1)}$ is zero
outside block $b_t$ and equals $A_t$ on block $b_t$.

Finally, choose $W_0\in\mathbb{R}^{d_0\times q}$ so that each row-block
corresponding to $b_t$ equals the identity $I_q$. At quadrature node $\omega_s$,
the quadrature-node pre-activation becomes
\[
\widetilde H^{(1)}(\omega_s)
= \sum_{t=1}^{S}\Psi_{t-1}(\omega_s)\,A_t
= R_s
\]
by~\eqref{eq:block_linear_system}. Since each $R_s$ is entrywise nonnegative
and $\sigma=\mathrm{ReLU}$, we obtain
\[
\sigma\!\bigl(\widetilde H^{(1)}(\omega_s)\bigr)=R_s,
\qquad
\sigma\!\bigl(\widetilde H^{(1)}(\omega_s)\bigr)W_{\mathrm{out}} = V_s.
\]
Projecting back to chaos coefficients gives, for every $m=0,\ldots,P$,
\[
H_m^{(1)}W_{\mathrm{out}}
= \sum_{s=1}^{S}\mu_s\,
\sigma\!\bigl(\widetilde H^{(1)}(\omega_s)\bigr)W_{\mathrm{out}}\Psi_m(\omega_s)
= \sum_{s=1}^{S}\mu_s\,V_s\,\Psi_m(\omega_s)
= Y_m,
\]
where the last equality is exactly~\eqref{eq:quadrature_inversion}.

\textbf{Step 4 (Combining with the tail).}
The retained coefficients are matched exactly, so the order-$P$ output
$\widetilde Y_P$ lies in
$\bigoplus_{n=0}^{P}\mathcal{C}_n\otimes\mathbb{R}^{N\times d}$ and has no
chaos coefficients for $n>P$. By orthogonality of the chaos subspaces and
Parseval's identity, it satisfies
\[
\|Y-\widetilde Y_P\|_{L^2}^2
= \sum_{n>P}\|Y_n\|_F^2.
\]
Thus choosing $P$ so that the right-hand side is below $\varepsilon^2$ proves
density in $L^2$.
\end{proof}

\begin{remark}[Practical vs.\ theoretical capacity]
\label{rem:practical_capacity}
The constructive part of Theorem~\ref{thm:universal_approx} is an existence
result for a restricted subfamily of DSS-GNNs: the proof requires
$\operatorname{rank}(X)=N$ (i.e., $d_{\rm in}\ge N$), a one-layer construction with
lifted width $d_0 = 2dS = 2d(P+1)$ and output width $d_1=2d$,
quadrature size $S=P+1$, the ReLU
activation, constant gates ($P_g=0$), the identity low-pass filter, and a
final linear readout that combines positive and negative hidden channels. The
construction uses the identity low-pass filter and constant gates, so it
bounds the capacity of the chaos axis alone;
Theorem~\ref{thm:joint_spectral_operator} describes the graph-frequency axis.
The practical model differs from these conditions ($d_{\rm in}$ is typically
smaller than $N$, $P$ is small, and the filters are learned), so the theorem
establishes representational capacity, while calibration, OOD detection, and
robustness are established empirically.
The explicit exponential bound adds that, for targets satisfying the
entire-extension condition, a low truncation order should capture most of the
gain, which is empirically consistent with the plateau at $P=1\text{--}2$ observed in
Figure~\ref{fig:p_ablation}(b).
\end{remark}

\begin{remark}[Rank-deficient features]
\label{rem:rank_deficient}
When $\operatorname{rank}(X)=r<N$, the same construction still matches every
truncated target whose chaos-coefficient columns lie in
$\operatorname{range}(X)$, under one positivity condition required by the
ReLU step: the nonnegative quadrature-node evaluations $R_s=[V_s^+,V_s^-]$
of Step~3 must themselves have columns in $\operatorname{range}(X)$, since
$A_t=XB_t$ forces every $A_t$, and hence every $R_s$, into that range. The
condition is satisfiable, for example, by appending a constant feature
column to $X$: with $V_s^+:=V_s+c\,\mathbf{1}\mathbf{1}^\top$ and
$V_s^-:=c\,\mathbf{1}\mathbf{1}^\top$ for $c\ge\max_s\|V_s\|_{\max}$, both
parts are nonnegative and have columns in
$\operatorname{range}([X,\mathbf 1])\supseteq\operatorname{range}(X)$, so
every truncated target with columns in $\operatorname{range}([X,\mathbf 1])$
is matched. We state exactly this matching property; no claim is made for
targets whose coefficients are outside that range.
\end{remark}

\subsection{Proof of the Truncation Bound in Theorem~\ref{thm:universal_approx}}
\label{sec:proof_decay}

The analytic Gaussian-growth condition referenced in
Theorem~\ref{thm:universal_approx} is the following: each centered coordinate
$\bar Y_{ik}(\omega):=Y_{ik}(\omega)-\mathbb{E}[Y_{ik}]$ extends to an entire
function satisfying
$|\bar Y_{ik}(z)|\le A_{ik}\exp(a|z|^2)$ for some
$a\in(0,\tfrac14)$ uniform in $(i,k)$, with
$\sum_{i,k}A_{ik}^2<\infty$. Under this condition, set
$\alpha:=2a/(1-2a)<1$; the proof yields the theorem for every
$\beta\in(\sqrt{\alpha},1)$, with $C_{\rm tr}$ depending on
$a$, $\beta$, and $\sum_{i,k}A_{ik}^2$.

\begin{proof}[Proof of the truncation bound]
We prove the bound coordinate-wise and sum.  Fix a centered coordinate
function $g:=\bar Y_{ik}$ with chaos coefficients
$c_n:=\mathbb{E}[g(\omega)\Psi_n(\omega)]$, so that
$g(\omega)=\sum_{n=1}^{\infty}c_n\Psi_n(\omega)$.

\textbf{Step 1 (Hermite generating function).}
The generating function of the normalized Hermite
polynomials~\citep{szego1975orthogonal} gives
$\sum_{n=0}^{\infty}\frac{t^n}{\sqrt{n!}}\Psi_n(\omega)
=\exp(t\omega-t^2/2)$.
Multiplying by $g(\omega)$ and taking the Gaussian expectation yields
\begin{equation}
\phi(s)
\;:=\;
\sum_{n=1}^{\infty}\frac{c_n}{\sqrt{n!}}\,s^n
\;=\;
\mathbb{E}\!\big[g(\omega+s)\big]
\;=\;
\int_{\mathbb{R}} g(x+s)\,d\gamma(x),
\label{eq:hermite_generating}
\end{equation}
which is verified by expanding $g$ in the Hermite basis and using
$\mathbb{E}[\Psi_n(\omega)\exp(s\omega-s^2/2)]=s^n/\sqrt{n!}$.
The series defining $\phi$ converges for every complex $s$ because
$|c_n|\le\|g\|_{L^2(\gamma)}$, so $\phi$ is entire; the right-hand side is
entire by the growth bound below and dominated convergence, and the two
sides agree for real $s$, hence for all $s\in\mathbb{C}$.

\textbf{Step 2 (Growth bound on $\phi$).}
The hypothesis $|g(z)|\le A\exp(a|z|^2)$ with $a<\tfrac14$ gives
\begin{align*}
|\phi(s)|
&\le \int_{\mathbb{R}} A\,\exp\!\big(a|x+s|^2\big)\,d\gamma(x)
= A\,e^{a|s|^2}\int_{\mathbb{R}}
  \exp\!\big(ax^2+2a\,x\operatorname{Re}(s)\big)\,d\gamma(x).
\end{align*}
Completing the square in the exponent
($ax^2+2a\operatorname{Re}(s)\,x-x^2/2
=-({\tfrac12}-a)x^2+2a\operatorname{Re}(s)\,x$, which is a Gaussian
integral convergent for $a<\tfrac12$) yields
\[
|\phi(s)|
\;\le\;
\frac{A}{\sqrt{1-2a}}\;
\exp\!\Big(\frac{a}{1-2a}\,|s|^2\Big).
\]

\textbf{Step 3 (Cauchy estimate and Stirling bound).}
Since $\phi$ is entire, Cauchy's inequality on the circle $|s|=R$ gives
\[
\frac{|c_n|}{\sqrt{n!}}
\;\le\;
R^{-n}\sup_{|s|=R}|\phi(s)|
\;\le\;
\frac{A}{\sqrt{1-2a}}\;R^{-n}\,e^{a'R^2},
\qquad a':=\frac{a}{1-2a}.
\]
Setting $R=\sqrt{n/(2a')}$ (the minimizer of $-n\ln R+a'R^2$):
\[
|c_n|
\;\le\;
\frac{A}{\sqrt{1-2a}}\;\sqrt{n!}\;\Big(\frac{2a'}{n}\Big)^{n/2}e^{n/2}.
\]
By Stirling's inequality $\sqrt{n!}\le e^{1/2}(2\pi n)^{1/4}(n/e)^{n/2}$,
the product $(n/e)^{n/2}(2a'/n)^{n/2}e^{n/2}$ simplifies to $(2a')^{n/2}$,
giving
\begin{equation}
c_n^2
\;\le\;
\frac{A^2\,e\sqrt{2\pi n}}{1-2a}\;(2a')^{n}
\;=\;
\frac{A^2\,e\sqrt{2\pi n}}{1-2a}\;\alpha^{n},
\qquad \alpha:=2a'=\frac{2a}{1-2a}.
\label{eq:coeff_bound}
\end{equation}
Since $a<\tfrac14$, we have $\alpha<1$.

\textbf{Step 4 (Tail bound).}
For any $\eta\in(\alpha,1)$ the $\sqrt{n}$ factor in~\eqref{eq:coeff_bound} is
absorbed into the geometric ratio:
$\sqrt{n}\,\alpha^n\le C_\eta\,\eta^n$, where
$C_\eta:=\sup_{n\ge 1}\sqrt{n}\,(\alpha/\eta)^n<\infty$ since
$\alpha/\eta<1$. Substituting this into~\eqref{eq:coeff_bound} and summing the
resulting geometric series over $n>P$:
\[
\sum_{n=P+1}^{\infty}c_n^2
\;\le\;
\frac{A^2\,e\sqrt{2\pi}\,C_\eta}{1-2a}
\;\frac{\eta^{P+1}}{1-\eta}.
\]
Summing over all coordinates $(i,k)$ and applying Parseval on the graph
eigenbasis ($\sum_{j}\|\hat Y_{n,j}\|^2=\|Y_n\|_{\ell^2(V)}^2
=\sum_{i,k}c_{n,ik}^2$)~\citep{shuman2013gsp}:
\[
\sum_{n=P+1}^{\infty}\sum_{j=1}^{N}\|\hat Y_{n,j}\|^2
\;\le\;
C\,\eta^{P}
\]
for a constant $C$ depending on $a$, $\eta$, and $\sum_{i,k}A_{ik}^2$.
Now fix any $\beta\in(\sqrt{\alpha},1)$ and set $\eta=\beta^2$. Taking square
roots gives
\[
\Big(\sum_{n>P}\sum_{j=1}^{N}\|\hat Y_{n,j}\|^2\Big)^{1/2}
\le C_{\rm tr}\beta^{P},
\qquad C_{\rm tr}:=\sqrt{C}.
\]
Since the construction above matches all retained coefficients exactly,
$\|Y-\widetilde Y_P\|_{L^2}$ equals this omitted tail, which gives the bound
stated in Theorem~\ref{thm:universal_approx}.
Since $\beta>\sqrt{\alpha}$ can be chosen arbitrarily close to $\sqrt{\alpha}$, the
effective decay rate is $e^{-cP}$ with
$c=\tfrac12\ln(1/\alpha)=\tfrac12\ln\!\big(\tfrac{1-2a}{2a}\big)>0$.
\end{proof}

\subsection{Proof of Proposition~\ref{prop:curvature} (Stochastic and mean-logit readouts)}
\label{sec:proof_readouts}

\begin{lemma}[Readout regularity and quadrature conditions]
\label{lem:readout_conditions}
Let $z_i^{(s)}=\sum_{n=0}^{P}Z_{i,n}\Psi_n(\omega_s)$, where $Z_{i,0}$ is the
mean logit. Assume $S\ge P+1$, so the quadrature rule integrates
$\Psi_n\Psi_m$ exactly for $n,m\le P$, and assume
$\ell(\cdot,y_i)$ is $C^3$ on a convex neighborhood of $Z_{i,0}$ containing all
quadrature logits. Then, with
$\delta_i^{(s)}:=z_i^{(s)}-Z_{i,0}$,
\[
\sum_s\mu_s\,\delta_i^{(s)}=0,
\qquad
\sum_s\mu_s\,\delta_i^{(s)}(\delta_i^{(s)})^\top=\Sigma_{z_i},
\]
and Taylor expansion around $Z_{i,0}$ has a remainder of order
$O(\max_s\|\delta_i^{(s)}\|^3)$. For cross-entropy,
$\nabla_z^2\ell(z,y)=\operatorname{diag}(p)-pp^\top$ with
$p=\operatorname{softmax}(z)$; for finite logits its nullspace is
$\mathrm{span}\{\mathbf 1\}$.
\end{lemma}

\begin{proof}
The first identity follows from
$\sum_s\mu_s\Psi_n(\omega_s)=\mathbb{E}[\Psi_n]=0$ for $n\ge1$. The second
follows from exact quadrature orthonormality,
$\sum_s\mu_s\Psi_n(\omega_s)\Psi_m(\omega_s)=\delta_{nm}$, giving
$\sum_s\mu_s\delta_i^{(s)}(\delta_i^{(s)})^\top
=\sum_{n=1}^{P}Z_{i,n}Z_{i,n}^\top=\Sigma_{z_i}$. The $C^3$ condition gives
the stated Taylor remainder. The cross-entropy Hessian is the covariance
matrix of a categorical distribution with probabilities~$p$, so its
finite-logit nullspace consists exactly of common shifts.
\end{proof}

\propcurvature*

\begin{proof}
The quadrature logits at node~$\omega_s$ are
$z_i^{(s)} = \sum_{n=0}^{P} Z_{i,n}\Psi_n(\omega_s)$, and the
quadrature-averaged loss is
$\bar\ell_i = \sum_{s=1}^{S}\mu_s\,\ell(z_i^{(s)},y_i)$.

\textbf{Step 1 (Taylor expansion at each quadrature node).}
Write $z_i^{(s)} = Z_{i,0} + \delta_i^{(s)}$ where
$\delta_i^{(s)} = \sum_{n=1}^{P} Z_{i,n}\Psi_n(\omega_s) \in \mathbb{R}^C$
is the stochastic perturbation. Expanding $\ell$ to second order around
$Z_{i,0}$, and using that $\ell(\cdot,y_i)$ is $C^3$ in a neighborhood of
$Z_{i,0}$:
\[
\ell(z_i^{(s)},y_i)
= \ell(Z_{i,0},y_i)
+ \nabla_z\ell(Z_{i,0},y_i)^\top \delta_i^{(s)}
+ \tfrac{1}{2}(\delta_i^{(s)})^\top
  \nabla_z^2\ell(Z_{i,0},y_i)\,\delta_i^{(s)}
+ O(\|\delta_i^{(s)}\|^3).
\]

\textbf{Step 2 (Quadrature averaging of the linear term).}
The first-order term vanishes under quadrature averaging:
\[
\sum_{s=1}^{S}\mu_s\,\nabla_z\ell^\top\delta_i^{(s)}
= \nabla_z\ell^\top
  \sum_{n=1}^{P}Z_{i,n}
  \underbrace{\sum_{s=1}^{S}\mu_s\,\Psi_n(\omega_s)}_{=\;\mathbb{E}[\Psi_n]\;=\;0}
= 0,
\]
since $\mathbb{E}[\Psi_n]=0$ for $n\ge 1$ by the orthonormality of the Hermite
basis~\eqref{eq:orthonormal_hermite}, and the Gauss--Hermite rule integrates
$\Psi_n$ exactly for $n\le 2S-1$.

\textbf{Step 3 (Quadrature averaging of the quadratic term).}
For the second-order term, abbreviate $\mathcal{H}_i := \nabla_z^2\ell(Z_{i,0},y_i)$ and
compute:
\begin{align*}
\sum_{s=1}^{S}\mu_s\,(\delta_i^{(s)})^\top \mathcal{H}_i\,\delta_i^{(s)}
&= \sum_{s=1}^{S}\mu_s
   \Bigl(\sum_{n=1}^{P}Z_{i,n}\Psi_n(\omega_s)\Bigr)^{\!\top}
   \mathcal{H}_i
   \Bigl(\sum_{m=1}^{P}Z_{i,m}\Psi_m(\omega_s)\Bigr) \\
&= \sum_{n,m=1}^{P} Z_{i,n}^\top \mathcal{H}_i\,Z_{i,m}
   \underbrace{\sum_{s=1}^{S}\mu_s\,\Psi_n(\omega_s)\Psi_m(\omega_s)}_{=\;\delta_{nm}} \\
&= \sum_{n=1}^{P} Z_{i,n}^\top \mathcal{H}_i\,Z_{i,n}
\;=\; \operatorname{tr}\!\Big(\mathcal{H}_i\sum_{n=1}^{P}Z_{i,n}Z_{i,n}^\top\Big)
\;=\; \operatorname{tr}(\mathcal{H}_i\,\Sigma_{z_i}),
\end{align*}
where the orthonormality $\sum_s\mu_s\Psi_n(\omega_s)\Psi_m(\omega_s)=\delta_{nm}$
holds by exactness of the Gauss--Hermite rule for polynomials of degree
$n+m\le 2P\le 2S-1$, and $\Sigma_{z_i}=\sum_{n=1}^{P}Z_{i,n}Z_{i,n}^\top$ is the
logit covariance from~\eqref{eq:logit_covariance}.

\textbf{Step 4 (Combining).}
Substituting Steps~2 and~3 into the averaged Taylor expansion:
\begin{equation}
\bar\ell_i
= \sum_{s}\mu_s\,\ell(z_i^{(s)},y_i)
= \ell(Z_{i,0},y_i) + 0
  + \tfrac{1}{2}\operatorname{tr}\!\big(\nabla_z^2\ell(Z_{i,0},y_i)\,
    \Sigma_{z_i}\big) + O\!\Big(\max_s\|\delta_i^{(s)}\|^3\Big).
\label{eq:curvature_prop}
\end{equation}

\textbf{Step 5 (Positive semi-definiteness and nullspace characterization for cross-entropy).}
For the cross-entropy loss $\ell(z,y) = -\log p_{y}$ where
$p = \operatorname{softmax}(z)$, the Hessian is
$\nabla_z^2\ell = \operatorname{diag}(p) - pp^\top$. This is the covariance
matrix of a categorical distribution with probabilities~$p$, hence positive
semi-definite: for any $v\in\mathbb{R}^C$,
$v^\top(\operatorname{diag}(p)-pp^\top)v = \mathbb{E}_{\hat y\sim p}[v_{\hat y}^2]
- (\mathbb{E}_{\hat y\sim p}[v_{\hat y}])^2 = \operatorname{Var}_{\hat y\sim p}(v_{\hat y})
\ge 0$. Since $\mathcal{H}_i\succeq 0$ and $\Sigma_{z_i}\succeq 0$,
\[
\operatorname{tr}(\mathcal{H}_i\Sigma_{z_i})
= \operatorname{tr}\!\big(\Sigma_{z_i}^{1/2}\mathcal{H}_i\Sigma_{z_i}^{1/2}\big)\ge 0.
\]
Moreover, $\operatorname{tr}(\mathcal{H}_i\Sigma_{z_i})=0$ if and only if
$\Sigma_{z_i}^{1/2}\mathcal{H}_i\Sigma_{z_i}^{1/2}=0$, equivalently
$\operatorname{range}(\Sigma_{z_i})\subseteq \operatorname{null}(\mathcal{H}_i)$.
For finite logits, softmax probabilities satisfy $p_c>0$ for every class, and
the Hessian nullspace is
$\operatorname{null}(\mathcal{H}_i)=\mathrm{span}\{\mathbf{1}\}$; thus purely common logit
offsets do not contribute to the correction. The special case
$\Sigma_{z_i}=0$ recovers the non-stochastic model, while in the fully
confident limit $p\to e_y$ we have $\mathcal{H}_i\to 0$.

\textbf{Step 6 (Energy shift identity).}
For any scalar $a$,
\[
\operatorname{softmax}(z+a\mathbf 1)_c
=
\frac{e^{z_c+a}}{\sum_{r=1}^C e^{z_r+a}}
=
\frac{e^{z_c}}{\sum_{r=1}^C e^{z_r}}
=
\operatorname{softmax}(z)_c.
\]
Thus any uncertainty summary that is a function only of the softmax
probabilities, such as predictive entropy, is invariant under common-logit
shifts. By contrast,
\begin{equation}
s_{\mathrm{energy}}(z+a\mathbf 1)
=
-\log\sum_{c=1}^C e^{z_c+a}
=
-\log\!\Big(e^a\sum_{c=1}^C e^{z_c}\Big)
=
s_{\mathrm{energy}}(z)-a.
\label{eq:energy_common_shift}
\end{equation}

For completeness, the analogous second-order expansion for the
quadrature-averaged energy follows by the same calculation. Write
$z_i^{(s)}=Z_{i,0}+\delta_i^{(s)}$ and
$A(z)=\log\sum_c e^{z_c}$. The gradient and Hessian of $A$ are
\[
\nabla A(z)=p,
\qquad
\nabla^2 A(z)=\operatorname{diag}(p)-pp^\top=:H_A(z),
\]
where $p=\operatorname{softmax}(z)$. Since
$s_{\mathrm{energy}}=-A$, Taylor expansion around $Z_{i,0}$ gives
\[
s_{\mathrm{energy}}(z_i^{(s)})
=
s_{\mathrm{energy}}(Z_{i,0})
-p_i^\top\delta_i^{(s)}
-\tfrac12(\delta_i^{(s)})^\top H_A(Z_{i,0})\delta_i^{(s)}
+O(\|\delta_i^{(s)}\|^3).
\]
The quadrature average of the linear term vanishes as in Step~2, and the
quadratic term averages as in Step~3:
\[
\sum_s\mu_s(\delta_i^{(s)})^\top H_A(Z_{i,0})\delta_i^{(s)}
=
\operatorname{tr}\!\big(H_A(Z_{i,0})\Sigma_{z_i}\big),
\]
using exact quadrature orthonormality for pairwise products $\Psi_n\Psi_m$
with $n,m\le P$. Therefore
\begin{equation}
\sum_s\mu_s\,s_{\mathrm{energy}}(z_i^{(s)})
=
s_{\mathrm{energy}}(Z_{i,0})
-
\tfrac12\operatorname{tr}\!\big(H_A(Z_{i,0})\,\Sigma_{z_i}\big)
+ O\!\Big(\max_s\|\delta_i^{(s)}\|^3\Big).
\label{eq:energy_second_order}
\end{equation}
The Hessian $H_A(z)$ is the same categorical covariance matrix as the
cross-entropy Hessian, so the stochastic energy correction also vanishes for
common-logit stochastic variation. The non-stochastic common shift of the mean
logit, however, is read out linearly by~\eqref{eq:energy_common_shift}.
\end{proof}

\begin{corollary}[Readout gap]
\label{cor:readout_gap}
Assume $S\ge P+1$ as in Lemma~\ref{lem:readout_conditions}. Let
$f(z)=\mathrm{softmax}(z)$ and
$B:=\sup_{z\in\mathbb{R}^C}\max_c\|\nabla_z^2 f_c(z)\|_{\mathrm{op}}$, which
is finite because every entry of $\nabla_z^2 f_c$ is a polynomial in the
probabilities $f(z)\in[0,1]^C$; since
$\nabla_z^2 f_c=f_c\big[(e_c-f)(e_c-f)^\top-(\operatorname{diag}(f)-ff^\top)\big]$
is $f_c$ times a difference of positive semidefinite matrices of operator
norm at most $2(1-f_c)^2$ and $\tfrac12$, $B\le\tfrac12$ for every $C$. Then
for every node~$i$,
\[
\|\bar p_i-\mathrm{softmax}(Z_{i,0})\|_\infty
\;\le\;\tfrac{B}{2}\operatorname{tr}(\Sigma_{z_i}),
\]
and $\arg\max_c\bar p_{i,c}=\arg\max_c\mathrm{softmax}(Z_{i,0})_c$ whenever
the gap between the two largest entries of $\mathrm{softmax}(Z_{i,0})$
exceeds $B\operatorname{tr}(\Sigma_{z_i})$.
\end{corollary}

\begin{proof}
Write $z_i^{(s)}=Z_{i,0}+\delta_i^{(s)}$. Taylor's theorem with integral
remainder gives, for each class~$c$,
\[
f_c(z_i^{(s)})
= f_c(Z_{i,0})+\nabla f_c(Z_{i,0})^\top\delta_i^{(s)}
+\int_0^1(1-t)\,(\delta_i^{(s)})^\top
\nabla^2 f_c\big(Z_{i,0}+t\delta_i^{(s)}\big)\,\delta_i^{(s)}\,dt,
\]
and the remainder is bounded in absolute value by
$\tfrac{B}{2}\|\delta_i^{(s)}\|^2$. Averaging with the quadrature weights,
the linear term vanishes because $\sum_s\mu_s\delta_i^{(s)}=0$
(Lemma~\ref{lem:readout_conditions}), and
$\sum_s\mu_s\|\delta_i^{(s)}\|^2=\operatorname{tr}(\Sigma_{z_i})$ by the same
lemma, which gives the bound. Each component of $\bar p_i$ therefore lies
within $\tfrac{B}{2}\operatorname{tr}(\Sigma_{z_i})$ of the corresponding
component of $\mathrm{softmax}(Z_{i,0})$, so the two argmax rules agree
whenever the top-two gap of $\mathrm{softmax}(Z_{i,0})$ exceeds
$B\operatorname{tr}(\Sigma_{z_i})$. Since~\eqref{eq:loss_base} penalizes
$\mathcal{E}_i=\operatorname{tr}(\Sigma_{z_i})$ on training nodes, the bound
is the structural reason the two readouts agree closely in practice
(Section~\ref{sec:exp_calibration}, Table~\ref{tab:readout_coincidence}).
\end{proof}

\section{Baseline Details}
\label{app:baselines_datasets}

\subsection{Baseline methods}

Table~\ref{tab:method_capabilities} summarizes which tasks each method
evaluates. G-$\Delta$UQ is the closest baseline in scope because it reports all
three task families; we compare against it on calibration
(Table~\ref{tab:acc_brier}) and distribution shift
(Table~\ref{tab:good_table2_style}). On OOD detection, the anchor-sampling
mechanism in G-$\Delta$UQ causes out-of-memory errors in the GNNSafe
OOD pipeline, preventing a direct comparison. Across the evaluated
benchmarks, the DSS representation gives the strongest calibration and GOOD
results among the compared methods, with the best AUROC on 7 of 11 OOD
settings. The
calibration and GOOD results use plain cross-entropy or ERM training without
task-specific robust objectives; in the OOD detection experiments
DSS-Hybrid is trained with the same energy-margin objective as GNNSafe++
(Appendix~\ref{app:implementation}).

\begin{table}[h]
\centering
\small
\caption{Tasks evaluated by each method. \cmark = evaluated in the
original paper; (\cmark) = evaluated as a secondary metric or in the appendix;
\cmark$^*$ = evaluated in the original paper but OOM on our benchmark;
\xmark = not evaluated in the original paper or not designed for that task.}
\label{tab:method_capabilities}
\begin{tabular}{lccc}
\toprule
Method & Calibration & OOD Detection & Robust Classif. \\
\midrule
MC-dropout~\citep{gal2016dropout}           & (\cmark) & \xmark & \xmark \\
Deep Ensembles~\citep{lakshminarayanan2017simple} & \cmark & (\cmark) & \xmark \\
G-$\Delta$UQ~\citep{trivedi2024gdeltauq}    & \cmark & \cmark$^*$ & \cmark \\
TFE-GNN~\citep{duan2024tfe}                 & (\cmark) & \xmark & \xmark \\
GNNSafe~\citep{wu2023gnnsafe}               & \xmark & \cmark & \xmark \\
GNNSafe++~\citep{wu2023gnnsafe}             & \xmark & \cmark & \xmark \\
Graph-EBM~\citep{fuchsgruber2024gebm}       & (\cmark) & \cmark & \xmark \\
GPN~\citep{stadler2021gpn}                   & \cmark & \cmark & \xmark \\
IRM~\citep{arjovsky2019irm}                 & \xmark & \xmark & \cmark \\
GroupDRO~\citep{sagawa2020groupdro}          & \xmark & \xmark & \cmark \\
TAR~\citep{tar}                              & \xmark & \xmark & \cmark \\
\midrule
\textbf{DSS-GNN / DSS-Hybrid (ours)}       & \cmark & \cmark & \cmark \\
\bottomrule
\end{tabular}
\end{table}

We compare against the following baselines across the three evaluation tasks.

\paragraph{Calibration baselines.}
\begin{itemize}[leftmargin=1.5em, itemsep=2pt]
\item \textbf{TFE-GNN}~\citep{duan2024tfe}: A spectral GNN that combines three
  complementary polynomial filters (low-pass, high-pass, and band-pass) via
  learned attention weights. It uses a propagate-first architecture where
  spectral filtering is applied once before an MLP backbone.
\item \textbf{G-$\Delta$UQ}~\citep{trivedi2024gdeltauq}: An intrinsic
  uncertainty method that estimates predictive uncertainty through anchor
  distances in the learned embedding space, without requiring ensembles or
  Monte Carlo sampling.
\end{itemize}

\paragraph{OOD detection baselines.}
\begin{itemize}[leftmargin=1.5em, itemsep=2pt]
\item \textbf{GNNSafe}~\citep{wu2023gnnsafe}: Uses the energy score
  $E(x) = -T\log\sum_c\exp(z_c/T)$ as an OOD detection signal, propagated
  over the graph adjacency to exploit neighborhood consistency. Does not
  require OOD exposure during training.
\item \textbf{GNNSafe++}~\citep{wu2023gnnsafe}: Extends GNNSafe with an
  energy-based regularization loss that uses OOD training data to
  separate in-distribution and OOD energy scores. Uses per-dataset margin
  hyperparameters $(m_{\mathrm{in}}, m_{\mathrm{out}}, \lambda)$.
\item \textbf{Graph-EBM}~\citep{fuchsgruber2024gebm}: Models the joint
  distribution of node features and labels using an energy-based model with
  graph-aware components, including local and group energy terms propagated
  through graph diffusion.
\item \textbf{MC-dropout}~\citep{gal2016dropout} and \textbf{Deep
  Ensembles}~\citep{lakshminarayanan2017simple}: Stochastic baselines on the
  same GCN backbone and training protocol, with $M{=}20$ dropout passes at
  test time and $M{=}5$ independently trained models respectively; both are
  scored with the same energy score and graph propagation as the
  energy-based methods (Appendix~\ref{app:implementation}).
\item \textbf{GPN}~\citep{stadler2021gpn}: Graph Posterior Network, an
  evidential model that propagates per-node Dirichlet evidence over the
  graph; we report the values obtained on this benchmark by
  \citet{wu2023gnnsafe}.
\end{itemize}

\paragraph{Distribution-shift baselines.}
Table~\ref{tab:good_table2_style} evaluates the GOOD/TAR baselines under
concept shift and adds G-$\Delta$UQ. Results for ERM through TAR are taken
from the TAR paper~\citep{tar}. Representative baselines include:
\begin{itemize}[leftmargin=1.5em, itemsep=2pt]
\item \textbf{ERM}: Standard empirical risk minimization without any
  robustness modification.
\item \textbf{IRM}~\citep{arjovsky2019irm}: Invariant Risk Minimization,
  which penalizes features that are not invariant across training environments.
\item \textbf{GroupDRO}~\citep{sagawa2020groupdro}: Distributionally robust
  optimization that upweights the worst-performing group during training.
\item \textbf{VREx}~\citep{krueger2021vrex}: Variance Risk Extrapolation,
  which penalizes variance of risks across training environments.
\item \textbf{TAR}~\citep{tar}: Topology-Aware Dynamic Reweighting, which
  uses graph-local transport to reweight training nodes based on loss and
  neighborhood structure.
\end{itemize}

\section{Dataset Statistics and Licenses}
\label{app:datasets}

Tables~\ref{tab:cal_stats}--\ref{tab:good_stats} summarize the datasets used
in our experiments.  Table~\ref{tab:dataset_sources} lists the source
repositories and licenses for all datasets.  All datasets are publicly
available and were used in accordance with their respective licenses.

\begin{table}[h]
\centering
\small
\caption{Dataset sources and licenses.}
\label{tab:dataset_sources}
\resizebox{\textwidth}{!}{%
\begin{tabular}{lll}
\toprule
Dataset(s) & Repository & License \\
\midrule
Cora, Citeseer, PubMed
  & \url{https://github.com/kimiyoung/planetoid} & MIT \\
Texas, Cornell, Wisconsin,
  & \multirow{2}{*}{\url{https://github.com/bingzhewei/geom-gcn}} & \multirow{2}{*}{not stated} \\
\quad Chameleon, Squirrel & & \\
Coauthor-CS, Amazon-Photo
  & \url{https://github.com/shchur/gnn-benchmark} & MIT \\
Roman-Empire, Amazon-Ratings,
  & \multirow{2}{*}{\url{https://github.com/yandex-research/heterophilous-graphs}} & \multirow{2}{*}{MIT} \\
\quad Minesweeper, Tolokers, Questions & & \\
Twitch
  & \url{https://github.com/benedekrozemberczki/MUSAE} & GPL-3.0 \\
ogbn-arxiv
  & \url{https://github.com/snap-stanford/ogb} & MIT \\
GOOD benchmark
  & \url{https://github.com/divelab/GOOD} & GPL-3.0 \\
\bottomrule
\end{tabular}%
}
\end{table}

\subsection{Calibration benchmarks}

\begin{table}[h]
\centering
\small
\caption{Node classification datasets for calibration experiments
(Section~\ref{sec:exp_calibration}). Results use 10 splits per dataset
(Appendix~\ref{app:implementation}); edges are counted in both directions.
Cornell uses the corrected files of the geom-gcn repository, whose initial
release duplicated the Texas features and labels. Datasets below the mid-rule are from the heterophilous suite of
\citet{platonov2023critical}.}
\label{tab:cal_stats}
\begin{tabular}{lrrrr}
\toprule
Dataset & Nodes & Edges & Classes & Features \\
\midrule
Cora        &   2{,}708 &  10{,}556 &  7 & 1{,}433 \\
Citeseer    &   3{,}327 &   9{,}228 &  6 & 3{,}703 \\
PubMed      &  19{,}717 &  88{,}651 &  3 &    500 \\
Texas       &     183 &     574 &  5 & 1{,}703 \\
Cornell     &     183 &     557 &  5 & 1{,}703 \\
Wisconsin   &     251 &     916 &  5 & 1{,}703 \\
Chameleon   &   2{,}277 &  62{,}792 &  5 & 2{,}325 \\
Squirrel    &   5{,}201 & 396{,}846 &  5 & 2{,}089 \\
CS          &  18{,}333 & 163{,}788 & 15 & 6{,}805 \\
\midrule
Roman-Empire  & 22{,}662 &  65{,}854 & 18 &   300 \\
Amazon-Ratings & 24{,}492 & 186{,}100 &  5 &   300 \\
Minesweeper   & 10{,}000 &  78{,}804 &  2 &     7 \\
Tolokers      & 11{,}758 & 1{,}038{,}000 &  2 &    10 \\
Questions     & 48{,}921 & 307{,}080 &  2 &   301 \\
\bottomrule
\end{tabular}
\end{table}

\subsection{OOD detection benchmarks}

\begin{table}[h]
\centering
\small
\caption{OOD detection datasets (Section~\ref{sec:exp_ood}). For node-OOD
(Cora, Amazon-Photo, Coauthor-CS), OOD nodes are generated via structure
(stochastic block model), feature (random interpolation of node features), or
label (held-out classes)
perturbations. For cross-graph OOD, the ID and OOD graphs are from different
domains. Statistics are for the in-distribution graph.}
\label{tab:ood_stats}
\begin{tabular}{llrrrr}
\toprule
Dataset & OOD type & Nodes & Edges & Classes & Features \\
\midrule
Cora          & structure / feature / label &   2{,}708 &  10{,}556 &  7 & 1{,}433 \\
Amazon-Photo  & structure / feature / label &   7{,}650 & 238{,}162 &  8 &   745 \\
Coauthor-CS   & structure / feature / label &  18{,}333 & 163{,}788 & 15 & 6{,}805 \\
\midrule
Twitch (DE$\to$ES/FR/RU) & cross-graph &  9{,}498 & 306{,}276 &  2 & 3{,}170 \\
Arxiv ($\le$2015$\to$2018--20) & cross-graph (temporal) &  169{,}343 & 2{,}315{,}598 & 40 &   128 \\
\bottomrule
\end{tabular}
\end{table}

\subsection{GOOD distribution-shift benchmarks}

Table~\ref{tab:good_stats} reports the node-level concept-shift settings from
GOOD~\citep{gui2022good} used in our distribution-shift experiments.

\begin{table}[h]
\centering
\small
\caption{GOOD benchmark settings (Section~\ref{sec:exp_good}). All use concept
shift. The domain column indicates the covariate that defines the distribution
shift. Node counts are for the full graph before splitting; GOOD-Cora is
built on the full Cora citation graph (CoraFull), not on the 7-class
Planetoid subset of Table~\ref{tab:cal_stats}.}
\label{tab:good_stats}
\begin{tabular}{llrrrr}
\toprule
Setting & Domain & Nodes & Edges & Classes & Features \\
\midrule
GOOD-Cora / degree  & node degree  & 19{,}793 & 126{,}842 & 70 & 8{,}710 \\
GOOD-Cora / word    & word diversity & 19{,}793 & 126{,}842 & 70 & 8{,}710 \\
GOOD-Arxiv / degree & node degree  & 169{,}343 & 2{,}315{,}598 & 40 &   128 \\
GOOD-Arxiv / time   & publication year & 169{,}343 & 2{,}315{,}598 & 40 &   128 \\
GOOD-CBAS / color   & node color   &     700 &   3{,}962 &  4 &     4 \\
GOOD-WebKB / university & university &     617 &   1{,}138 &  5 & 1{,}703 \\
GOOD-Twitch / language & user language & 34{,}120 & 892{,}346 &  2 &   128 \\
\bottomrule
\end{tabular}
\end{table}

\section{Implementation Details}
\label{app:implementation}

\subsection{General training protocol}

All DSS-GNN models are trained with the mean-logit cross-entropy loss
$\mathcal{L}_{\mathrm{base}}$~\eqref{eq:loss_base}; the quadrature-averaged
loss
\begin{equation}
\mathcal{L}_{\mathrm{quad}}
=
\frac{1}{|\mathcal{V}_{\mathrm{tr}}|}
\sum_{i\in\mathcal{V}_{\mathrm{tr}}}
\bar\ell_i
\;+\;
\lambda_{\mathrm{reg}}
\frac{1}{|\mathcal{V}_{\mathrm{tr}}|}
\sum_{i\in\mathcal{V}_{\mathrm{tr}}} \mathcal{E}_i,
\qquad
\bar\ell_i = \sum_{s=1}^{S} \mu_s\,\ell(z_i^{(s)}, y_i),
\label{eq:loss_quad}
\end{equation}
supervises every quadrature node, agrees with
$\mathcal{L}_{\mathrm{base}}$ to first order by
Proposition~\ref{prop:curvature}, and is used only in that analysis. The loss
includes the chaos energy regularizer; $\lambda_{\mathrm{reg}}$ is selected
per dataset on validation Brier from $\{0,10^{-3},10^{-2},10^{-1}\}$
(Table~\ref{tab:selection}; Table~\ref{tab:lreg_ablation} reports its effect
at $P{=}2$), with $0.01$ as the default. Training uses early
stopping on validation loss with the patience listed per task below. All
calibration results report mean $\pm$ standard deviation over 10 splits,
identical for every method. For the nine datasets above the mid-rule of
Table~\ref{tab:cal_stats} the splits are random with fixed seeds: a
class-balanced training set whose per-class quota totals 60\% of the nodes
when every class is large enough, 20\% of the nodes for validation, and the
remainder for testing. The five heterophilous-suite datasets use the first
10 public 50/25/25 splits of \citet{platonov2023critical}.

\subsection{Calibration (Table~\ref{tab:acc_brier})}

\paragraph{Shared defaults.}
Chebyshev filter degrees $K_{\mathrm{lp}}{=}K_{\mathrm{hp}}{=}4$, chaos
order centered at $P{=}2$ in the base grid, quadrature nodes $S{=}4$,
$\lambda_{\mathrm{reg}}{=}0.01$, deterministic scalar gates ($P_g{=}0$),
loss type = mean, activation = ReLU, learning rate $0.01$, weight decay
$5{\times}10^{-4}$, propagation dropout $0.5$, linear dropout $0.0$, epochs
$1000$, patience $200$.  The best
architecture variant per dataset was selected from a grid over propagation
type (Chebyshev / TFE-adjacency / propagate-first), graph filter symmetry
(symmetric / random-walk), optimizer (Adam / RMSprop), hidden dimension
($64$ / $128$), number of layers ($2$ / $3$), and dropout ($0.5$ / $0.6$ /
$0.8$), evaluated on validation Brier score. Roman-Empire uses hidden 32 and
$K_{\mathrm{lp}}{=}K_{\mathrm{hp}}{=}2$, outside this grid, Minesweeper uses learning rate $0.005$, and Chameleon
uses sum aggregation. The chaos order $P$ is a regular
hyperparameter selected per dataset on validation Brier over the sweep of
Appendix~\ref{app:ablation}; Table~\ref{tab:selection} lists the selected
order per dataset together with that single configuration's accuracy and
Brier.

\paragraph{Per-dataset configurations.}
Table~\ref{tab:per_dataset_config} lists the architecture variant and any
hyperparameter overrides from the shared defaults for each dataset.

\begin{table}[ht]
\centering
\small
\caption{Per-dataset non-$P$ DSS-GNN configuration for the calibration experiments
(Table~\ref{tab:acc_brier}).  ``Prop-first'' = propagate-once structure
(spectral filtering followed by MLP); ``TFE-adj'' = TFE-style adjacency
propagation; ``Cheb'' = Chebyshev filters on rescaled Laplacian.
Unlisted hyperparameters use the shared defaults above; $P$,
$\lambda_{\mathrm{reg}}$, and $S$ per dataset are listed in
Table~\ref{tab:selection}.}
\label{tab:per_dataset_config}
\begin{tabular}{llcccl}
\toprule
Dataset & Architecture & Hidden & $L$ & Optimizer & Overrides \\
\midrule
Cora        & Prop-first (rw) & 64  & 2 & Adam    & dropout 0.8 \\
Citeseer    & Cheb            & 64  & 2 & Adam    & none \\
PubMed      & Cheb            & 128 & 2 & Adam    & none \\
Texas       & Prop-first      & 64  & 2 & RMSprop & none \\
Cornell     & Cheb            & 128 & 2 & Adam    & dropout 0.8 \\
Wisconsin   & Cheb            & 64  & 2 & Adam    & dropout 0.6 \\
Chameleon   & Prop-first (sum) & 64  & 2 & Adam   & none \\
Squirrel    & Prop-first (rw) & 64  & 2 & RMSprop & none \\
CS          & TFE-adj         & 64  & 2 & Adam    & none \\
Roman-Emp.  & TFE-adj         & 32  & 3 & Adam    & hidden 32, $K{=}2$ \\
Amz-Rat.    & TFE-adj         & 64  & 3 & Adam    & none \\
Minesweeper & Cheb            & 64  & 2 & Adam    & lr 0.005 \\
Tolokers    & Prop-first (rw) & 64  & 2 & RMSprop & none \\
Questions   & Prop-first      & 64  & 2 & RMSprop & none \\
\bottomrule
\end{tabular}
\end{table}

\paragraph{Selected chaos order per dataset.}
Table~\ref{tab:selection} lists, for each dataset, the chaos order and
regularization strength selected on validation Brier, with the accuracy and
Brier of that configuration over the same 10 splits, taken from the sweeps of
Tables~\ref{tab:p_acc} and~\ref{tab:p_brier}, and its validation Brier
(recorded in a separate run of the same configuration on the same splits);
Table~\ref{tab:acc_brier} reports these configurations. DSS-GNN has the best
mean Brier on all 14 datasets, including the closest comparisons
(Texas $0.240$ vs $0.255$, Squirrel $0.394$ vs $0.395$, Amazon-Ratings
$0.634$ vs $0.637$, all TFE-GNN; Questions $0.053$ vs $0.054$,
G-$\Delta$UQ). At least one of $P{=}1$ or $P{=}2$ is within $0.02$ Brier of
the best swept order on every dataset (Section~\ref{sec:ablation}), so the
larger selected orders are near-ties rather than requirements. The
quadrature size is $S{=}4$ throughout: exact for the polynomial filter and
gate terms at $P\le3$ (Appendix~\ref{sec:chaos_background}), while at
$P\ge4$ the highest-order Hermite products are integrated approximately. The
recommended default $P{=}1$--$2$ is exact for these terms at $S{=}4$, and the $S$ sweep at $P{=}2$
(Table~\ref{tab:S_ablation}) moves Brier by at most $0.011$ on 12 of 14
datasets.

\begin{table}[ht]
\centering
\small
\caption{Selected configuration per dataset for the calibration experiments
(Table~\ref{tab:acc_brier}): chaos order $P$, chaos energy regularization
$\lambda_{\mathrm{reg}}$, quadrature size $S$, and that single
configuration's test accuracy (\%) and Brier score (mean $\pm$ SD over the
same 10 splits). The last column reports the validation Brier of the
selected configuration, recorded in a separate run on the same 10 splits
(for Cornell, in the same run).}
\label{tab:selection}
\begin{tabular}{l ccc cc c}
\toprule
Dataset & $P$ & $\lambda_{\mathrm{reg}}$ & $S$ & Accuracy & Brier & Val.\ Brier \\
\midrule
Cora           & 1 & 0.01 & 4 & 86.63 $\pm$ 1.26 & 0.207 $\pm$ 0.019 & 0.203 \\
Citeseer       & 2 & 0 & 4 & 80.20 $\pm$ 1.28 & 0.308 $\pm$ 0.010 & 0.315 \\
PubMed         & 5 & 0.01 & 4 & 89.72 $\pm$ 0.31 & 0.156 $\pm$ 0.005 & 0.152 \\
Texas          & 2 & 0.01 & 4 & 91.31 $\pm$ 3.44 & 0.240 $\pm$ 0.111 & 0.190 \\
Cornell        & 1 & 0.01  & 4 & 85.11 $\pm$ 5.12 & 0.231 $\pm$ 0.079 & 0.166 \\
Wisconsin      & 2 & 0 & 4 & 93.25 $\pm$ 3.39 & 0.104 $\pm$ 0.047 & 0.088 \\
Chameleon      & 1 & 0.01 & 4 & 75.36 $\pm$ 1.19 & 0.339 $\pm$ 0.019 & 0.333 \\
Squirrel       & 1 & 0.1 & 4 & 68.06 $\pm$ 2.21 & 0.394 $\pm$ 0.027 & 0.415 \\
CS             & 1 & 0.01 & 4 & 96.17 $\pm$ 0.23 & 0.062 $\pm$ 0.004 & 0.060 \\
Roman-Empire   & 1 & 0 & 4 & 78.84 $\pm$ 0.61 & 0.300 $\pm$ 0.007 & 0.304 \\
Amazon-Ratings & 8 & 0.01 & 4 & 49.72 $\pm$ 0.64 & 0.634 $\pm$ 0.005 & 0.637 \\
Minesweeper    & 6 & 0.001 & 4 & 87.66 $\pm$ 1.10 & 0.169 $\pm$ 0.014 & 0.170 \\
Tolokers       & 2 & 0.01  & 4 & 79.88 $\pm$ 0.49 & 0.268 $\pm$ 0.005 & 0.305 \\
Questions      & 4 & 0.001 & 4 & 97.17 $\pm$ 0.04 & 0.053 $\pm$ 0.001 & 0.054 \\
\bottomrule
\end{tabular}
\end{table}

\subsection{OOD detection (Section~\ref{sec:exp_ood})}

DSS-Hybrid uses a 2-layer GCN base encoder (hidden 64, dropout 0.0) with a
DSS residual branch ($K_{\mathrm{lp}}{=}3$, $K_{\mathrm{hp}}{=}2$, $P{=}1$,
$S{=}4$, deterministic scalar gates).  The base encoder is warmed up for 50
epochs before the DSS branch is activated; the residual scale
$\gamma_{\mathrm{res}}$ is initialized to $0.1$.  All methods run in our pipeline use Adam with learning rate $0.01$, weight
decay $0.01$, 200 epochs, and 3 seeds, following the GNNSafe protocol. The
GNNSafe pipeline uses BatchNorm in its GCN backbones, and so does the
DSS-Hybrid base; the standalone model of Appendix~\ref{app:crosseval} adds
per-chaos-channel BatchNorm in all OOD and GOOD cells but OOD Twitch and
GOOD-Arxiv/degree.  OOD detection scores are
computed with the energy score~\eqref{eq:energy_score} at temperature
$T{=}1$, propagated over the graph adjacency with $K{=}2$ hops and mixing
$\alpha{=}0.5$, following the GNNSafe protocol~\citep{wu2023gnnsafe}.
DSS-Hybrid is trained with the GNNSafe++ energy-margin regularizer, using
the per-dataset margin hyperparameters
$(m_{\mathrm{in}}, m_{\mathrm{out}}, \lambda)$ of \citet{wu2023gnnsafe}, as
GNNSafe++ is; Graph-EBM, MC-dropout, and Deep Ensembles are trained without
it, as GNNSafe is.
MC-dropout keeps dropout active at test time with BatchNorm frozen and
aggregates $M{=}20$ stochastic passes; Deep Ensembles aggregate $M{=}5$
independently initialized backbones; both are scored with the same energy
score and propagation as the energy-based methods, applied to the mean
logits across passes or members.

\paragraph{Standalone model in the OOD and GOOD pipelines.}
For the cross-evaluation of Appendix~\ref{app:crosseval}, the standalone
model adds one BatchNorm1d per chaos channel between hidden DSS layers where
noted; Table~\ref{tab:acc_brier} uses no BatchNorm. The per-cell
configurations are listed in Appendix~\ref{app:crosseval}.

\paragraph{Graph-EBM reproduction.}
Graph-EBM is run with a port of the official
\texttt{GraphEBMWrapper}~\citep{fuchsgruber2024gebm} at the official default
hyperparameters: the multi-scale corrected-conditional energy model with
Mahalanobis logit correction, three energy scales, and label-propagation
diffusion (diagonal untied covariance, $\gamma{=}1$, unit energy weights,
$\alpha{=}0.5$ with $k{=}10$ steps). The scorer is fitted on the same GCN
encoder (with BatchNorm, as in the GNNSafe pipeline), trained with the same
hyperparameters and seeds as the other baselines run in our pipeline;
Graph-EBM's own diffusion replaces the GNNSafe score propagation for this
row. On an exported Cora/structure run (trained GCN logits and embeddings,
both graphs, labels, and masks) the port and the official implementation
agree at Spearman rank correlation $0.994$, and exactly (rank correlation
$1.0$) once the interval and covariance estimators of the official
implementation are used.

\subsection{Distribution shift (Table~\ref{tab:good_table2_style})}

DSS-Hybrid uses a 2-layer GCN base encoder with a DSS residual branch
($K_{\mathrm{lp}}{=}3$, $K_{\mathrm{hp}}{=}2$, $P{=}1$, $S{=}4$, hidden 64,
deterministic scalar gates, 3 layers total).  Training uses Adam with learning
rate $10^{-3}$, weight decay $5{\times}10^{-4}$, dropout $0.5$, 500 epochs,
patience 200.  The base
encoder is warmed up for 25 epochs.  Standard ERM training is used (no
robust objective).  Published baseline results (ERM through TAR) are taken
from the TAR paper~\citep{tar}; G-$\Delta$UQ is run in our pipeline.

\subsection{Software, hardware, and compute budget}

Experiments use PyTorch and PyTorch Geometric (the released environment pins
PyTorch 2.5) on NVIDIA A100 (40\,GB) and H100 GPUs; the ogbn-arxiv timing
runs of Table~\ref{tab:timing} use one H100, and the Cornell experiments run
on CPU.  Mixed precision (FP16) is used for calibration experiments to
reduce memory on larger graphs.  On the A100 runs, per-epoch wall-clock times range from
${\sim}35$\,ms (Cora, $P{=}0$) to ${\sim}650$\,ms (Amazon-Ratings, $P{=}3$);
see Table~\ref{tab:timing} for a detailed breakdown.

\paragraph{Compute estimate.}
The reported experiments (calibration: 14 datasets $\times$ 10 runs;
OOD: 5 datasets $\times$ 3 seeds; GOOD: 7 settings) total approximately
8 A100 GPU-hours.  The ablation studies (chaos order $P$, regularization
$\lambda_{\mathrm{reg}}$, quadrature nodes $S$, and contribution
decomposition; Appendix~\ref{app:ablation}) add approximately
100 GPU-hours.  Hyperparameter selection (grid search over architecture
variant, hidden dimension, optimizer, and dropout) required an additional
${\sim}25$ GPU-hours.  The total compute for the project, including
preliminary experiments not reported in the paper, is approximately
200 A100 GPU-hours.

Code is available at \coderepo.

\section{Ablation Study Details}
\label{app:ablation}

\subsection{Polynomial order \texorpdfstring{$P$}{P} across all datasets}

Figure~\ref{fig:p_ablation} summarizes the chaos-order sweep referenced in
Section~\ref{sec:ablation}, and Tables~\ref{tab:p_acc} and~\ref{tab:p_brier}
report the per-dataset classification accuracy and Brier score for DSS-GNN
across all 14 datasets for $P\in\{0,1,2,3,4,5,6,8\}$. All results
are mean $\pm$ standard deviation over 10 random splits. For each dataset, the
best-performing $P$ is highlighted in bold.

\begin{figure}[ht]
\centering
\includegraphics[width=.9\textwidth]{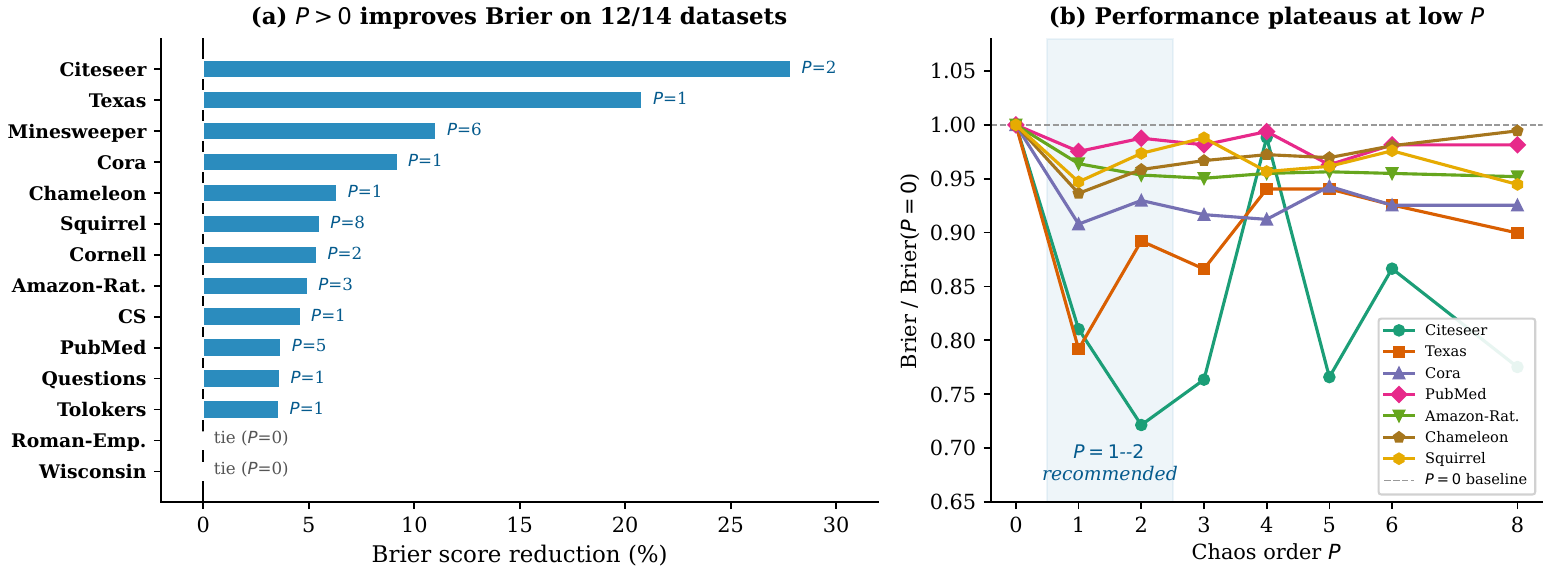}
\caption{\textbf{(a)}~Percentage Brier reduction of the best swept order
$P{>}0$ over $P{=}0$ (Table~\ref{tab:p_brier}). \textbf{(b)}~Brier relative
to $P{=}0$ versus chaos order, for seven datasets.}
\label{fig:p_ablation}
\end{figure}

\begin{table}[ht]
\centering
\caption{Accuracy (\%) of DSS-GNN for varying chaos order $P$ across all datasets.}
\label{tab:p_acc}
\resizebox{\textwidth}{!}{
\scriptsize
\begin{tabular}{lcccccccc}
\toprule
Dataset & $P=0$ & $P=1$ & $P=2$ & $P=3$ & $P=4$ & $P=5$ & $P=6$ & $P=8$ \\
\midrule
Cora & 85.12 $\pm$ 1.06 & \textbf{86.63} $\pm$ 1.26 & 86.33 $\pm$ 1.12 & 86.33 $\pm$ 0.82 & \textbf{86.63} $\pm$ 1.14 & 86.32 $\pm$ 1.13 & 86.61 $\pm$ 0.84 & 86.38 $\pm$ 0.96 \\
Citeseer & 65.39 $\pm$ 1.02 & 79.21 $\pm$ 0.79 & \textbf{80.20} $\pm$ 1.28 & 78.47 $\pm$ 0.80 & 71.50 $\pm$ 4.58 & 78.92 $\pm$ 1.13 & 76.77 $\pm$ 2.19 & 78.49 $\pm$ 1.08 \\
PubMed & 89.12 $\pm$ 0.57 & 89.33 $\pm$ 0.49 & 89.33 $\pm$ 0.66 & 89.38 $\pm$ 0.67 & 89.33 $\pm$ 0.54 & \textbf{89.72} $\pm$ 0.31 & 89.40 $\pm$ 0.57 & 89.39 $\pm$ 0.51 \\
Texas & 75.08 $\pm$ 23.70 & 88.36 $\pm$ 4.57 & \textbf{91.31} $\pm$ 3.44 & 89.18 $\pm$ 4.92 & \textbf{91.31} $\pm$ 3.30 & 87.54 $\pm$ 5.59 & 89.01 $\pm$ 2.82 & 87.37 $\pm$ 5.43 \\
Cornell & 83.62 $\pm$ 6.80 & 85.11 $\pm$ 5.12 & \textbf{85.74} $\pm$ 6.73 & 83.62 $\pm$ 5.39 & 84.26 $\pm$ 6.87 & 83.83 $\pm$ 7.38 & 83.62 $\pm$ 6.53 & 82.98 $\pm$ 7.61 \\
Wisconsin & \textbf{93.25} $\pm$ 3.46 & 92.87 $\pm$ 2.92 & \textbf{93.25} $\pm$ 3.39 & 88.50 $\pm$ 3.35 & 93.00 $\pm$ 3.46 & 90.75 $\pm$ 2.84 & 89.63 $\pm$ 4.82 & 90.50 $\pm$ 2.92 \\
Chameleon & 74.53 $\pm$ 2.32 & \textbf{75.36} $\pm$ 1.19 & 74.53 $\pm$ 1.74 & 74.51 $\pm$ 1.57 & 74.77 $\pm$ 1.86 & 74.92 $\pm$ 1.60 & 74.40 $\pm$ 1.65 & 73.63 $\pm$ 1.69 \\
Squirrel & 66.35 $\pm$ 2.56 & \textbf{68.06} $\pm$ 2.21 & 67.26 $\pm$ 2.13 & 66.99 $\pm$ 1.24 & 67.39 $\pm$ 2.44 & 67.23 $\pm$ 1.55 & 67.02 $\pm$ 1.50 & 67.71 $\pm$ 1.51 \\
CS & 95.94 $\pm$ 0.20 & \textbf{96.17} $\pm$ 0.23 & 96.01 $\pm$ 0.25 & 96.03 $\pm$ 0.19 & 96.13 $\pm$ 0.25 & 95.96 $\pm$ 0.29 & 96.02 $\pm$ 0.23 & 96.02 $\pm$ 0.17 \\
Roman-Emp. & \textbf{79.01} $\pm$ 0.27 & 78.84 $\pm$ 0.61 & 78.81 $\pm$ 0.61 & 78.98 $\pm$ 0.87 & 78.75 $\pm$ 0.32 & 78.77 $\pm$ 0.69 & \textbf{79.01} $\pm$ 0.27 & 78.97 $\pm$ 0.79 \\
Amazon-Rat. & 45.40 $\pm$ 0.76 & 48.13 $\pm$ 1.14 & 49.41 $\pm$ 0.71 & 49.65 $\pm$ 1.15 & 49.36 $\pm$ 0.70 & 49.05 $\pm$ 0.73 & 49.35 $\pm$ 0.83 & \textbf{49.72} $\pm$ 0.64 \\
Minesweeper & 86.14 $\pm$ 0.60 & 86.93 $\pm$ 0.51 & 87.37 $\pm$ 1.10 & 86.37 $\pm$ 1.06 & 86.53 $\pm$ 0.96 & 86.17 $\pm$ 0.45 & \textbf{87.66} $\pm$ 1.10 & 87.42 $\pm$ 0.93 \\
Tolokers & 79.33 $\pm$ 0.66 & 79.82 $\pm$ 0.45 & \textbf{79.88} $\pm$ 0.49 & 79.28 $\pm$ 0.79 & 79.62 $\pm$ 0.56 & 79.32 $\pm$ 0.66 & 79.21 $\pm$ 0.75 & 79.36 $\pm$ 0.79 \\
Questions & 97.15 $\pm$ 0.06 & 97.15 $\pm$ 0.06 & 97.13 $\pm$ 0.06 & 97.16 $\pm$ 0.06 & \textbf{97.17} $\pm$ 0.04 & 97.15 $\pm$ 0.06 & 97.15 $\pm$ 0.06 & 97.14 $\pm$ 0.09 \\
\bottomrule
\end{tabular}
}
\end{table}

\begin{table}[ht]
\centering
\caption{Brier Score ($\downarrow$) of DSS-GNN for varying chaos order $P$ across all datasets.}
\label{tab:p_brier}
\resizebox{\textwidth}{!}{
\scriptsize
\begin{tabular}{lcccccccc}
\toprule
Dataset & $P=0$ & $P=1$ & $P=2$ & $P=3$ & $P=4$ & $P=5$ & $P=6$ & $P=8$ \\
\midrule
Cora & 0.228 $\pm$ 0.013 & \textbf{0.207} $\pm$ 0.019 & 0.212 $\pm$ 0.016 & 0.209 $\pm$ 0.016 & 0.208 $\pm$ 0.019 & 0.215 $\pm$ 0.015 & 0.211 $\pm$ 0.011 & 0.211 $\pm$ 0.015 \\
Citeseer & 0.427 $\pm$ 0.008 & 0.346 $\pm$ 0.008 & \textbf{0.308} $\pm$ 0.010 & 0.326 $\pm$ 0.012 & 0.422 $\pm$ 0.027 & 0.327 $\pm$ 0.012 & 0.370 $\pm$ 0.023 & 0.331 $\pm$ 0.013 \\
PubMed & 0.162 $\pm$ 0.007 & 0.158 $\pm$ 0.007 & 0.160 $\pm$ 0.008 & 0.159 $\pm$ 0.008 & 0.161 $\pm$ 0.007 & \textbf{0.156} $\pm$ 0.005 & 0.159 $\pm$ 0.008 & 0.159 $\pm$ 0.007 \\
Texas & 0.269 $\pm$ 0.082 & \textbf{0.213} $\pm$ 0.111 & 0.240 $\pm$ 0.111 & 0.233 $\pm$ 0.096 & 0.253 $\pm$ 0.096 & 0.253 $\pm$ 0.086 & 0.249 $\pm$ 0.114 & 0.242 $\pm$ 0.082 \\
Cornell & 0.241 $\pm$ 0.087 & 0.231 $\pm$ 0.079 & \textbf{0.228} $\pm$ 0.088 & 0.250 $\pm$ 0.090 & 0.263 $\pm$ 0.129 & 0.247 $\pm$ 0.107 & 0.252 $\pm$ 0.079 & 0.273 $\pm$ 0.111 \\
Wisconsin & \textbf{0.103} $\pm$ 0.048 & 0.107 $\pm$ 0.041 & 0.104 $\pm$ 0.047 & 0.161 $\pm$ 0.060 & \textbf{0.103} $\pm$ 0.048 & 0.129 $\pm$ 0.037 & 0.153 $\pm$ 0.052 & 0.135 $\pm$ 0.040 \\
Chameleon & 0.362 $\pm$ 0.019 & \textbf{0.339} $\pm$ 0.019 & 0.347 $\pm$ 0.022 & 0.350 $\pm$ 0.016 & 0.352 $\pm$ 0.016 & 0.351 $\pm$ 0.018 & 0.355 $\pm$ 0.013 & 0.360 $\pm$ 0.015 \\
Squirrel & 0.416 $\pm$ 0.028 & 0.394 $\pm$ 0.027 & 0.405 $\pm$ 0.024 & 0.411 $\pm$ 0.019 & 0.398 $\pm$ 0.028 & 0.400 $\pm$ 0.016 & 0.406 $\pm$ 0.014 & \textbf{0.393} $\pm$ 0.017 \\
CS & 0.065 $\pm$ 0.003 & \textbf{0.062} $\pm$ 0.004 & 0.063 $\pm$ 0.003 & 0.064 $\pm$ 0.003 & 0.063 $\pm$ 0.003 & 0.064 $\pm$ 0.004 & 0.063 $\pm$ 0.003 & 0.063 $\pm$ 0.003 \\
Roman-Emp. & \textbf{0.298} $\pm$ 0.005 & 0.300 $\pm$ 0.007 & 0.300 $\pm$ 0.008 & \textbf{0.298} $\pm$ 0.005 & 0.299 $\pm$ 0.005 & 0.300 $\pm$ 0.009 & 0.299 $\pm$ 0.007 & \textbf{0.298} $\pm$ 0.005 \\
Amazon-Rat. & 0.666 $\pm$ 0.003 & 0.642 $\pm$ 0.010 & 0.635 $\pm$ 0.006 & \textbf{0.633} $\pm$ 0.006 & 0.636 $\pm$ 0.006 & 0.637 $\pm$ 0.005 & 0.636 $\pm$ 0.004 & 0.634 $\pm$ 0.005 \\
Minesweeper & 0.190 $\pm$ 0.007 & 0.180 $\pm$ 0.006 & 0.170 $\pm$ 0.014 & 0.185 $\pm$ 0.011 & 0.181 $\pm$ 0.011 & 0.187 $\pm$ 0.005 & \textbf{0.169} $\pm$ 0.014 & 0.171 $\pm$ 0.012 \\
Tolokers & 0.278 $\pm$ 0.010 & \textbf{0.268} $\pm$ 0.004 & \textbf{0.268} $\pm$ 0.005 & 0.276 $\pm$ 0.012 & 0.280 $\pm$ 0.009 & 0.277 $\pm$ 0.010 & 0.278 $\pm$ 0.009 & 0.279 $\pm$ 0.009 \\
Questions & 0.055 $\pm$ 0.001 & \textbf{0.053} $\pm$ 0.001 & 0.054 $\pm$ 0.002 & \textbf{0.053} $\pm$ 0.001 & \textbf{0.053} $\pm$ 0.001 & \textbf{0.053} $\pm$ 0.001 & \textbf{0.053} $\pm$ 0.001 & \textbf{0.053} $\pm$ 0.002 \\
\bottomrule
\end{tabular}
}
\end{table}

On four datasets the best swept order exceeds two (PubMed, Squirrel,
Amazon-Ratings, and Minesweeper); on every dataset $P=1$ or $P=2$ is within
0.02 Brier of the best swept value.

\subsection{Chaos order on OOD detection and under distribution shift}
\label{app:chaos_ood_good}

\paragraph{OOD detection.}
Varying only the chaos order $P\in\{0,1,2,3\}$ in the DSS-Hybrid OOD
configuration (GCN base, energy score with propagation, energy-margin
training, every other hyperparameter fixed; 3 runs) leaves energy AUROC
flat: Cora-structure $91.0/91.2/90.9/90.8$ and Amazon-Photo/feature
$99.4/99.5/99.5/99.5$ for $P=0,1,2,3$, Coauthor-CS likewise flat, and on
ogbn-arxiv the AUROC moves by at most $0.8$ across $P=0$ to $3$ (the runs of
Table~\ref{tab:timing}). The absolute levels differ from
Table~\ref{tab:gnnsafe_local_table2style} because the ablation is a separate
set of runs at one configuration per dataset. This is the behavior the
design predicts: the energy score reads the mean logit
alone~\eqref{eq:energy_score}, so $P$ can be selected for calibration without
affecting detection, and the detection margin comes from the DSS
representation under a fixed readout (Section~\ref{sec:exp_ood}).

\paragraph{Distribution shift.}
At the validation-selected standalone configurations of
Appendix~\ref{app:crosseval}, Table~\ref{tab:good_pabl} varies only $P$
(3 seeds). $P{=}1$ is the argmax on both the OOD-validation and the test
split on all three settings, so at these configurations the chaos order
contributes $+1.7$, $+1.2$, and $+0.6$ points of shifted accuracy over the
deterministic $P{=}0$ arm.

\begin{table}[ht]
\centering
\small
\caption{Shifted test accuracy (\%) of standalone DSS-GNN on GOOD concept
shift as the chaos order varies at the validation-selected configuration
(mean $\pm$ SD over 3 seeds).}
\label{tab:good_pabl}
\begin{tabular}{l ccc}
\toprule
Setting & $P{=}0$ & $P{=}1$ & $P{=}2$ \\
\midrule
GOOD-Cora / word    & 63.24 $\pm$ 0.40 & \textbf{64.89} $\pm$ 0.21 & 64.01 $\pm$ 0.25 \\
GOOD-Cora / degree  & 61.38 $\pm$ 0.81 & \textbf{62.60} $\pm$ 0.59 & 61.22 $\pm$ 0.07 \\
GOOD-Arxiv / time   & 64.81 $\pm$ 0.27 & \textbf{65.43} $\pm$ 1.07 & 64.48 $\pm$ 0.27 \\
\bottomrule
\end{tabular}
\end{table}

\subsection{Energy-score propagation}
\label{app:propagation}

Table~\ref{tab:prop_ablation} isolates the score-propagation step of the
energy readout on DSS-Hybrid: the raw energy and the propagated energy are
computed from the same trained checkpoints (one fixed configuration, 3
runs), so the difference is due to propagation alone. Propagation helps on
9 of 11 settings, by up to $+18.2$ AUROC (Twitch) and $+14.5$
(Cora-structure); the two exceptions are Cora/label ($-6.7$), where the
fixed smoothing averages away part of the leave-out signal inside
homophilous neighborhoods, and Amazon-Photo/structure ($-0.4$), at noise
level. The same step transfers to the standalone model (Cora/label $+0.9$,
Amazon-Photo/label $+13.4$ at identical checkpoints). We apply propagation
uniformly for protocol parity with GNNSafe++.

\begin{table}[ht]
\centering
\small
\caption{Effect of energy-score propagation on DSS-Hybrid: change in AUROC
(points) from the raw to the propagated energy score at identical
checkpoints (3 runs, fixed configuration).}
\label{tab:prop_ablation}
\begin{tabular}{l ccc}
\toprule
Dataset & Structure & Feature & Label \\
\midrule
Cora         & $+14.5$ & $+5.4$ & $-6.7$ \\
Amazon-Photo & $-0.4$  & $+1.7$ & $+3.9$ \\
Coauthor-CS  & $+1.9$  & $+1.1$ & $+1.3$ \\
\midrule
Twitch (cross-graph) & \multicolumn{3}{c}{$+18.2$} \\
Arxiv (cross-graph)  & \multicolumn{3}{c}{$+7.3$} \\
\bottomrule
\end{tabular}
\end{table}

\subsection{Chaos energy regularization \texorpdfstring{$\lambda_{\mathrm{reg}}$}{lambda reg}}

Table~\ref{tab:lreg_ablation} reports the Brier score as the chaos energy
regularization strength $\lambda_{\mathrm{reg}}$ varies from 0 to 0.1 with
$P=2$ fixed. Each column is an independent set of 10-split runs; at each
dataset's selected $\lambda_{\mathrm{reg}}$ the cell agrees with the $P{=}2$
column of Table~\ref{tab:p_brier} within one standard deviation. The
regularizer of~\eqref{eq:loss_base} penalizes large chaos coefficients; at
$\lambda_{\mathrm{reg}}=0$ the chaos expansion is still active (all $P+1$
channels are learned), but without an explicit energy penalty. Moderate
regularization ($\lambda_{\mathrm{reg}} \in \{10^{-3}, 10^{-2}\}$) yields the
best calibration on the majority of datasets. On Texas,
$\lambda_{\mathrm{reg}}$ matters most ($0.366$ at $0$ versus $0.213$ at
$10^{-2}$); on Citeseer, Wisconsin, and Roman-Empire
$\lambda_{\mathrm{reg}}{=}0$ is best and larger values degrade Brier, which is
why $\lambda_{\mathrm{reg}}$ is selected per dataset on validation Brier.

\begin{table}[ht]
\centering
\small
\caption{Effect of chaos energy regularization $\lambda_{\mathrm{reg}}$ on
Brier score ($\downarrow$), with $P=2$ fixed. Bold indicates the best
$\lambda_{\mathrm{reg}}$ per dataset.}
\label{tab:lreg_ablation}
\begin{tabular}{l cccc}
\toprule
Dataset & $\lambda_{\mathrm{reg}}{=}0$ & $10^{-3}$ & $10^{-2}$ & $10^{-1}$ \\
\midrule
Cora        & 0.219 & 0.219 & \textbf{0.207} & 0.220 \\
Citeseer    & \textbf{0.308} & 0.310 & 0.318 & 0.499 \\
PubMed      & 0.159 & \textbf{0.156} & \textbf{0.156} & 0.157 \\
Texas       & 0.366 & 0.246 & \textbf{0.213} & 0.251 \\
Cornell     & 0.230 & 0.237 & \textbf{0.228} & 0.342 \\
Wisconsin   & \textbf{0.103} & 0.107 & 0.108 & 0.174 \\
Chameleon   & 0.362 & 0.354 & \textbf{0.339} & 0.342 \\
Squirrel    & 0.405 & 0.407 & 0.398 & \textbf{0.393} \\
CS          & 0.064 & 0.066 & \textbf{0.062} & 0.083 \\
\midrule
Roman-Emp.  & \textbf{0.298} & 0.323 & 0.603 & 0.789 \\
Amazon-Rat. & 0.663 & 0.660 & \textbf{0.633} & \textbf{0.633} \\
Minesweeper & 0.186 & \textbf{0.169} & 0.190 & 0.228 \\
Tolokers    & 0.281 & 0.275 & \textbf{0.268} & 0.280 \\
Questions   & 0.054 & \textbf{0.053} & 0.054 & 0.055 \\
\bottomrule
\end{tabular}
\end{table}

\subsection{Number of quadrature nodes \texorpdfstring{$S$}{S}}

Table~\ref{tab:S_ablation} reports the Brier score as the number of
Gauss--Hermite quadrature nodes $S$ varies from 2 to 8, with $P=2$ fixed.
Each column is an independent set of 10-split runs; the $S{=}4$ column agrees
with the $P{=}2$ column of Table~\ref{tab:p_brier} within one standard
deviation.
For the deterministic gates used in these runs ($P_g=0$), the theoretical
minimum for exact chaos arithmetic at $P=2$ is
$S\ge\lceil(P_g+2P+1)/2\rceil = 3$ (Appendix~\ref{sec:chaos_background}).
Brier varies by at most $0.011$ across $S\in\{2,\dots,8\}$ on 12 of 14
datasets. The exceptions are Wisconsin, where $S{=}2$, below the theoretical
minimum, gives a higher Brier score ($0.142$ vs $0.109$ at $S{=}3$),
consistent with the exactness requirement of
Theorem~\ref{thm:nonintrusive_equals_intrusive}, and Cornell, where the
spread across $S$ ($0.218$ to $0.286$) is within one standard deviation of
its 10-split results.
The chaos-order sweeps of Tables~\ref{tab:p_acc}--\ref{tab:p_brier} use
$S{=}4$ throughout, exact for $P\le3$; at $P\ge4$ the highest-order Hermite
products are integrated approximately (Appendix~\ref{app:implementation}).

\begin{table}[ht]
\centering
\small
\caption{Brier score ($\downarrow$) vs.\ number of quadrature nodes $S$, with
$P=2$ fixed. The theoretical minimum is $S=3$.}
\label{tab:S_ablation}
\begin{tabular}{l ccccc}
\toprule
Dataset & $S{=}2$ & $S{=}3$ & $S{=}4$ & $S{=}6$ & $S{=}8$ \\
\midrule
Cora        & 0.216 & \textbf{0.207} & 0.217 & 0.209 & 0.212 \\
Citeseer    & 0.310 & \textbf{0.308} & 0.312 & 0.313 & 0.311 \\
PubMed      & 0.157 & \textbf{0.156} & \textbf{0.156} & 0.158 & \textbf{0.156} \\
Texas       & 0.216 & 0.221 & \textbf{0.213} & 0.217 & 0.218 \\
Cornell     & 0.235 & 0.286 & 0.228 & \textbf{0.218} & 0.242 \\
Wisconsin   & 0.142 & 0.109 & 0.115 & \textbf{0.103} & 0.108 \\
Chameleon   & 0.341 & 0.344 & 0.343 & \textbf{0.339} & 0.340 \\
Squirrel    & 0.395 & 0.397 & \textbf{0.393} & 0.396 & 0.400 \\
CS          & \textbf{0.062} & \textbf{0.062} & \textbf{0.062} & 0.063 & 0.063 \\
\midrule
Roman-Emp.  & \textbf{0.298} & \textbf{0.298} & \textbf{0.298} & \textbf{0.298} & \textbf{0.298} \\
Amazon-Rat. & \textbf{0.633} & 0.634 & 0.636 & 0.636 & 0.637 \\
Minesweeper & 0.170 & \textbf{0.169} & 0.170 & 0.171 & 0.170 \\
Tolokers    & 0.269 & 0.269 & 0.273 & 0.270 & \textbf{0.268} \\
Questions   & 0.055 & \textbf{0.053} & 0.054 & 0.055 & 0.055 \\
\bottomrule
\end{tabular}
\end{table}

\subsection{Contribution decomposition: backbone vs.\ chaos}

Table~\ref{tab:decompose} decomposes the calibration improvement into two
factors: the graph-spectral backbone and the chaos expansion. The ``ChebNet''
column reports a vanilla single-branch Chebyshev filter ($K{=}4$, hidden 64)
as a diagnostic backbone baseline, not as part of the main uncertainty-aware
baseline comparison in Table~\ref{tab:acc_brier}. The ``$P{=}0$'' column uses
the Table~\ref{tab:acc_brier}
architecture for each dataset (which may include propagate-first or dual
branches) with the chaos expansion disabled. The ``Best $P{>}0$'' column adds
the chaos expansion on the same architecture. The $P{=}0 \to P{>}0$ comparison
is strictly fair (same architecture, only $P$ changes) and shows that the chaos
expansion improves calibration on 12 of 14 datasets ($3.6$--$27.9\%$
Brier reduction), indicating that the stochastic expansion is the primary
source of the calibration gains over the $P{=}0$ version in this controlled
comparison. The ChebNet column gives useful context: on the small WebKB
datasets Texas and Wisconsin, the single-filter ChebNet attains a lower
absolute Brier score, while the controlled $P{=}0$ versus $P{>}0$
comparison still isolates the contribution of the chaos expansion. The
ChebNet column is also the single-filter control for the filtering
design: at $P{=}0$ the per-dataset dual-filter architecture has lower or
equal Brier than the single-branch ChebNet on 12 of 14 datasets (equal on
Questions; Chameleon $0.362$ vs
$0.776$, Squirrel $0.416$ vs $0.798$, Roman-Empire $0.298$ vs $0.368$), with
Texas and Wisconsin as the exceptions noted above.

\begin{table}[ht]
\centering
\small
\caption{Contribution decomposition: Brier score ($\downarrow$). ChebNet
$K{=}4$ is a vanilla single-filter diagnostic; $P{=}0$ and Best $P{>}0$ are
the $P{=}0$ column and the best $P{>}0$ column of Table~\ref{tab:p_brier}
(same per-dataset architecture). Boldface compares
the controlled $P{=}0$ versus Best $P{>}0$ pair, and the rightmost column
isolates the chaos contribution. Table~\ref{tab:acc_brier} reports the
selected order of Table~\ref{tab:selection}, which can differ from the best
swept order.}
\label{tab:decompose}
\begin{tabular}{l ccc c}
\toprule
Dataset & ChebNet $K{=}4$ & $P{=}0$ & Best $P{>}0$ & $\Delta$ Chaos \\
\midrule
Cora        & 0.281 & 0.228 & \textbf{0.207} & $-9.2\%$ \\
Citeseer    & 0.470 & 0.427 & \textbf{0.308} & $-27.9\%$ \\
PubMed      & 0.171 & 0.162 & \textbf{0.156} & $-3.7\%$ \\
Texas       & 0.208 & 0.269 & \textbf{0.213} & $-20.8\%$ \\
Cornell     & 0.324 & 0.241 & \textbf{0.228} & $-5.4\%$ \\
Wisconsin   & 0.073 & \textbf{0.103} & \textbf{0.103} & $\phantom{-}0\%$ \\
Chameleon   & 0.776 & 0.362 & \textbf{0.339} & $-6.4\%$ \\
Squirrel    & 0.798 & 0.416 & \textbf{0.393} & $-5.5\%$ \\
CS          & 0.069 & 0.065 & \textbf{0.062} & $-4.6\%$ \\
\midrule
Roman-Emp.  & 0.368 & \textbf{0.298} & \textbf{0.298} & $\phantom{-}0\%$ \\
Amazon-Rat. & 0.676 & 0.666 & \textbf{0.633} & $-5.0\%$ \\
Minesweeper & 0.208 & 0.190 & \textbf{0.169} & $-11.1\%$ \\
Tolokers    & 0.303 & 0.278 & \textbf{0.268} & $-3.6\%$ \\
Questions   & 0.055 & 0.055 & \textbf{0.053} & $-3.6\%$ \\
\bottomrule
\end{tabular}
\end{table}

\subsection{Computational overhead}

\paragraph{Asymptotic complexity.}
Although the doubly-spectral expansion is defined on the tensor product
$\ell^2(V)\otimes L^2(\mathbb{R},\gamma)$, truncating the chaos basis at
order~$P$ retains only $P{+}1$ coefficients along the stochastic axis, so the
graph-filtering cost and the coefficient memory scale as $(P{+}1)\times$ the
deterministic ChebNet baseline and the node-wise dense cost as $S\times$.
Concretely, one DSS layer takes
$\mathcal{O}\!\big((P{+}1)(K_{\mathrm{lp}}+K_{\mathrm{hp}})|E|d + SNd^2
+ S(P{+}1)Nd\big)$ time and $\mathcal{O}((P{+}1+S)Nd)$ memory, where $|E|$ is
the number of edges, $d$ the hidden width, and $S$ the number of
Gauss--Hermite quadrature nodes; the chaos extension is therefore
\emph{additive}, not combinatorial, in the stochastic direction, and no
eigendecomposition is ever computed. With $S$ at the exactness threshold
($S=P+1$) the quadrature term is quadratic in $P$; this is the worst-case
operation count, and the measured growth in $P$ is much smaller (below). By
Theorem~\ref{thm:universal_approx} the truncation error decays exponentially
in $P$ for targets satisfying the growth condition of
Appendix~\ref{sec:proof_decay}, consistent with the empirical plateau at
$P{=}1$--$2$, so the constant factor over a deterministic backbone stays
small. All chaos coefficients share the same Chebyshev filters and projection
weights, so the parameter count per layer is
$\mathcal{O}\big(K_{\mathrm{lp}}+K_{\mathrm{hp}}+(P{+}1)(P_g{+}1)+d_\ell d_{\ell+1}\big)$
(filter coefficients, gate coefficients $a_{n,r},b_{n,r}$, and the linear
map); the chaos order enters only through the gate coefficients and the input
lift. By contrast, the MC-dropout and Deep Ensemble baselines require $M$
forward passes or models at inference ($M{=}20$ and $M{=}5$ in our
comparisons).

\paragraph{Wall-clock measurements.}
Table~\ref{tab:timing} reports the wall-clock time per training epoch as the
chaos order $P$ increases, on three datasets spanning two orders of
magnitude in size. The overhead grows linearly with $P$: each additional
chaos order adds one extra graph filter evaluation and Gauss--Hermite
projection per layer. At $P=2$ the per-epoch cost ranges from $2.9\times$ on
Cora to $1.2\times$ on Amazon-Ratings and $1.1\times$ on ogbn-arxiv relative
to the deterministic baseline ($P=0$), within the asymptotic bound above
and well below the $M\times$ overhead of MC-dropout or ensembles. The
measured overhead decreases as graphs grow because the per-order arithmetic
is batched into single kernels, and on larger graphs fixed per-epoch costs
and better accelerator utilization account for a larger share of the epoch
time. On ogbn-arxiv (169K nodes, one H100 GPU) each additional chaos order
adds about $4\%$ per-epoch time, and peak GPU memory grows linearly in
$P{+}1$ ($1.68/2.75/3.88/5.02$\,GiB for $P=0,\dots,3$), consistent with the memory analysis above; peak memory on Cora is
$75.9$\,MiB.

\begin{table}[ht]
\centering
\small
\caption{Wall-clock time per epoch (ms) vs.\ chaos order $P$ (Cora and
Amazon-Ratings on an A100; ogbn-arxiv on one H100, 3 runs, with peak GPU
memory in MiB).}
\label{tab:timing}
\begin{tabular}{l cccc c}
\toprule
Dataset & $P{=}0$ & $P{=}1$ & $P{=}2$ & $P{=}3$ & Ratio $P{=}2$/$P{=}0$ \\
\midrule
Cora (2.7K nodes)        & 35 & 67 & 102 & 131 & $2.9\times$ \\
Amazon-Rat. (25K nodes)  & 464 & 526 & 576 & 646 & $1.2\times$ \\
ogbn-arxiv (169K nodes)  & 2002 & 2088 & 2174 & 2273 & $1.1\times$ \\
ogbn-arxiv peak memory (MiB) & 1721 & 2813 & 3976 & 5141 & $2.3\times$ \\
\bottomrule
\end{tabular}
\end{table}

\subsection{Summary of component evidence}
\label{app:component_summary}

Table~\ref{tab:component_summary} collects, for each architectural
component, the controlled comparison that isolates it and the effect
measured. The chaos expansion improves calibration (and, at the selected GOOD
configurations, shifted accuracy); the spectral filtering design and the
residual improve accuracy under heterophily and shift; the energy score with
propagation provides OOD detection. Each task's readout uses a different part
of the same representation.

\begin{table}[ht]
\centering
\scriptsize
\caption{Component-wise evidence. Each row names one component, the
comparison that isolates it, and the measured effect.}
\label{tab:component_summary}
\begin{tabular}{>{\raggedright\arraybackslash}p{0.21\textwidth} >{\raggedright\arraybackslash}p{0.28\textwidth} >{\raggedright\arraybackslash}p{0.40\textwidth}}
\toprule
Component & Comparison & Effect \\
\midrule
Chaos expansion ($P$) & $P{=}0$ vs best $P{>}0$, same architecture
(Tables~\ref{tab:p_acc}, \ref{tab:p_brier}, \ref{tab:decompose}) &
Calibration: Citeseer Brier $-27.9\%$ and accuracy $+14.8$ points (the
endpoint also beats plain GCN, $0.308$ vs $0.318$ Brier,
Tables~\ref{tab:acc_brier}/\ref{tab:simple_baselines}); Texas $-20.8\%$;
Minesweeper $-11.1\%$ \\
Chaos expansion ($P$) & OOD ablation varying only $P$
(Appendix~\ref{app:chaos_ood_good}) & Energy AUROC flat across $P$
(Cora-structure $91.0/91.2/90.9/90.8$ at one fixed configuration); detection
reads the mean logit, so $P$ can be selected for calibration without changing
detection \\
Chaos expansion ($P$) & GOOD ablation at the selected configurations
(Table~\ref{tab:good_pabl}) & Shifted accuracy: $P{=}1$ beats $P{=}0$ on
GOOD-Cora word ($+1.7$), degree ($+1.2$), and GOOD-Arxiv time ($+0.6$) \\
Spectral backbone (chaos off) & $P{=}0$ vs GCN
(Tables~\ref{tab:p_acc}, \ref{tab:simple_baselines}) & Heterophilous
accuracy comes from the spectral backbone: Roman-Empire $79.0$ at $P{=}0$ vs
GCN $51.3$ \\
Filtering design vs single filter & Per-dataset architecture at $P{=}0$ vs
single-branch ChebNet (Table~\ref{tab:decompose}) & Lower or equal Brier for the full
design on 12 of 14 datasets (equal on Questions; Chameleon $0.362$ vs $0.776$, Squirrel $0.416$
vs $0.798$, Roman-Empire $0.298$ vs $0.368$); exceptions are Texas and Wisconsin \\
DSS residual branch (as a unit) & Hybrid vs identically trained GCN base, 14
datasets (Table~\ref{tab:crosseval_calibration}) & Accuracy: Roman-Empire
$+27.7$, Minesweeper $+6.6$; no dataset loses beyond noise (worst $-0.8$) \\
DSS representation under a fixed OOD readout & DSS-Hybrid vs GNNSafe++ (same
energy readout, propagation, objective, GCN encoder in both;
Table~\ref{tab:gnnsafe_local_table2style}) & Cora $94.3/97.6/94.1$ vs
$90.6/95.6/92.8$ \\
Energy-score propagation & Same-checkpoint on/off, all 11 settings
(Table~\ref{tab:prop_ablation}) & Helps on 9 of 11, up to $+18.2$ AUROC
(Twitch) and $+14.5$ (Cora-structure) \\
\bottomrule
\end{tabular}
\end{table}

\clearpage
\section{Comparison with Standard GNN Baselines}
\label{app:simple_baselines}

Table~\ref{tab:acc_brier} in the main text compares DSS-GNN against
recent uncertainty-aware graph methods (TFE-GNN and
G-$\Delta$UQ).  Here we additionally compare against
GCN~\citep{kipf2017gcn} and GAT~\citep{velickovic2018gat}, two widely used GNN
architectures that do not incorporate any uncertainty mechanism.  Both
baselines use 2 layers, 64 hidden units, Adam optimizer
(lr$=$0.01, weight decay$=$5$\times$10$^{-4}$), and the same 10 splits as DSS-GNN.

Table~\ref{tab:simple_baselines} reports accuracy and Brier score.  The main
observation is that \emph{higher accuracy does not imply better calibration}.
On Cora, GCN and GAT both achieve higher accuracy than DSS-GNN ($87.5\%$ and
$87.8\%$ vs.\ $86.6\%$) yet have worse Brier scores ($0.213$ and $0.222$
vs.\ $0.207$).  On Citeseer, GAT matches DSS-GNN in accuracy ($80.6\%$
vs.\ $80.2\%$) but is substantially worse in Brier ($0.353$ vs.\ $0.308$).
On many heterophilous benchmarks, GCN and GAT have much lower accuracy and
higher Brier scores than DSS-GNN.  The calibration gains of DSS-GNN are therefore not an accuracy effect.

\paragraph{Stochastic baselines under the same protocol.}
Table~\ref{tab:stoch_calibration} adds MC-dropout ($M{=}20$ stochastic
passes at test time, probabilities averaged) and Deep Ensembles ($M{=}5$
independently trained models, probabilities averaged) on the GCN of
Table~\ref{tab:simple_baselines}, under the full 10-split protocol and the
same code path; the plain-GCN control from that code path reproduces
Table~\ref{tab:simple_baselines} (Cora $87.47$ vs $87.45$, Wisconsin
$50.25$ vs $50.25$). DSS-GNN's Brier is lower on all 14 datasets than either
baseline, and ensembling improves the GCN's Brier by at most $0.008$ on any
dataset while MC-dropout sometimes worsens it (Cora): averaging a
conventional GNN does not remove the calibration difference. On Cora and Questions
the margins are small ($0.207$ vs $0.213$ and $0.053$ vs $0.055$). The same
multi-pass cost structure applies to sampling-based Bayesian GNNs and
latent-variable GNNs; the chaos expansion is instead deterministic, with no
sampling variance.

\begin{table}[ht]
\centering
\small
\caption{Brier score ($\downarrow$) with test accuracy (\%) in parentheses
for DSS-GNN (Table~\ref{tab:acc_brier}) and stochastic GCN baselines under
the same 10-split protocol (means over 10 splits).}
\label{tab:stoch_calibration}
\begin{tabular}{l cccc}
\toprule
Dataset & DSS-GNN & MC-dropout & Deep Ensemble & Plain GCN \\
\midrule
Cora        & \textbf{0.207} (86.63) & 0.2188 (87.36) & 0.2134 (87.45) & 0.2125 (87.47) \\
Citeseer    & \textbf{0.308} (80.20) & 0.3230 (80.03) & 0.3170 (79.77) & 0.3176 (80.00) \\
PubMed      & \textbf{0.156} (89.72) & 0.2176 (86.20) & 0.2153 (86.16) & 0.2155 (86.10) \\
Texas       & \textbf{0.240} (91.31) & 0.5862 (64.43) & 0.5686 (62.79) & 0.5758 (64.43) \\
Cornell     & \textbf{0.231} (85.11) & 0.6829 (52.77) & 0.6826 (52.98) & 0.6839 (52.34) \\
Wisconsin   & \textbf{0.104} (93.25) & 0.6370 (50.38) & 0.6373 (51.25) & 0.6435 (50.25) \\
Chameleon   & \textbf{0.339} (75.36) & 0.7378 (41.73) & 0.7371 (42.71) & 0.7374 (42.21) \\
Squirrel    & \textbf{0.394} (68.06) & 0.7850 (28.11) & 0.7847 (28.81) & 0.7849 (28.38) \\
CS          & \textbf{0.062} (96.17) & 0.0821 (94.55) & 0.0809 (94.63) & 0.0816 (94.58) \\
\midrule
Roman-Emp.  & \textbf{0.300} (78.84) & 0.6554 (50.92) & 0.6471 (51.49) & 0.6477 (51.37) \\
Amz-Rat.    & \textbf{0.634} (49.72) & 0.6772 (44.08) & 0.6769 (44.00) & 0.6771 (44.07) \\
Minesweeper & \textbf{0.169} (87.66) & 0.2874 (80.20) & 0.2874 (80.20) & 0.2874 (80.21) \\
Tolokers    & \textbf{0.268} (79.88) & 0.2970 (78.78) & 0.2970 (78.78) & 0.2971 (78.79) \\
Questions   & \textbf{0.053} (97.17) & 0.0548 (97.06) & 0.0549 (97.07) & 0.0549 (97.07) \\
\bottomrule
\end{tabular}
\end{table}

\begin{table*}[ht]
\centering
\small
\caption{Node classification: test accuracy (\%, $\uparrow$) and Brier score
($\downarrow$) for DSS-GNN vs.\ standard GNN baselines. Mean $\pm$ SD over
10 splits.  Best per dataset in \textbf{bold}.}
\label{tab:simple_baselines}
\resizebox{\textwidth}{!}{%
\begin{tabular}{l ccc ccc}
\toprule
 & \multicolumn{3}{c}{Accuracy (\%)} & \multicolumn{3}{c}{Brier Score} \\
\cmidrule(lr){2-4} \cmidrule(lr){5-7}
Dataset & DSS-GNN & GCN & GAT & DSS-GNN & GCN & GAT \\
\midrule
Cora        & 86.63 $\pm$ 1.26 & 87.45 $\pm$ 1.61 & \textbf{87.75} $\pm$ 1.71 & \textbf{0.207} $\pm$ 0.019 & 0.213 $\pm$ 0.014 & 0.222 $\pm$ 0.015 \\
Citeseer    & 80.20 $\pm$ 1.28 & 79.84 $\pm$ 0.96 & \textbf{80.57} $\pm$ 1.12 & \textbf{0.308} $\pm$ 0.010 & 0.318 $\pm$ 0.007 & 0.353 $\pm$ 0.007 \\
PubMed      & \textbf{89.72} $\pm$ 0.31 & 86.21 $\pm$ 0.31 & 85.72 $\pm$ 0.28 & \textbf{0.156} $\pm$ 0.005 & 0.215 $\pm$ 0.005 & 0.232 $\pm$ 0.004 \\
Texas       & \textbf{91.31} $\pm$ 3.44 & 64.43 $\pm$ 6.80 & 71.80 $\pm$ 7.06 & \textbf{0.240} $\pm$ 0.111 & 0.576 $\pm$ 0.031 & 0.555 $\pm$ 0.038 \\
Cornell     & \textbf{85.11} $\pm$ 5.12 & 52.34 $\pm$ 9.99 & 49.57 $\pm$ 12.63 & \textbf{0.231} $\pm$ 0.079 & 0.684 $\pm$ 0.080 & 0.631 $\pm$ 0.073 \\
Wisconsin   & \textbf{93.25} $\pm$ 3.39 & 50.25 $\pm$ 5.30 & 56.62 $\pm$ 4.71 & \textbf{0.104} $\pm$ 0.047 & 0.643 $\pm$ 0.024 & 0.598 $\pm$ 0.021 \\
Chameleon   & \textbf{75.36} $\pm$ 1.19 & 42.21 $\pm$ 1.78 & 43.83 $\pm$ 2.54 & \textbf{0.339} $\pm$ 0.019 & 0.737 $\pm$ 0.009 & 0.712 $\pm$ 0.015 \\
Squirrel    & \textbf{68.06} $\pm$ 2.21 & 28.39 $\pm$ 1.17 & 27.57 $\pm$ 0.75 & \textbf{0.394} $\pm$ 0.027 & 0.785 $\pm$ 0.003 & 0.789 $\pm$ 0.002 \\
CS          & \textbf{96.17} $\pm$ 0.23 & 94.61 $\pm$ 0.24 & 94.45 $\pm$ 0.18 & \textbf{0.062} $\pm$ 0.004 & 0.082 $\pm$ 0.004 & 0.085 $\pm$ 0.004 \\
\midrule
Roman-Emp.  & \textbf{78.84} $\pm$ 0.61 & 51.30 $\pm$ 0.31 & 42.85 $\pm$ 0.53 & \textbf{0.300} $\pm$ 0.007 & 0.648 $\pm$ 0.001 & 0.755 $\pm$ 0.003 \\
Amz-Rat.    & \textbf{49.72} $\pm$ 0.64 & 43.97 $\pm$ 0.37 & 42.20 $\pm$ 0.43 & \textbf{0.634} $\pm$ 0.005 & 0.677 $\pm$ 0.002 & 0.690 $\pm$ 0.001 \\
Minesweeper & \textbf{87.66} $\pm$ 1.10 & 80.21 $\pm$ 0.19 & 80.10 $\pm$ 0.08 & \textbf{0.169} $\pm$ 0.014 & 0.287 $\pm$ 0.003 & 0.281 $\pm$ 0.003 \\
Tolokers    & \textbf{79.88} $\pm$ 0.49 & 78.81 $\pm$ 0.47 & 78.19 $\pm$ 0.16 & \textbf{0.268} $\pm$ 0.005 & 0.297 $\pm$ 0.005 & 0.299 $\pm$ 0.002 \\
Questions   & \textbf{97.17} $\pm$ 0.04 & 97.08 $\pm$ 0.02 & 97.02 $\pm$ 0.00 & \textbf{0.053} $\pm$ 0.001 & 0.055 $\pm$ 0.000 & 0.057 $\pm$ 0.000 \\
\bottomrule
\end{tabular}%
}
\end{table*}

\clearpage
\section{Cross-Evaluation of the Two Deployment Modes}
\label{app:crosseval}

This appendix gives the full tables for Section~\ref{sec:exp_crosseval}:
the hybrid under the calibration protocol, the standalone model under the OOD
and GOOD protocols, the readout-coincidence measurement of
Section~\ref{sec:exp_calibration}, the effect of BatchNorm under the
calibration protocol, and the covariate-shift counterpart of
Table~\ref{tab:good_table2_style}.

Table~\ref{tab:head_to_head} (Section~\ref{sec:exp_crosseval}) summarizes
both deployment modes on all three tasks; the per-dataset comparisons
summarized in its cells are Tables~\ref{tab:crosseval_calibration},
\ref{tab:crosseval_ood}, and~\ref{tab:crosseval_good} below.

\subsection{Protocols}
\label{app:crosseval_protocol}

\paragraph{Hybrid under the calibration protocol.}
DSS-Hybrid (2-layer GCN base with the DSS residual branch of
Appendix~\ref{app:implementation}) is trained under the identical 10-split
protocol of Table~\ref{tab:acc_brier}, paired with its own GCN base trained
identically with the residual removed and matched budgets, so the pair
isolates the residual branch.

\paragraph{Standalone under the OOD protocol.}
Standalone DSS-GNN is trained in the GNNSafe pipeline (200 epochs, 3 runs,
energy score with propagation, pipeline defaults: hidden 64, 2 layers,
weight decay $10^{-2}$, learning rate $10^{-2}$, $P{=}1$, $S{=}4$). Rows
marked $+$BN add the per-chaos-channel BatchNorm of
Appendix~\ref{app:implementation}, with every other hyperparameter
unchanged; the variant is selected on the validation split. Twitch uses learning rate $10^{-3}$, hidden 128, and no BatchNorm, selected
on validation ID accuracy. The Arxiv row is trained with the GNNSafe++
energy-margin objective at the Arxiv margin hyperparameters of
\citet{wu2023gnnsafe} with $m_{\mathrm{out}}{=}-1$, selected on validation
ID accuracy. All other rows use standard ID training.

\paragraph{Standalone under the GOOD protocol.}
Standalone DSS-GNN is trained in the GOOD pipeline with $L_1$ row-normalized
input features, 2 layers, $K_{\mathrm{lp}}{=}3$, $K_{\mathrm{hp}}{=}2$,
500 epochs, patience 200, 3 seeds; the epoch is selected by OOD-validation
loss within each run and the configuration (BatchNorm on/off, learning rate,
hidden width, $P$) on the GOOD OOD-validation split: CBAS and WebKB
$+$BN, learning rate $10^{-2}$; Twitch $+$BN, $10^{-3}$; GOOD-Cora/word
$+$BN, $5\times10^{-3}$, hidden 128, weight decay $10^{-3}$; GOOD-Cora/degree
$+$BN, $3\times10^{-3}$, hidden 128; GOOD-Arxiv/time $+$BN, $10^{-2}$, hidden
256; GOOD-Arxiv/degree no BatchNorm, $10^{-2}$, hidden 64; $P{=}1$
throughout.

\subsection{Calibration: standalone, hybrid, and the hybrid's GCN base}

\begin{table}[ht]
\centering
\scriptsize
\caption{Test accuracy (\%) and Brier score, mean $\pm$ SD over the same 10
splits. The standalone column is Table~\ref{tab:acc_brier}; the hybrid and
its GCN base are a matched pair (identical training, budgets, and splits;
the only difference is the DSS residual).}
\label{tab:crosseval_calibration}
\resizebox{\textwidth}{!}{%
\begin{tabular}{l cc cc cc}
\toprule
 & \multicolumn{2}{c}{Standalone DSS-GNN} & \multicolumn{2}{c}{DSS-Hybrid} & \multicolumn{2}{c}{GCN base} \\
\cmidrule(lr){2-3}\cmidrule(lr){4-5}\cmidrule(lr){6-7}
Dataset & Accuracy & Brier & Accuracy & Brier & Accuracy & Brier \\
\midrule
Cora           & 86.63 $\pm$ 1.26 & 0.207 $\pm$ 0.019 & 84.33 $\pm$ 0.97  & 0.241 $\pm$ 0.015 & 84.20 $\pm$ 1.02  & 0.243 $\pm$ 0.016 \\
Citeseer       & 80.20 $\pm$ 1.28 & 0.308 $\pm$ 0.010 & 74.71 $\pm$ 1.05  & 0.375 $\pm$ 0.009 & 74.32 $\pm$ 1.13  & 0.378 $\pm$ 0.009 \\
PubMed         & 89.72 $\pm$ 0.31 & 0.156 $\pm$ 0.005 & 88.27 $\pm$ 0.56  & 0.177 $\pm$ 0.006 & 87.93 $\pm$ 0.56  & 0.181 $\pm$ 0.006 \\
Texas          & 91.31 $\pm$ 3.44 & 0.240 $\pm$ 0.111 & 45.57 $\pm$ 7.68  & 0.737 $\pm$ 0.034 & 46.39 $\pm$ 7.55  & 0.737 $\pm$ 0.034 \\
Cornell        & 85.11 $\pm$ 5.12 & 0.231 $\pm$ 0.079 & 55.11 $\pm$ 14.03 & 0.656 $\pm$ 0.142 & 47.66 $\pm$ 12.01 & 0.727 $\pm$ 0.068 \\
Wisconsin      & 93.25 $\pm$ 3.39 & 0.104 $\pm$ 0.047 & 48.12 $\pm$ 6.50  & 0.680 $\pm$ 0.029 & 48.50 $\pm$ 6.14  & 0.680 $\pm$ 0.028 \\
Chameleon      & 75.36 $\pm$ 1.19 & 0.339 $\pm$ 0.019 & 33.50 $\pm$ 2.10  & 0.777 $\pm$ 0.004 & 34.00 $\pm$ 2.44  & 0.777 $\pm$ 0.004 \\
Squirrel       & 68.06 $\pm$ 2.21 & 0.394 $\pm$ 0.027 & 26.35 $\pm$ 2.02  & 0.790 $\pm$ 0.003 & 26.23 $\pm$ 2.36  & 0.790 $\pm$ 0.003 \\
CS             & 96.17 $\pm$ 0.23 & 0.062 $\pm$ 0.004 & 94.26 $\pm$ 0.24  & 0.087 $\pm$ 0.004 & 94.23 $\pm$ 0.21  & 0.088 $\pm$ 0.003 \\
\midrule
Roman-Empire   & 78.84 $\pm$ 0.61 & 0.300 $\pm$ 0.007 & 74.36 $\pm$ 0.61  & 0.361 $\pm$ 0.007 & 46.63 $\pm$ 0.26  & 0.690 $\pm$ 0.003 \\
Amazon-Ratings & 49.72 $\pm$ 0.64 & 0.634 $\pm$ 0.005 & 46.36 $\pm$ 0.57  & 0.661 $\pm$ 0.004 & 46.32 $\pm$ 0.48  & 0.661 $\pm$ 0.002 \\
Minesweeper    & 87.66 $\pm$ 1.10 & 0.169 $\pm$ 0.014 & 86.91 $\pm$ 0.45  & 0.184 $\pm$ 0.006 & 80.31 $\pm$ 0.21  & 0.289 $\pm$ 0.003 \\
Tolokers       & 79.88 $\pm$ 0.49 & 0.268 $\pm$ 0.005 & 81.19 $\pm$ 0.50  & 0.255 $\pm$ 0.006 & 79.46 $\pm$ 0.34  & 0.291 $\pm$ 0.006 \\
Questions      & 97.17 $\pm$ 0.04 & 0.053 $\pm$ 0.001 & 97.09 $\pm$ 0.04  & 0.053 $\pm$ 0.001 & 97.08 $\pm$ 0.04  & 0.054 $\pm$ 0.001 \\
\bottomrule
\end{tabular}%
}
\end{table}

The hybrid is at or above its own base on 11 of 14 datasets in accuracy; the
three negative deltas (Texas $-0.82$, Chameleon $-0.50$, Wisconsin $-0.38$)
are far inside one standard deviation, and every positive Brier delta
(hybrid worse) is at most $0.0008$, also far inside one standard deviation.
The gains are structural and large where the GCN inductive bias fails:
Roman-Empire $+27.7$ accuracy and $-0.33$ Brier, Minesweeper $+6.6$ and
$-0.11$, Tolokers $+1.7$ and $-0.04$. On the small heterophilous graphs where
the GCN backbone itself fails (Texas, Cornell, Wisconsin, Chameleon,
Squirrel), the residual matches the base rather than improving on it, and the
standalone model is the deployment choice. In hybrid mode the argmax of the
quadrature predictive and of the mean logit agree on every test node of
every dataset (disagreement $0.0000$ on all 14).

\subsection{OOD detection: standalone versus hybrid}

\begin{table}[ht]
\centering
\scriptsize
\caption{AUROC (\%) and ID test accuracy (\%) on the 11 OOD settings.
Standalone: 3 runs, mean $\pm$ SD; hybrid:
Tables~\ref{tab:gnnsafe_local_table2style}--\ref{tab:gnnsafe_local_table1style}
(AUROC only for the nine node-OOD settings$^{*}$). The last column gives the
standalone training objective.}
\label{tab:crosseval_ood}
\resizebox{\textwidth}{!}{%
\begin{tabular}{l cc cc l}
\toprule
Setting & Standalone AUROC & Standalone ID acc. & Hybrid AUROC & Hybrid ID acc. & Objective (standalone) \\
\midrule
Cora / structure         & 79.79 $\pm$ 1.46 ($+$BN) & 69.50 & 94.32 & n/r$^{*}$ & standard (no OOD exposure) \\
Cora / feature           & 88.15 $\pm$ 1.02 ($+$BN) & 69.17 & 97.60 & n/r$^{*}$ & standard (no OOD exposure) \\
Cora / label             & 94.03 $\pm$ 1.09 ($+$BN) & 89.24 & 94.11 & n/r$^{*}$ & standard (no OOD exposure) \\
Amazon-Photo / structure & 96.53 $\pm$ 1.41 ($+$BN) & 91.59 & 99.69 & n/r$^{*}$ & standard (no OOD exposure) \\
Amazon-Photo / feature   & 98.10 $\pm$ 0.11 ($+$BN) & 92.00 & 99.66 & n/r$^{*}$ & standard (no OOD exposure) \\
Amazon-Photo / label     & 96.53 $\pm$ 0.10 ($+$BN) & 94.81 & 97.52 & n/r$^{*}$ & standard (no OOD exposure) \\
Coauthor-CS / structure  & 94.33 $\pm$ 0.29 ($+$BN) & 92.64 & 99.96 & n/r$^{*}$ & standard (no OOD exposure) \\
Coauthor-CS / feature    & 97.79 $\pm$ 0.16 ($+$BN) & 92.47 & 99.95 & n/r$^{*}$ & standard (no OOD exposure) \\
Coauthor-CS / label      & 95.87 $\pm$ 1.83 ($+$BN) & 94.00 & 98.07 & n/r$^{*}$ & standard (no OOD exposure) \\
Twitch (cross-graph)     & 51.85 $\pm$ 4.83 (no separation) & 68.39 $\pm$ 0.20 & 95.75 & 70.52 & standard (no OOD exposure) \\
Arxiv (cross-graph)      & 73.73 $\pm$ 0.24 ($+$BN) & 61.16 $\pm$ 0.58 & 73.36 & 53.99 & GNNSafe++ margin objective \\
\bottomrule
\end{tabular}%
}
\end{table}

$^{*}$Table~\ref{tab:gnnsafe_local_table2style} reports AUROC only. Hybrid ID accuracy on these nine settings (3 runs under the same
protocol) is Cora $81.6/81.1/90.6$,
Amazon-Photo $92.4/91.0/95.4$, Coauthor-CS $93.0/93.0/95.5$
(structure/feature/label).

The hybrid is the stronger detector on 10 of 11 settings; the standalone
trains on all 11 (ID accuracies in Table~\ref{tab:crosseval_ood}) and is competitive on the
label-shift settings (Cora/label $94.03$ vs $94.11$) and on Arxiv, where
under the GNNSafe++ margin objective it reaches $73.73$ AUROC, above the
hybrid's $73.36$, with the best Arxiv ID accuracy of any method in
Table~\ref{tab:gnnsafe_local_table1style} ($61.16$). The one standalone
exception is Twitch: the classifier itself trains (ID accuracy $68.39$ vs the
hybrid's $70.52$), but its energy score does not separate the cross-graph OOD
inputs, so the hybrid ($95.75$) is preferred there. The per-chaos-channel BatchNorm of Appendix~\ref{app:implementation} is
used on all nine perturbation settings, where the sparse public splits
require it for every backbone in the pipeline
(Section~\ref{sec:exp_crosseval}).

\subsection{Distribution shift: standalone versus hybrid}

\begin{table}[ht]
\centering
\small
\caption{Shifted (OOD) test accuracy (\%) on the GOOD concept-shift
settings. Standalone: mean $\pm$ SD over 3 seeds, configuration selected on
the GOOD OOD-validation split (Appendix~\ref{app:crosseval_protocol});
hybrid and best baseline from Table~\ref{tab:good_table2_style}.}
\label{tab:crosseval_good}
\begin{tabular}{l ccl}
\toprule
Setting & Standalone & Hybrid & Best baseline \\
\midrule
GOOD-CBAS / color        & 82.38 $\pm$ 13.17 ($+$BN) & 88.57 & TAR 87.29 \\
GOOD-WebKB / university  & 36.41 $\pm$ 1.60 ($+$BN)  & 40.77 & TAR 30.83 \\
GOOD-Twitch / language   & 60.28 $\pm$ 0.51 ($+$BN)  & 61.21 & TAR 57.20 \\
GOOD-Cora / word         & 64.89 $\pm$ 0.21 ($+$BN)  & 64.81 & TAR 64.73 \\
GOOD-Cora / degree       & 62.60 $\pm$ 0.59 ($+$BN)  & 62.73 & TAR 61.73 \\
GOOD-Arxiv / time        & 65.43 $\pm$ 1.07 ($+$BN)  & 66.76 & TAR 66.08 \\
GOOD-Arxiv / degree      & 64.62 $\pm$ 0.15          & 66.43 & G-$\Delta$UQ 64.34 \\
\bottomrule
\end{tabular}
\end{table}

The standalone model is above every baseline of
Table~\ref{tab:good_table2_style} on 5 of 7 settings
(WebKB, Twitch, Cora word, Cora degree, Arxiv degree) and at ERM's level on
the remaining two (ERM: $82.43$ on CBAS and $65.64$ on Arxiv/time), while
the hybrid is above every baseline on all 7; on GOOD-Cora/word the two modes
agree within noise ($64.81$ vs $64.89\pm0.21$). GOOD-CBAS is a small synthetic
graph with the largest seed variance of the seven settings (standalone SD
$13.2$ over 3 seeds). ID accuracies of the
standalone runs are between $63.5$ and $96.2$ across the seven settings
(CBAS $96.2$, WebKB $76.0$, Twitch $63.5$, Cora word $68.8$, Cora degree
$68.9$, Arxiv time $74.1$, Arxiv degree $71.5$).

\subsection{Readout coincidence}

Table~\ref{tab:readout_coincidence} compares the two readouts of
Section~\ref{sec:readout} at identical checkpoints.

\begin{table}[ht]
\centering
\small
\caption{Quadrature readout versus mean-logit readout at identical
checkpoints under the full 10-split protocol of Table~\ref{tab:acc_brier}:
absolute Brier difference and the fraction of test nodes on which
$\arg\max\bar p_i$ and $\arg\max Z_{i,0}$ disagree (mean over 10 splits).
Accuracy differs by at most $0.04$ points on every dataset. In hybrid mode
the disagreement is $0.0000$ on all 14 datasets.}
\label{tab:readout_coincidence}
\begin{tabular}{l cc | l cc}
\toprule
Dataset & $|\Delta$Brier$|$ & Disagree (\%) & Dataset & $|\Delta$Brier$|$ & Disagree (\%) \\
\midrule
Cora        & 0.0002 & 0.03 & Roman-Emp.  & 0.0000 & 0.00 \\
Citeseer    & 0.0000 & 0.00 & Amz-Rat.    & 0.0000 & 0.00 \\
PubMed      & 0.0000 & 0.02 & Minesweeper & 0.0003 & 0.21 \\
Texas       & 0.0021 & 0.00 & Tolokers    & 0.0000 & 0.00 \\
Cornell     & 0.0000 & 0.00 & Questions   & 0.0000 & 0.00 \\
Wisconsin   & 0.0000 & 0.00 & CS          & 0.0001 & 0.00 \\
Chameleon   & 0.0001 & 0.00 & Squirrel    & 0.0007 & 0.08 \\
\bottomrule
\end{tabular}
\end{table}

\subsection{BatchNorm under the calibration protocol}
\label{app:bn_calibration}

Table~\ref{tab:bn_calibration} re-runs the Table~\ref{tab:acc_brier}
protocol with and without the per-chaos-channel BatchNorm on four datasets
(10 splits each). The BatchNorm-free controls reproduce
Table~\ref{tab:acc_brier} within a fraction of one standard deviation (the
Citeseer run uses the selected $P{=}2$ configuration of
Table~\ref{tab:selection}, with $\lambda_{\mathrm{reg}}{=}0$, and reproduces
that row as well). With BatchNorm, Cora and Roman-Empire are
unchanged within noise and Wisconsin is within one standard deviation, while
Citeseer degrades far beyond noise ($-13.2$ accuracy). BatchNorm is therefore a stabilization for the sparse-label OOD pipeline, where every baseline backbone also uses it, and is not used under the full-label calibration protocol.

\begin{table}[ht]
\centering
\small
\caption{Test accuracy (\%) / Brier, mean $\pm$ SD over 10 splits, under the
protocol of Table~\ref{tab:acc_brier} at one calibration configuration per
dataset, without and with per-chaos-channel BatchNorm.}
\label{tab:bn_calibration}
\resizebox{\textwidth}{!}{%
\begin{tabular}{l ccc}
\toprule
Dataset & BatchNorm off (control) & Table~\ref{tab:acc_brier} & BatchNorm on \\
\midrule
Cora         & 86.44 $\pm$ 1.56 / 0.216 $\pm$ 0.022 & 86.63 $\pm$ 1.26 / 0.207 $\pm$ 0.019 & 85.78 $\pm$ 1.76 / 0.220 $\pm$ 0.017 \\
Citeseer     & 80.27 $\pm$ 1.18 / 0.307 $\pm$ 0.010 & 80.20 $\pm$ 1.28 / 0.308 $\pm$ 0.010 & 67.03 $\pm$ 1.85 / 0.480 $\pm$ 0.024 \\
Wisconsin    & 93.62 $\pm$ 1.72 / 0.109 $\pm$ 0.029 & 93.25 $\pm$ 3.39 / 0.104 $\pm$ 0.047 & 91.75 $\pm$ 4.75 / 0.134 $\pm$ 0.061 \\
Roman-Empire & 78.82 $\pm$ 1.08 / 0.300 $\pm$ 0.012 & 78.84 $\pm$ 0.61 / 0.300 $\pm$ 0.007 & 78.60 $\pm$ 0.64 / 0.300 $\pm$ 0.009 \\
\bottomrule
\end{tabular}%
}
\end{table}

\subsection{Covariate shift on GOOD}
\label{app:covariate}

Table~\ref{tab:good_covariate} is the covariate-shift counterpart of
Table~\ref{tab:good_table2_style}: the same seven GOOD settings under their
covariate splits, comparing ERM-trained DSS-Hybrid (GCN encoder,
configuration selected by validation loss) with the same-protocol ERM GCN of
the GNNSafe pipeline (3 seeds). DSS-Hybrid is above the ERM control on 6 of
7 settings; the one loss, GOOD-Cora/degree, is within $1.8$ points. The two
Arxiv ERM baselines did not converge within the shared 120-epoch budget (ID
accuracy 45 to 50, versus 72 to 77 for the DSS runs on the same splits), so
those two margins are not informative on their own. With the
seven concept-shift settings of Table~\ref{tab:good_table2_style} this
covers both shift families that GOOD defines.

\begin{table}[ht]
\centering
\small
\caption{Shifted test accuracy (\%) under GOOD covariate shift: ERM-trained
DSS-Hybrid (GCN encoder, validation-loss-selected configuration) versus the
same-protocol ERM GCN. $^{*}$ERM baseline not converged within the shared budget (see text).}
\label{tab:good_covariate}
\begin{tabular}{l cc}
\toprule
Setting (covariate split) & ERM GCN & DSS-Hybrid (ERM) \\
\midrule
GOOD-CBAS / color        & 49.05 & \textbf{53.33} \\
GOOD-WebKB / university  & 13.49 & \textbf{29.10} \\
GOOD-Twitch / language   & 42.78 & \textbf{52.64} \\
GOOD-Cora / word         & 61.84 & \textbf{64.18} \\
GOOD-Cora / degree       & \textbf{56.66} & 54.94 \\
GOOD-Arxiv / time        & 44.80$^{*}$ & \textbf{70.39} \\
GOOD-Arxiv / degree      & 25.68$^{*}$ & \textbf{57.60} \\
\bottomrule
\end{tabular}
\end{table}

\clearpage
\section{Architecture Overview and Notation}
\label{app:overview}

Figure~\ref{fig:architecture} gives a schematic of the DSS-GNN architecture
of Section~\ref{sec:method}: the two spectral axes that index the chaos
coefficients, one DSS layer, the two readouts and the chaos energy, and the two deployment
modes. Table~\ref{tab:notation} summarizes the notation of the DSS layer,
Eqs.~\eqref{eq:pce_state}--\eqref{eq:dss_projection_update}.

\begin{figure}[ht]
\centering
\includegraphics[width=\textwidth]{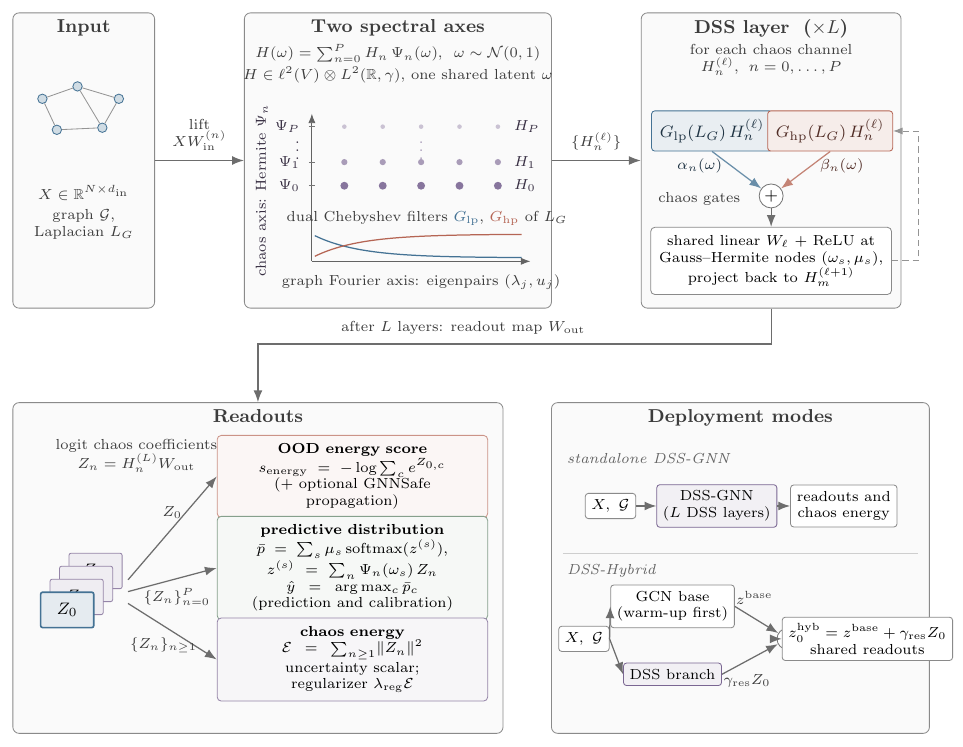}
\caption{Overview of DSS-GNN. Input features $X$ on graph $\mathcal{G}$ are
lifted to chaos coefficients $H_n^{(0)}=XW_{\mathrm{in}}^{(n)}$, which are indexed by two spectral axes: the graph Fourier axis (Laplacian eigenpairs
$(\lambda_j,u_j)$, filtered by dual Chebyshev filters
$G_{\mathrm{lp}},G_{\mathrm{hp}}$) and the Wiener chaos axis (Hermite basis
$\Psi_n$ of a scalar Gaussian latent $\omega$). Each DSS layer propagates every
chaos channel $H_n^{(\ell)}$ through both filters, couples them with per-order
chaos gates $\alpha_n,\beta_n$, and applies a shared linear map and ReLU at
Gauss--Hermite nodes before projecting back
(Eqs.~\eqref{eq:cheby_filters}--\eqref{eq:dss_projection_update}). The final
coefficients $Z_n=H_n^{(L)}W_{\mathrm{out}}$ are the inputs of the two task readouts, the
mean-logit energy score for OOD detection~\eqref{eq:energy_score} and the
quadrature-averaged predictive $\bar p$ for class prediction and
calibration~\eqref{eq:predictive_average}, together with the chaos energy
$\mathcal{E}$, which is used as regularizer and uncertainty
scalar~\eqref{eq:logit_covariance} (bottom left). Bottom right: the two deployment modes, standalone DSS-GNN
and DSS-Hybrid, which adds the DSS branch residually (scale
$\gamma_{\mathrm{res}}$) to a warmed-up GCN base (Section~\ref{sec:hybrid});
both use the same readouts. The node index $i$ is suppressed for
readability.}
\label{fig:architecture}
\end{figure}

\begin{table}[t]
\centering
\small
\caption{Notation for the DSS layer,
Eqs.~\eqref{eq:pce_state}--\eqref{eq:dss_projection_update}.}
\label{tab:notation}
\begin{tabular}{p{0.24\textwidth}p{0.70\textwidth}}
\toprule
Symbol & Meaning \\
\midrule
$\omega$ & the single shared standard Gaussian latent; never sampled,
integrated out by quadrature \\
$\Psi_n$ & normalized Hermite polynomial of order $n$; orthonormal basis of
$L^2(\mathbb{R},\gamma)$ \\
$P$ & chaos truncation order; $n,m=0,\dots,P$ index input/output chaos
orders \\
$H_n^{(\ell)}\in\mathbb{R}^{N\times d_\ell}$ & order-$n$ chaos coefficient
of the layer-$\ell$ embeddings~\eqref{eq:pce_state}; the embedding field is
$H^{(\ell)}(\omega)=\sum_n H_n^{(\ell)}\Psi_n(\omega)$ \\
$\widetilde L_G$, $T_k$ & rescaled Laplacian $2L_G/\lambda_{\max}-I$ and
Chebyshev polynomial of degree $k$ (Section~\ref{sec:background}) \\
$G_{\mathrm{lp}},G_{\mathrm{hp}}$ & learned low/high-pass filter polynomials
of degrees $K_{\mathrm{lp}},K_{\mathrm{hp}}$ with coefficients
$c_k^{(\mathrm{lp})},c_k^{(\mathrm{hp})}$~\eqref{eq:cheby_filters} \\
$\alpha_n(\omega),\beta_n(\omega)$ & per-order chaos gates coupling the two
filter branches; deterministic scalars, or order-$P_g$ expansions with
coefficients $a_{n,r},b_{n,r}$~\eqref{eq:random_gates} \\
$\omega_s,\mu_s,S$ & Gauss--Hermite quadrature nodes and weights,
$s=1,\dots,S$; exact for the filter/gate polynomials when $S\ge P+1$ (for
$P_g\le1$) \\
$\widetilde H^{(\ell+1)}(\omega_s)$ & pre-activation field evaluated at
quadrature node $\omega_s$~\eqref{eq:sample_preact} \\
$W_\ell$, $\sigma$ & trainable linear map $d_\ell\to d_{\ell+1}$ and pointwise
activation; Eq.~\eqref{eq:dss_projection_update} projects
$\sigma(\widetilde H^{(\ell+1)}(\omega_s))$ back onto the $\Psi_m$ basis \\
\bottomrule
\end{tabular}
\end{table}

\ifarxiv\else
\clearpage
\input{sections/checklist}
\fi

\end{document}